\documentclass[a4paper]{article}

\usepackage{epsfig}
\usepackage{amssymb}
\usepackage{natbib}
\usepackage{graphicx}
\usepackage{hyperref}
\hypersetup{
	colorlinks=false,
	linkcolor=blue,
	citecolor=blue,
	filecolor=magenta,      
	urlcolor=cyan,
}
\usepackage{amsfonts}
\usepackage{amstext}  
\usepackage{comment}
\usepackage{mathtools}
\usepackage{amsmath}
\usepackage{dsfont}
\usepackage{float}
\usepackage{commath}
\usepackage{bbm}
\usepackage{longtable}
\usepackage{tcolorbox}
\usepackage{algorithm}
\usepackage[noend]{algpseudocode}
\usepackage{statmath}
\usepackage{tagging}
\usepackage{tikz}
\usepackage{mytikzstyles}
\usepackage{standalone}
\usepackage{thmtools}
\usepackage{lscape}
\usepackage{fullpage}

\graphicspath{{figures/},{tikz/}, {}}

\newcommand{\dataset}{{\cal D}}

\let\argmin\relax\DeclareMathOperator*{\argmin}{arg\,min}

\renewcommand{\P}{\mathbb{P}}
\newcommand{\Exp}{\mathrm{E}}

\newcommand{\R}{\mathbb{R}}

\newcommand{\epsi}{\epsilon}
\newcommand{\1}{\mathbbm{1}}

\newcommand{\ind}{\perp\!\!\!\!\perp} 
\newcommand{\rank}{\mathrm{rank}}
\newcommand{\vspan}{\mathrm{span}}
\renewcommand{\norm}[1]{\left\lVert#1\right\rVert}
\newcommand{\train}{\mathrm{tr}}
\newcommand{\test}{\mathrm{te}}
\newcommand{\supp}{\mathrm{supp}}

\newcommand{\col}{\mathrm{im}}
\newcommand{\interior}[1]{%
  {\kern0pt#1}^{\mathrm{o}}%
}
\newcommand{\modelname}{SIMDG}
\newcommand{\cfname}{BCF}
\newcommand{\risk}{\mathcal{R}}
\newcommand{\imp}{{\mathcal{I}_0}}
\newcommand{\cf}{\star}
\newcommand{\ols}{{\operatorname{LS}}}

\renewcommand{\d}{\mathrm{d}}

\newcommand{\N}{\mathbb{N}}

\newcommand{\cF}{\mathcal{F}}

\newcommand{\cM}{\mathcal{M}}
\newcommand{\cN}{\mathcal{N}}

\newcommand{\BlackBox}{\rule{1.5ex}{1.5ex}}
\ifdefined\proof
    \renewenvironment{proof}{\par\noindent{\bf Proof\ }}{\hfill\BlackBox\\[2mm]}
\else
    \newenvironment{proof}{\par\noindent{\bf Proof\ }}{\hfill\BlackBox\\[2mm]}
\fi
\newtheorem{example}{Example} 
\newtheorem{theorem}{Theorem}
\newtheorem{lemma}[theorem]{Lemma} 
\newtheorem{proposition}[theorem]{Proposition} 
\newtheorem{remark}[theorem]{Remark}
\newtheorem{cor}[theorem]{Corollary}
\newtheorem{definition}[theorem]{Definition}

\newtheorem{assumption}{Assumption}
\newtheorem{setup}{Setting}

\definecolor{myRubineRed}{rgb}{0.82, 0.0, 0.34}
\definecolor{myDeepGreen}{rgb}{0.0, 0.4, 0.0}

\title{\bf Boosted Control Functions:\\
Distribution generalization and invariance\\
in confounded models}

\author{
Nicola Gnecco\textsuperscript{1}\thanks{Part of this work was conducted while NG and JP were at the University of Copenhagen.} ,\,\,
Jonas Peters\textsuperscript{2}\footnotemark[1],\,
Sebastian Engelke\textsuperscript{3},\, 
and 
Niklas Pfister\textsuperscript{4}
}
\date{\textsuperscript{1}Imperial College London, \textsuperscript{2}ETH Zürich,  \textsuperscript{3}University of Geneva, \\ \textsuperscript{4}University of Copenhagen,  
}

\begin{document}

\maketitle

\begin{abstract}

  Modern machine learning methods and the availability of large-scale data have significantly advanced our ability to predict target quantities from large sets of covariates. 
  However, these methods often struggle under distributional shifts, particularly in the presence of hidden confounding. While the impact of hidden confounding is well-studied in causal effect estimation, e.g., instrumental variables, its implications for prediction tasks under shifting distributions remain underexplored.
  This work addresses this gap by introducing a strong notion of invariance that, unlike existing weaker notions, allows for distribution generalization even in the presence of nonlinear, non-identifiable structural functions. Central to this framework is the Boosted Control Function (BCF), a novel, identifiable target of inference that satisfies the proposed strong invariance notion and is provably worst-case optimal under distributional shifts. The theoretical foundation of our work lies in Simultaneous Equation Models for Distribution Generalization (SIMDGs), which bridge machine learning with econometrics by describing data-generating processes under distributional shifts.
  To put these insights into practice, we propose the ControlTwicing algorithm to estimate the BCF using nonparametric machine-learning techniques and 
  study its generalization performance on synthetic and real-world datasets compared to robust and empirical risk minimization approaches.
\end{abstract}

\section{Introduction}
Prediction and forecasting methods are fundamental in
describing how a target quantity behaves in the future or under
different settings. With recent advances in machine learning and the
availability of large-scale data, prediction has become reliable in many applications, such as
macroeconomic forecasting \citep{stock2016}, and
predicting effects of policies
\citep{hill2011,kleinberg2015, atheyimbens2016, kunzel2019}. However, it is also well-known
that focusing solely on prediction when reasoning about a system
under changing conditions can be misleading,
especially in the presence of unobserved confounding.
While there are
several well-established methods
for
dealing with unobserved confounding in causal effect
estimation, less research has focused on
comparable approaches for prediction tasks.
In this work, we consider the problem of predicting a response when the training and testing distributions differ in the presence of unobserved confounding.

In the literature on causal effect estimation from observational data,
the main approaches to deal with unobserved confounding are instrumental
variables \citep{angrist1996identification, ng1995nonparametric,newey1999}, regression discontinuity \citep{angrist1999using}, and difference-in-differences \citep{angrist1991}.
Instrumental variable approaches, for example, use specific
types of exogenous variables (called instruments) that can be seen
as natural experiments to identify and estimate causal effects.
Most of the existing methods
require
the causal
effect to be identifiable.
In simple scenarios such as linear or binary models, these 
identifiability
conditions are well understood \citep{angrist1996identification,amemiya1985advanced}.
Identification strategies in more complex scenarios impose  specific structures on the causal effect, e.g., sparsity \citep{pfister22a}, independence between the instruments and the residuals \citep{dunker2014iterative,dunker2021adaptive, saengkyongam2022exploiting, loh2023}, or the milder mean-independence condition \citep{newey2003instrumental}.
However, when there are many endogenous covariates, finding a sufficiently large number of
valid instruments to identify and estimate the causal effect is often not feasible -- even when the effect is sparse.

In this work, we show that even when the causal function is not
identifiable, it is possible to exploit the heterogeneity induced by a
set of exogenous variables to learn a prediction function that yields
valid predictions under a class of distributional shifts.
Formally, we consider the task of
predicting a real-valued outcome $Y \in \R$ from a large set of
(possibly) endogenous covariates $X \in \R^p$ when the training and
testing data follow a different distribution.
Distribution generalization has received much attention particularly
in the machine learning community
and is usually tackled from a
worst-case point of view, which is particularly relevant in
high-stakes applications.
Given a set of
potential distributions $\mathcal{P}$ of a random vector that contains
the components $(X, Y)$, we aim
to find a predictive function
$f^*$ minimizing the worst-case risk over the set of distributions
$\mathcal{P}$, that is
\begin{align*}%
	\sup_{P \in \mathcal{P}}\Exp_P[(Y - f^*(X))^2] = \inf_{f \in \mathcal{F}}\sup_{P \in \mathcal{P}}\Exp_P[(Y - f(X))^2],
\end{align*}
where $\mathcal{F}$ is a class of measurable functions.

The task of distribution generalization is intractable without characterizing the set
of potential distributions $\mathcal{P}$.
For example, one can model $\mathcal{P}$ by changing the marginal $P_\train^{X} \neq P_\test^{X}$, known as covariate shift \citep{shimodaira2000,sugiyama2007covariate},
the conditional $P_\train^{Y\mid X} \neq P_\test^{Y \mid X}$, known as concept shift \citep{quinonero2009dataset,gama2014survey}, or a combination of both, \citep{arjovsky2020invariant,krueger21a}.
Alternatively, one can model $\mathcal{P}$ as a set of distributions
that are within a neighborhood of the training distribution
$P_\train^{X, Y}$ with respect to the Wasserstein distance
\citep{sinha2017certifying} or $f$-divergence \citep{bagnell2005robust, hu2018does}.
Furthermore, one can model $\mathcal{P}$ as the convex hull of different distributions $P_{\train, 1}^{X, Y}, \dots, P_{\train, m}^{X, Y}$ that are observed at training time \citep{meinshausen2015, sagawa2019distributionally}.

Here,
we model the heterogeneity in the distributions
$\mathcal{P}$ via a vector of exogenous variables $Z \in \R^r$ which
are only observed at training time and induce shifts in the
conditional mean of the response given the covariates.
Our framework builds upon the literature of invariance-based methods for distribution generalization.
Invariance (also known as autonomy or modularity \citep{haavelmo1944probability, aldrich1989,pearl2009causality}) is one of the key ideas in causality and states that predictions made with causal models are unaffected by interventions on the covariates. Invariance has been used for causal structure search
\citep[e.g.,][]{peters2016causal} but is also implicitly used in causal effect
estimation methods such as instrumental variables. More recently,
invariance has also been employed as a concept to achieve distribution generalization via the invariant most predictive (IMP)
function \citep{magliacane2018domain,
rojas2018invariant,arjovsky2020invariant,buhlmann2020invariance,
christiansen2020causal,
krueger21a,jakobsen2022distributional,saengkyongam2022exploiting}.

In this work, we propose a strong notion of invariance that, unlike weaker notions \citep[e.g., ][]{arjovsky2020invariant,buhlmann2020invariance,
christiansen2020causal, krueger21a,jakobsen2022distributional}, ensures distribution generalization when the structural function is nonlinear and possibly not identifiable.
Building upon the control function approach
\citep{ng1995nonparametric,newey1999}, we propose the boosted control
function (BCF), 
a novel, identifiable target of inference that satisfies the proposed strong invariance notion.
The BCF aligns with the
invariant most
predictive (IMP)
function and  is provably worst-case optimal under distributional shifts induced by the exogenous variables.
Moreover, we provide necessary and sufficient conditions for identifying the BCF under continuous and discrete exogenous variables.
Our theoretical results rest upon the simultaneous equation models for distribution generalization (SIMDGs), a novel framework to describe distributional shifts induced by the exogenous variables. SIMDGs establish a connection between distribution generalization in machine learning and simultaneous equation models
\citep{haavelmo1944probability,amemiya1985advanced} in econometrics.

We further develop the ControlTwicing algorithm which estimates the BCF
using the `twicing'
idea from \citet{tukey1977exploratory}.
Our method can be
applied using
various
machine learning methods such as ridge regression,
lasso, random
forests,
boosted trees and neural networks.
In a set of numerical experiments,
we study the
generalizing
properties of the BCF estimator
and observe its
advantage over standard prediction methods
when
the training and testing distributions differ.
Moreover,
on the California housing dataset \citep{pace1997sparse}, we show that the BCF is
robust to previously unseen distributional shifts and is competitive with existing robust methods.

\section{SIMs for Distribution Generalization}\label{sec:setup}

Given a vector of covariates $X \in \R^p$, 
our goal is
to identify
a function $f:\R^p\rightarrow\R$ 
that predicts
a response $Y \in \R$ 
from $X$,
under distributional shifts 
induced
by exogenous variables $Z \in \R^r$.
Let $\mathcal{P}$ 
denote
a
collection of
distributions of
a random vector that contains the components $(X, Y, Z)$
and let $P_\train \in \mathcal{P}$ denote the training distribution,
from which we observe
$(X, Y, Z) \sim P_\train^{X, Y, Z}$. To evaluate the performance of
a candidate function
$f \in \mathcal{F}$,
where $\mathcal{F}$ is a class of measurable functions,
we define the risk function
$\mathcal{R}: \mathcal{P} \times \mathcal{F} \to [0, \infty]$
for all $P \in \mathcal{P}$ and $f \in \mathcal{F}$ by
\begin{align}\label{eq:risk}
	\risk(P, f) \coloneqq \Exp_P[(Y - f(X))^2].
\end{align}
We assume that we observe independent copies of
$(X, Y, Z)\sim P_{\train}^{X, Y, Z}$ and want to predict $Y$ from $X$ under
a
potentially different but unknown testing distribution
$P_{\test}^{X, Y}$, where $P_{\test} \in\mathcal{P}$. %
In this work, our estimand (or target) is
a function $\tilde{f}$ that minimizes the risk~\eqref{eq:risk} under the worst-case distribution
in $\mathcal{P}$.
\begin{definition}[Distribution generalization]\label{def:dist-gen}
	Denote by $\mathcal{F}$ a subset of measurable functions $f: \R^p \to \R$ and by $\mathcal{P}$ a class of distributions of the random vector $(X, Y, Z)$. Let $\tilde{f}: \R^p \to \R$ be a function satisfying
	\begin{equation}
		\label{eq:minimax_target}
		\sup_{P\in\mathcal{P}} \risk(P, \tilde{f})
		= \inf_{f\in\mathcal{F}}\sup_{P\in\mathcal{P}}\risk(P, f)< \infty.
	\end{equation}
	Then, we say that
	\emph{distribution generalization is achievable}
	and
	$\tilde{f}$
	is a
	\emph{generalizing function}.
\end{definition}

Without any constraints on the class of
distributions $\mathcal{P}$, distribution generalization is generally
not achievable.
One way to constrain the class of
distributions $\mathcal{P}$ is via causal models as done e.g.\
by \citet{christiansen2020causal}
who consider structural causal models
with independent error terms.
Here, we consider semi-parametric simultaneous equation models where

\noindent\begin{subequations}
	\begin{minipage}{.4\textwidth}
		\begin{align}\label{eq:sem}
			\begin{split}
				Y &= f_0(X) + U,\\
				X &= M_0Z + V,
			\end{split}
		\end{align}
		\vspace*{1pt}
	\end{minipage}%
	\begin{minipage}{.61\textwidth}
		\begin{align}
			 & (U, V) \ind Z, \label{eq:sem-2}                                               \\
			 & \Exp[(U, V)] = 0,\ \Exp\big[\norm{(U, V)}_2^2\big] < \infty, \label{eq:sem-3}
		\end{align}
		\vspace*{1pt}
	\end{minipage}
\end{subequations}
(see Figure~\ref{fig:graphs} in Appendix~\ref{sec:scm_equations} for a visualization)
and assume that $Z$ and $X$ are centered to mean zero under $P_\train$ (see also Setting~\ref{set:setting-1}).
The first equation in~\eqref{eq:sem} is in structural form and  describes the
mechanism between the
dependent variable $Y$ and the 
endogenous
covariates $X$ under distributional shifts of the exogenous variable $Z$.
The function
$f_0: \R^p \to \R$ 
captures a possibly nonlinear
relationship between $X$ and $Y$.
The second equation in~\eqref{eq:sem} is in reduced form, that is, the
dependent variable $X$ depends only on the exogenous variables
$Z$ and $V$. The matrix $M_0 \in \R^{p\times r}$
describes the linear dependence between $X$ and $Z$
and is not required to be full-rank.
The
hidden variables $U$ and $V$ are not required to be independent, 
allowing
for
unobserved
confounding between $X$ and $Y$.
This confounding, together with shifts in $Z$, can induce both covariate shifts ($P_\train^X \neq P_\test^X$) and concept shifts  ($P_\train^{Y \mid X} \neq P_\test^{Y \mid X}$) within model~\eqref{eq:sem}.

\begin{remark}[Linear dependence between $X$ and $Z$]
	For the sake of simplicity, in~\eqref{eq:sem} we assume a linear
	functional dependence between $Z$ and $X$.  In principle, however,
	it is possible to consider nonlinear maps of the form
	$z \mapsto G(z) \coloneqq \Theta_0 \phi(z)$, for some real-valued
	matrix $\Theta_0 \in \R^{p \times q}$ with a (known) basis
	$\phi: \R^r \to \R^q$.
  Without parametric restrictions on the function 
  $\phi$,
  however, generalization to shifts in $Z$ is impossible \citep{christiansen2020causal}.
\end{remark}

\begin{remark}[$Z$ categorical]\label{rmk:cat-z}
	If $Z$ is an exogenous categorical variable (which often occurs in practical
	applications), our model may still apply after encoding $Z$
	appropriately. Suppose that $Z$ takes values in $\{a_0, \dots, a_r\}$
	with probabilities $\pi_0, \dots, \pi_r$. We then encode for all
	$j\in\{1,\dots,r\}$ the category $a_j$ as $e_j$, where
	$e_j \in \R^r$ is the $j$-th standard basis vector and $a_0$ as
	$-[\frac{\pi_1}{\pi_0}, \dots, \frac{\pi_r}{\pi_0}]^\top$.  The
	newly encoded variable $Z \in \R^r$ satisfies $\Exp[Z] = 0$ by
	construction. Moreover, under the assumptions that (a) $\Exp[X]=0$
	and that (b) the conditional distributions of $X$ given the
	categories are shifted versions of each other, we can express the
	relation between $X$ and $Z$ as in~\eqref{eq:sem}. More concretely,
	define $\mu_j \coloneqq \Exp[X \mid Z = j]$ for all
	$j \in \{0, \dots, r\}$ and the matrix
	$M_0 \coloneqq [\mu_1, \dots, \mu_r] \in \R^{p \times r}$. Then, (b)
	formally requires for all
	$a,b\in\text{supp}(P^Z_{\train})$\footnote{Let
		$X: \Omega \to \mathcal{X}$ be a random variable where $\Omega$
		denotes the sample space and $\mathcal{X}$ is a Euclidean space. Let
		$P^X$ denote the distribution of $X$. The support of $X$, denoted
		by $\supp(P^X)$, is the set of all $x \in \mathcal{X}$ such that every open neighborhood $N_x \subseteq \mathcal{X}$ of $x$ has positive
		probability.} that $X - \mu_a$ conditioned on $Z=a$ has the same
	distribution as $X-\mu_b$ conditioned on $Z=b$. Hence, together with
	(a) using $V=X - M_0Z$ implies that $\Exp[V] = 0$ and
	$X = M_0 Z + V$.
\end{remark}
When the exogenous variable is categorical, there exist different representations of $Z$ to express the predictor vector as in~\eqref{eq:sem}. The different representations of $Z$, however, are all equivalent with respect to the generalization guarantees (see Section~\ref{subsec:distgen}) since they describe the same linear span of the conditional means $\Exp[X \mid Z = a_j]$, where $j \in \{0, \dots, r\}$.

The following remark shows that
under certain conditions,
the proposed model~\eqref{eq:sem}--\eqref{eq:sem-3} allows the exogenous variable $Z$
to
directly affect the response~$Y$.

\begin{remark}[Allowing $Z$ to affect $Y$]\label{rmk:projectability}
	In this work, instead of identifying the structural function
	$f_0$, we aim at predicting $Y$ from $X$ under
	distributional shifts on $Z$.
  In contrast to widely-used assumptions on IVs, our model~\eqref{eq:sem}--\eqref{eq:sem-3} allows $Z$ to have a direct effect on $Y$,  provided that $Z$ can be expressed as a function of $X$ and $V$. More precisely, we allow for $Y = g_0(X) + \beta_0^\top Z + U$, for some $g_0 : \R^p \to \R$ and $\beta_0\in\col(M_0^\top)$.
	The assumption that $\beta_0\in\col(M_0^\top)$ has been
	termed \emph{projectability condition}
	\citep{rothenhausler2021anchor} and ensures that
	$\beta_0^\top Z$ can be expressed as a linear combination of the covariates $X$ and
	hidden variables $V$.
	This condition is automatically satisfied
	when $\rank(M_0) = r$ because then
	$\col(M_0^\top) = \R^r$.
	When $\rank(M_0) < r$,
	however, the projectability condition constrains the allowed
	models.
	Given the projectability condition, the exogenous variable 
  satisfies
  $Z = M_0^{\dagger}(X - V)$ via  equations~\eqref{eq:sem},
	where $M_0^{\dagger} \in \R^{r \times p}$ is the Moore--Penrose
	inverse of $M_0 \in \R^{p \times r}$.  Therefore, we can rewrite the structural equation $ Y = g_0(X) + \beta_0^T Z + U$ as
	\begin{align*}%
		Y & = \left( g_0(X) + \beta_0^\top M_0^{\dagger} X \right) + \left( U - \beta_0^\top M_0^{\dagger} V \right) \eqqcolon f_0(X) + \tilde{U}.
	\end{align*}
	By construction, the two structural equations
	for $Y$
	induce the same distribution $P^{X, Y, Z}$
	for any distributional shift on $Z$, and, therefore, they
	are equivalent for our purpose to solve~\eqref{eq:minimax_target}.
\end{remark}

To construct the set $\mathcal{P}$ of potential distributions,
we introduce the following  simultaneous equation model (SIM) for
distribution generalization (SIMDG).
\begin{definition}[SIM for Distribution Generalization (\modelname)]\label{def:ivg}
	Let
	$\mathcal{Q}_0$ be a set of distributions  over $\R^r$ and
	$\Lambda_0$ be a
	distribution over $\R^{1 + p}$ such that if
	$(U,V) \sim \Lambda_0$ then
	\eqref{eq:sem-3} holds.
	Let further
	$f_0: \R^p \to \R$ be a measurable function
	and
	$M_0\in\R^{p\times r}$
	a matrix.
	We call the tuple $(f_0, M_0, \Lambda_0, \mathcal{Q}_0)$ a
	\emph{\modelname}.
	For all $Q\in\mathcal{Q}_0$ the model
	$(f_0, M_0, \Lambda_0,Q)$
	induces a unique distribution over $(U, V, X, Y, Z)$ via
	$Z \sim Q$,
	$(U,V) \sim \Lambda_0$, \eqref{eq:sem-2}, and the
	simultaneous equations~\eqref{eq:sem}.
	We define the set of induced distributions
	by
	\begin{align}\label{eq:mathcalP}
		\mathcal{P}_0 :=
		\left\{P \text{ distr.\ over } \R^{2p+r+2} \mid
		\exists Q \in \mathcal{Q}_0\ \text{s.t. } (f_0, M_0,
		\Lambda_0, Q)  \text{ induces }P \text{ via \eqref{eq:sem}}\right\}.
	\end{align}
\end{definition}

A \modelname \ $(f_0, M_0, \Lambda_0, \mathcal{Q}_0)$
defines a
collection of distributions
$\mathcal{P}_0$.
In particular,
the training and testing distribution $P_\train$ and $P_\test$ are
(potentially different) distributions
induced by
(potentially different)
$Q_\train, Q_\test \in \mathcal{Q}_0$;
the changes in $Q \in \mathcal{Q}_0$, in turn,
induce mean shifts of $X$ in the directions of the columns of the matrix $M_0$.
In Appendix~\ref{sec:generative} we provide further technical details on how a SIMDG generates the class of distributions~$\mathcal{P}_0$.

Even by constraining the set of distributions by a \modelname, distribution generalization is achievable only if we further impose specific assumptions on either $\mathcal{P}_0$ or the function class $\mathcal{F}$ \citep{christiansen2020causal}.
Assumptions on $\mathcal{P}_0$  usually
require that
the training distribution dominates all distributions,
while assumptions on the function class $\mathcal{F}$ usually ensure that the target function extrapolates outside the training support in a known way.
\begin{assumption}[Set of distributions $\mathcal{P}_0$]\label{ass:ident0}
	For all $P \in \mathcal{P}_0$ it holds that $P \ll P_\train$.
\end{assumption}

\begin{assumption}[Function class $\mathcal{F}$]\label{ass:ident}
	The function class $\mathcal{F}$ is such that for all $f, g \in \mathcal{F}$
	and all $P \in \mathcal{P}_0$
	it holds
	\begin{equation*}
		f(X) = g(X),\ P_\train\text{-a.s.}\implies f(X) = g(X),\ P \text{-a.s.}
	\end{equation*}
\end{assumption}
Note that
Assumption~\ref{ass:ident0} implies Assumption~\ref{ass:ident}.
Throughout the paper, we 
  will assume that either Assumption~\ref{ass:ident0} or Assumption~\ref{ass:ident} are satisfied. Furthermore,
we will use the following data-generating process.
\begin{setup}[Data-generating process]
	\label{set:setting-1}
	\label{ass:train-test}
	Fix a \modelname \ $(f_0, M_0, \Lambda_0, \mathcal{Q}_0)$
	where $\mathcal{Q}_0$ induces a set of distributions $\mathcal{P}_0$.
	Moreover,  assume that $\sup_{P \in \mathcal P_0} \Exp_P [f_0(X)]^2 <\infty$.
	Let $Q_\train\in\mathcal{Q}_0$ such that $Z\sim Q_\train$
	satisfies $\Exp_{Q_\train}[Z] = 0$
	($Z$ can but does not have to be categorical, see Remark~\ref{rmk:cat-z})
	and
	$\Exp_{Q_\train}[ZZ^\top] \succ 0$, and let $Q_\test\in\mathcal{Q}_0$ be
	an arbitrary distribution.
	Denote the \emph{training
		distribution} by $P_\train \in \mathcal{P}_0$
	and the \emph{testing distribution} by $P_\test \in \mathcal{P}_0$; both are distributions over $(U, V, X, Y, Z)$ induced by
	$(f_0, M_0, \Lambda_0, Q_\train)$ and $(f_0, M_0, \Lambda_0, Q_\test)$, respectively.
	We now
	consider the following two phases.
	\begin{itemize}
		\item[(1)] \emph{(Training)} Observe  an i.i.d.\ sample
			$(X_1,Y_1,Z_1),\ldots,(X_n,Y_n,Z_n)$  of size $n$  with  distribution
			$P^{X, Y, Z}_\train$.
		\item[(2)] \emph{(Testing)} Given an independent draw $X\sim P^{X}_\test$
			predict the response $Y$.
	\end{itemize}
\end{setup}
In practical applications, a categorical exogenous variable $Z$ can represent the different environments from which the data was collected, such as hospitals \citep{bandi2018detection}. In such cases, using $Z$ as an additional covariate
does
not
enhance
the prediction
accuracy
for
environments not observed during training time.
A continuous exogenous variable $Z$
can represent
geospatial information, such as latitude and longitude. 
As we discuss in Section~\ref{sec:california-housing} based on the example of predicting house prices, 
using the geospatial information directly as covariates can lead to inaccurate predictions due to extrapolating to areas not included in the training data. However, using latitude and longitude as exogenous variables that model distribution shifts in the other covariates can improve predictive accuracy in areas that are not covered by the training data.
 
SIMDGs describe a set of models that are closely related to the SIMs from the instrumental variable literature \citep[e.g.,][]{newey1999} and the structural causal models (SCMs) from the causality literature in statistics \citep[e.g.,][]{pearl2009causality},
as we discuss in Appendices~\ref{sec:sim_equations} and~\ref{sec:scm_equations}.

\section{Invariant Most Predictive Functions}\label{sec:distr_gen}

The concept of invariant most predictive (IMP) function has been recently proposed
as a guiding principle to identify a generalizing
function
\citep{magliacane2018domain, rojas2018invariant,arjovsky2020invariant,buhlmann2020invariance,
	christiansen2020causal,  krueger21a,jakobsen2022distributional,saengkyongam2022exploiting}.
To provide a motivating example, consider
Setting~\ref{set:setting-1} with \modelname\
$(f_0, M_0, \Lambda_0, \mathcal{Q}_0)$, where $f_0$ belongs to the
class of linear functions $\mathcal{F}$ and $\mathcal{Q}_0$
consists of arbitrary distributions on $\R^r$.
For a fixed $z \in \R^r$, define the
point mass distribution $Q_z \coloneqq \delta_z\in\mathcal{Q}_0$, and denote by
$P_z$ the distribution induced by $(f_0, M_0, \Lambda_0, Q_z)$.
Then, using the simultaneous equations~\eqref{eq:sem}, for all
functions $g \in \mathcal{F}$ the risk of $g$ under the perturbed
distribution $P_z$ can be expressed as
\begin{align}\label{eq:risk-example}
	\begin{split}
		\risk(P_z, g)
		=&\ \Exp_{P_z}\left[ \left( Y - g(X) \right)^2 \right] = \Exp_{P_z}[U^2] + \delta^\top \Exp_{P_z}[XX^\top]\delta - 2 \delta^\top \Exp_{P_z}[X U]\\
		=&\ \Exp_{P_z}[(U - \delta^\top V)^2] + \delta^\top M_0 zz^\top  M_0^\top \delta,
	\end{split}
\end{align}
where $\delta^\top x \coloneqq f_0(x) - g(x)$. When
$M_0^\top \delta \neq 0$, the risk $\risk(P_z, g)$ can be made
arbitrarily large by increasing the magnitude of $z \in \R^r$, and therefore $g$ is
then
not a
generalizing function.
The reason is that
the
distribution of
residuals
$Y - g(X) = U - \delta^\top X = U - \delta^\top V - \delta^\top M_0 Z$
is
not invariant to changes in the marginal distribution of $Z$.
The
IMP function
identifies a function minimizing~\eqref{eq:risk-example} among those yielding an invariant distribution of residuals.
In the above
example with a linear function class, it is clear that the invariance of the
residual distribution is a necessary condition to achieve distribution
generalization.
However, for more general
function classes the relation between invariance and distribution generalization is more intricate.
Even more so, when the function class is not constrained to the linear setting, existing IMP-based approaches can fail at identifying a function that is invariant in the sense of Definition~\ref{def:inv_funs}
(see Example~\ref{ex:inv-stronger-ind-2}), and therefore, may not yield a generalizing function
(see Example~\ref{ex:opt-h1-not-imp}).

The goal of this section is to investigate the relation between invariance and distribution generalization in the more general setting when $\mathcal{F}$ is not constrained to be a parametric function class. To do so, we first formally define the notion of invariant function.

\begin{definition}[Invariant function]\label{def:inv_funs}
	Assume Setting~\ref{set:setting-1} and for all $P \in \mathcal{P}_0$ and all $f \in \mathcal{F}$ denote by $P^{Y - f(X)}$ the distribution of the random variable $Y - f(X)$ under the probability measure $P$.
	We say $f \in \mathcal{F}$ is \emph{invariant} w.r.t.\
	$\mathcal{P}_0$ if
	\begin{align*}%
		P_\train^{Y - f(X)} = P^{Y - f(X)},\ \text{for all}\ P \in \mathcal{P}_0.
	\end{align*}
	Furthermore, we define the \emph{set of invariant functions}
	\begin{align*}%
		\mathcal{I}_0 := \{f \in \mathcal{F} \mid f\ \text{is
			invariant w.r.t.\ }\mathcal{P}_0\}.
	\end{align*}
\end{definition}
Given a
\modelname\ $(f_0, M_0, \Lambda_0, \mathcal{Q}_0)$ and the set of
induced distributions $\mathcal{P}_0$, the structural function $f_0$ is always
invariant since the distribution of its residuals
$Y - f_0(X) = U$ does not depend on
$P\in\mathcal{P}_0$.
Additionally, depending on the relation
between the exogenous variables and the covariates, there may exist
further invariant functions other than $f_0$.
More precisely, the size of the set of invariant functions $\mathcal{I}_0$ depends on the rank and order conditions of identifiability \citep{amemiya1985advanced}.
For example, under Setting~\ref{set:setting-1}, if $\rank(M_0) = p$, then
the set of invariant functions is a
singleton $\mathcal{I}_0 = \left\{ f_0 \right\}$.
If, instead
$q \coloneqq \rank(M_0) < p$, then the set of invariant functions
$\mathcal{I}_0$ can contain infinitely many
elements, as we now argue.
Let $\ker(M_0^\top)$ denote the null space of $M_0^\top$ with dimension $p - q  > 0$.
Define the matrix $R\coloneqq(r_1,\ldots,r_{p-q})\in\R^{p\times(p-q)}$, where
$r_1,\ldots,r_{p-q}\in\R^p$ is a basis of $\ker(M_0^\top)$,
let $h:\R^{p-q}\rightarrow\R$ be an arbitrary function satisfying
$h \circ R^\top\in\mathcal{F}$,
and define
$f^* \coloneqq f_0 + h \circ R^\top\in\mathcal{F}$
(which holds true if $\mathcal{F}$ is closed under addition).
By using the reduced form
equation for $X$ in~\eqref{eq:sem}, we have that
$R^\top X = R^\top M_0 Z + R^\top V = R^\top V$, and so the distribution of
$Y - f^*(X) = U - h(R^\top V)$ remains the same for all
$P \in \mathcal{P}_0$.
Since there may be infinitely many such functions $h$,
the set $\mathcal{I}_0$ can contain infinitely many functions in addition to
$f_0 \in \mathcal{I}_0$.
This motivates the following definition.
\begin{definition}[Invariant most predictive (IMP) function]\label{def:imp}
	Assume Setting~\ref{set:setting-1} and denote by $\mathcal{I}_0$ the set of invariant functions (see Definition~\ref{def:inv_funs}).
	We call	$f_{\imp} \in \mathcal{F}$
	an \emph{invariant most predictive (IMP)} function if
	\begin{align}\label{eq:imp}
		\risk(P_{\train}, f_\imp) = \inf_{f \in \mathcal{I}_0} \risk(P_{\train}, f).
	\end{align}
\end{definition}
The notion of invariant most predictive (IMP) function provides a constructive method to tackle the distribution generalization problem.
A potential approach to identify an IMP function $f_\imp$ is to (i) identify the set of invariant functions $\mathcal{I}_0$, and (ii) solve the optimization problem in~\eqref{eq:imp} constrained to the set $\mathcal{I}_0$.
For instance, if $\mathcal{F}$ is the class of linear functions, then one can show that any invariant function $f$ is identified by the moment condition $\Exp_{P_\train}[(Y - f(X)) Z] = 0$ and that $\mathcal{I}_0 = \{f \in \mathcal{F} \mid \Exp_{P_\train}[(Y - f(X)) Z] = 0\}$. Furthermore, one can show that any function $f$ is invariant if and only if $f(x) = f_0(x) + \delta^\top x$, where $\delta \in \ker(M_0^\top)$ \citep{jakobsen2022distributional}.

Though in the linear setting the identification and characterization of the set of invariant functions $\mathcal{I}_0$ is straightforward, this is not true in a more general setting.
In particular, when $\mathcal{F}$ is a flexible function class, invariance-based methods for distribution generalization
consider weaker notions of invariance compared to Definition~\ref{def:inv_funs} and, in some cases, do not identify a generalizing function.

\subsection{Relaxing Invariance and its Implications}

The IMP function is defined as the solution to
the constrained optimization problem~\eqref{eq:imp} over the set of invariant functions $\mathcal{I}_0$.
 Existing invariance-based methods usually tackle this problem in two steps. First, they solve a relaxation of~\eqref{eq:imp} by optimizing over a larger set $\mathcal{H} \supseteq \mathcal{I}_0$; common examples of such sets are 
 $\mathcal{I}_0 \subseteq \mathcal{H}_1 \subseteq \mathcal{H}_2 \subseteq \mathcal{H}_3$ defined as
\begin{align*}%
	\begin{split}
		\mathcal{H}_1 & := \left\{f \in \mathcal{F} \mid Y - f(X) \ind Z\ \text{under}\ P_{\train}\right\}, \\
		\mathcal{H}_2 & := \left\{f \in \mathcal{F} \mid \Exp_{P_{\train}}[Y - f(X) \mid Z] = 0 \right\},   \\
		\mathcal{H}_3 & := \{f \in \mathcal{F} \mid \Exp_{P_{\train}}[(Y - f(X))Z] = 0 \}.
	\end{split}
\end{align*}
Second, they assume that the statements defining $\mathcal{H}_1, \mathcal{H}_2, \mathcal{H}_3$
also hold for all $P \in \mathcal{P}_0$.
The set $\mathcal{H}_1$ is considered by \citep{magliacane2018domain,rojas2018invariant} in a linear unconfounded setting
and by \citep{saengkyongam2022exploiting} in a nonlinear underidentified IV setting.
The set $\mathcal{H}_2$ defines a conditional mean independence restriction and is considered
by~\citep{arjovsky2020invariant} in a nonlinear unconfounded setting.
More recently, in the same setting as~\citep{arjovsky2020invariant}, \citep{krueger21a} consider the slightly stronger restriction
$\mathcal{H}_2 \cap \{f \in \mathcal{F} \mid \Exp_{P_\train}[(Y - f(X))^2] = c \text{ for some } c\in \R\}$.
The set $\mathcal{H}_3$ imposes an unconditional moment constraint and is studied by
\citep{jakobsen2022distributional} in a linear underidentified IV setting and by~\citep{buhlmann2020invariance,christiansen2020causal} in a nonlinear underidentified IV setting.
We now investigate the relationship between the sets $\mathcal{I}_0$, $\mathcal{H}_1$, $\mathcal{H}_2$ and $\mathcal{H}_3$
by showing
that the notion of invariance proposed in this work is stronger than the stochastic independence of the residuals. In particular, we show that some functions in $\mathcal{H}_1$ can fail to be invariant.
\begin{proposition}[Invariance implies independence of residuals]\label{prop:inv-implies-ind}
	Assume Setting~\ref{set:setting-1}
	and that for all
	$z\in\mathbb{R}^r$ it holds that $\delta_z\in\mathcal{Q}_0$.
	Then, it holds that $\mathcal{I}_0 \subseteq \mathcal{H}_1.$
	Moreover, if Assumption~\ref{ass:ident0} holds, it follows that $\mathcal{H}_1 \subseteq \mathcal{I}_0$.
	If Assumption~\ref{ass:ident0} is not satisfied, then even under
	Assumption~\ref{ass:ident}, it may happen that $\mathcal{H}_1 \neq \mathcal{I}_0$.
\end{proposition}
The following example, whose construction is in spirit related to~\citep[][Example~4.4]{tsai24b}, shows that $\mathcal{H}_1 \neq \mathcal{I}_0$.

\begin{example}[Independence of residuals does not imply invariance]\label{ex:inv-stronger-ind-2}
Consider the\\ SIMDG~$(f_0, 1, \Lambda_0, \mathcal{Q}_0)$, where the vector of observed variables $(X, Y, Z) \in \R^3$ is defined as
\begin{align*}
  X = Z + V, \quad Y = f_0(X) + U,
\end{align*}
$Z \sim Q_{\train}$ is supported on $\{-2\pi, 2\pi\}$ with $\mathbb{E}_{Q_\train}[Z] = 0$ and $(U, V) \sim \Lambda_0$ follows a multivariate centered Gaussian distribution such that $V \sim N(0, \sigma^2)$, $U \sim N(0, 1 + \sigma^2)$, and $\Exp[U \mid V] = V$.
Under $P_\train$, $X$ is a Gaussian mixture with strictly positive density on $\R$; hence  Assumption~\ref{ass:ident} is satisfied.
Consider the function $f(x) \coloneqq f_0(x) + \sin(x)$.
Under $P_\train$, the residuals
\begin{align*}
  Y - f(X) = U - \sin(X) = U - \sin(Z + V) =  U - \sin(V),
\end{align*}
since $\sin(Z+V) = \sin(V)$ for any $Z \in \{-2\pi, 2\pi\}$. Therefore, we have that $Y - f(X) \ind Z$ under $P_\train$, i.e., $f \in \mathcal{H}_1$.
Now, consider the distribution $P_{\test} \in \mathcal{P}_0$ induced by replacing $Q_{\train}$ with a point mass distribution at $\pi$, i.e., $Q_{\test} = \delta_{\pi}$. Under $P_\test$, the residuals $Y - f(X) = U - \sin(\pi + V)$ are distributed differently from $U - \sin(V)$. Therefore,  $P_\train^{Y - f(X)} \neq P_\test^{Y - f(X)}$, i.e., $f \notin \mathcal{I}_0$.
\end{example}
Under the setting of Example~\ref{ex:inv-stronger-ind-2}, we show in Section~\ref{subsec:distgen} an even stronger result
than $\mathcal{H}_1 \neq \mathcal{I}_0$:
optimizing the training risk over the larger set $\mathcal{H}_1 \supseteq \mathcal{I}_0$ can yield a function that is not the~IMP.

\section{Boosted Control Functions}\label{sec:ident}

Assume Setting~\ref{set:setting-1} with the SIMDG $(f_0, M_0, \Lambda_0, \mathcal{Q}_0)$.
In the conventionally considered case in which $f_0$ is
identifiable (which implies $\rank(M_0)=p$) and
nonlinear, there are two prevalent categories of methods for
identifying $f_0$; nonparametric IV methods
\citep{newey2003instrumental} and nonlinear control function
approaches \citep{ng1995nonparametric,newey1999}.
The nonparametric IV \citep{newey2003instrumental}
identifies the structural function $f_0$ by solving the
inverse problem
\begin{align}\label{eq:2sls}
	\Exp_{P_{\train}}[Y \mid Z] = \Exp_{P_{\train}}[f_0(X) \mid Z].
\end{align}
The nonparametric IV estimator of
$f_0$ is then given as the solution to
a finite-sample version
of~\eqref{eq:2sls}, where $f_0$ is approximated by power series or splines, for example.
The nonlinear control function approach~\citep{ng1995nonparametric,newey1999} identifies the structural function $f_0$ by computing a conditional expectation of the response $Y$ given the predictors $X$ and a set of control variables $V \coloneqq X - \Exp_{P_{\train}}[X \mid Z]$. We then have
\begin{align}\label{eq:contr-fun}
	\begin{split}
		\Exp_{P_{\train}}[Y \mid X, V] = &\ \Exp_{P_{\train}} [f_0(X) + U  \mid X, V]\\
		= &\ f_0(X) + \E_{P_{\train}}[U \mid V] \\
		= &\ f_0(X) + \gamma_0(V),
	\end{split}
\end{align}
where
$\gamma_0: v \mapsto \E_{P_\train}[U \mid V = v]$ is the control function (hence the name).
Later we will assume that the control function $\gamma_0$ belongs to a class $\mathcal{G}$ of measurable functions.
In contrast to nonparametric IV methods, the
nonlinear control function approach has the advantage of identifying
the structural function $f_0$ via the conditional expectation
in~\eqref{eq:contr-fun}, and therefore can estimate $f_0$ with
flexible nonparametric estimators, such as
nearest-neighbor regression or regression trees.

In this work, we are not interested in the structural function
$f_0$ directly but in a function that achieves distribution
generalization.
The key idea
is to adapt the control function approach
above in a way that allows us
to identify the IMP. In contrast to
$f_0$, the IMP can be identifiable even in settings where $\rank(M_0) < p$.
The following example illustrates the non-identifiability of
$f_0$ based on the standard control function approach.
\begin{example}[Non-identifiability of $f_0$ and $\gamma_0$]\label{ex:ident}
	Consider a \modelname\ over the variables $(X_1,X_2, Y,
		Z)\in\R^4$ with the structural equations
	\begin{align*}
		X_1 =  & \ Z + V_1,           \\
		X_2 =  & \ Z + V_2,           \\
		Y~   = & \ f_0(X_1, X_2) + U,
	\end{align*}
	where $f_0(x_1, x_2) \coloneqq x_1$,
	$\Lambda_0$ is a zero mean Gaussian distribution with $(U, V_1,
		V_2)\sim\Lambda_0$ satisfies $\E[U\vert V_1,V_2]=V_1$ and
	$\mathcal{Q}_0$ the set of all distributions on $\R$ with full
	support. Expanding the conditional expectation and using that
	$P_{\train}$-a.s.\ it holds that
	$X_1 - X_2 + V_2 = V_1$
	we get
	\begin{align*}
		\Exp_{P_{\train}}[Y \mid X_1, X_2, V_1, V_2]
		= & \ f_0(X_1, X_2) + \gamma_0(V_1, V_2)
		=  X_1 + V_1                          \\
		= & \ 2X_1 - X_2 + V_2
		= \tilde{f}(X_1, X_2) + \tilde{\gamma}(V_1, V_2),
	\end{align*}
	where
	$\gamma_0(v_1, v_2) \coloneqq v_1$,
	$\tilde{f}(x_1, x_2) \coloneqq 2x_1 - x_2$ and
	$\tilde{\gamma}(v_1, v_2) \coloneqq v_2$.
	Therefore, while the
	conditional expectation $\Exp_{P_{\train}}[Y \mid X, V]$ is
	identifiable, the separation into the structural function $f_0$
	and the control function $\gamma_0$ is not.
\end{example}

While it may happen that several functions $f$ and $\gamma$
satisfy $P_{\train}$-a.s.\ that
$f(X) + \gamma(V) = \Exp_{P_{\train}}[Y \mid X, V]$, we will show
that any such pair $f$ and $\gamma$ can be used to construct a
specific target function
that achieves distribution generalization. To
construct this target function, let $q\coloneqq\rank(M_0)$ and define
\begin{equation}
	\label{eq:definition_R}
	R\coloneqq
	\begin{cases}
		(r_1,\ldots,r_{p-q})\quad & \text{if $q<p$}  \\
		\bf0 \quad                & \text{if $q=p$,}
	\end{cases}
\end{equation}
where $(r_1,\ldots,r_{p-q})\in\R^{p\times (p-q)}$ is
an orthonormal
basis of
$\ker(M_0^\top)$ if $q<p$ and $\mathbf{0} \in\R^{p\times 1}$ is the
zero map. The matrix $R$ allows us to extract
invariant parts of
$X$ since $R^{\top}X=R^{\top}M_0Z+R^{\top}V=R^{\top}V$, which has a
fixed distribution for all $P\in\mathcal{P}_0$. We use it to
construct the target function as follows.
\begin{definition}[Boosted control function (BCF)]\label{def:control_estimator}
	Assume Setting~\ref{set:setting-1}, define $\gamma_0(V)
		\coloneqq \Exp_{P_\train}[U \mid V]$ and $R$ as in
	\eqref{eq:definition_R}.
	Then, we define the
	\emph{boosted control function (BCF)} $f_\star\in\mathcal{F}$
	for $P_{\train}$-a.e. $x\in\R^p$ by
	\begin{equation}\label{eq:def-cf}
		f_\star(x) \coloneqq f_0(x) + \Exp_{P_{\train}}[\gamma_0(V) \mid R^\top X = R^\top x].
	\end{equation}
	In particular, if $p=\rank(M_0)$, the BCF is
	$f_{\star}(\cdot)\coloneqq
		f_0(\cdot)+\Exp_{P_{\train}}[\gamma_0(V)]= f_0(\cdot)$, since
	$\Exp_{P_{\train}}[U]=0$.
\end{definition}
The BCF $f_\star$ is motivated by the IMP introduced in Definition~\ref{def:imp}. In particular, $f_\star$
is defined as the sum of the structural function $f_0$ and a term
depending on $X$ that is as predictive as possible and, at the same
time, invariant.
The name boosted control function alludes to
the second component which extracts invariant predictive information
from the remainder term $\gamma_0(V)$. In the following section, we
provide conditions under which the control function is identifiable.

\subsection{Identifiability Conditions}\label{sec:ident-conditions}

For the BCF $f_{\star}$ to be useful, we first need to ensure
that it is indeed identifiable from the training distribution
$P_{\train}$, that is, if $f_0$ and $\gamma_0$ are replaced by any
other functions $f$ and $\gamma$ satisfying
$f(X) + \gamma(V)=\Exp_{P_{\train}}[Y \mid X, V]$ in
\eqref{eq:def-cf} the BCF $f_{\star}$ does not change. Formally,
identifiability is defined as follows.
\begin{definition}[Identifiability of \cfname]\label{def:ident-cf}
	Assume Setting~\ref{set:setting-1} and define $R$ as in
	\eqref{eq:definition_R}.
	Assume $f_0 \in \mathcal{F}$ and $\gamma_0 \in \mathcal{G}$, where $\mathcal{F}$ and $\mathcal{G}$ are classes of measurable functions.
	The
	BCF $f_\cf$ is identifiable (with respect to $\mathcal{F}$ and $\mathcal{G}$)
	from the observational distribution
	$P_{\train}$
	if and only if
	for all
	$f \in \mathcal{F}$ and $\gamma \in \mathcal{G}$
	the following statement holds
	\begin{align*} %
		&\quad\quad\quad\quad\Exp_{P_{\train}}[Y \mid X, V] = f(X) + \gamma(V),\
		P_{\train}\text{-a.s.}\\
    & \quad\quad\quad\quad\quad\quad\quad\quad\quad\quad\quad\implies\\
		& \begin{cases}
			f_\cf(X) = f(X) +
			\Exp_{P_{\train}}[\gamma(V) \mid R^\top X] \ P_{\train}\text{-a.s.,}\, & \text{if }q<p,  \\
			f_\cf(X) = f(X) +
			\Exp_{P_{\train}}[\gamma(V)] \ P_{\train}\text{-a.s.}\,               & \text{if }q=p.
		\end{cases}
	\end{align*}
\end{definition}
An equivalent definition of identifiability of $f_\star$ states that for any additive function $f(X) + \gamma(V)$
satisfying~\eqref{eq:contr-fun} the difference $f(X) - f_0(X)$ can be
written as a function $\delta(R^\top X)$ depending on $X$ only via the
null space of $M_0^\top$.
This is formalized in the following proposition.
\begin{proposition}[Equivalent condition for identifiability]
	\label{prop:ident-cf}
	Assume Setting~\ref{set:setting-1}, let
	$q = \rank(M_0)$ and define $R$ as in
	\eqref{eq:definition_R}.
	The \cfname\ $f_\cf$ is identifiable
	from the observational distribution $P_{\train}$ if and only if
	for all measurable functions $h, g:\R^p\rightarrow\R$ the
	following statement holds,
	\begin{align}\label{eq:ident-cf}
  \begin{split}
		&\quad\quad\quad\quad\quad\quad h(X) + g(V) = 0,\ P_{\train}\text{-a.s.}\\
		&\quad\quad\quad\quad\quad\quad\quad\quad\quad\quad\quad\implies\\
		&\begin{cases}
			\exists \delta: \R^{p - q} \to \R\ \text{such that}\
			h(X) = \delta(R^\top X),\ P_{\train}\text{-a.s.,}\,                            & \text{if }q<p,  \\
			\exists c\in\mathbb{R} \text{ such that } h(X) = c,\ P_{\train}\text{-a.s.,}\, & \text{if }q=p.
		\end{cases}
  \end{split}
	\end{align}
\end{proposition}
A proof can be found in
Appendix~\ref{proof:ident-cf}.
Proposition~\ref{prop:ident-cf} can be seen as an extension of~\citet{newey1999}'s identifiability condition to the
underidentified setting:
\citet{newey1999} consider only the case $q=p$; a further difference to that work is that our goal
is to achieve distribution generalization rather than identifying the structural function $f_0$.
Identifiability of the BCF $f_\cf$ depends
on the assumptions we are willing to make on the class of structural functions $f$ and the control functions $\gamma$.
Depending on whether the exogenous variable $Z$ is categorical or continuous, we now provide two different sufficient conditions  for identifiability of the \cfname.
\begin{assumption}[Categorical $Z$ and linear $\gamma_0$]\label{ass:linear-cf}
	Let $\mathcal{G}$ be the class of linear functions.
	The exogenous variable $Z$ is categorical with values in a finite subset $\mathcal{Z}\subseteq\R^r$ (see, e.g., the construction proposed in Remark~\ref{rmk:cat-z})
	such that
	for all $z \in \mathcal{Z}$ $P_{\train}(Z=z) > 0$,
	$\Exp_{P_\train}[Z] = 0$ and $\E_{P_\train}[ZZ^\top] \succ 0$
	(which implies that $Z$ takes at least $r+1$ different values).
	Moreover,
	the control function $\gamma_0: \R^p \to \R$ is linear.
	Let $q = \rank(M_0)$ and
	assume that there exists $\{z_1, \tilde{z}_1,
		\ldots,z_q,\tilde{z}_q\} \subseteq \mathcal{Z}$
	such that
	for all $j \in \{1, \dots, q\}$ the distributions $P_\train^{V +
			M_0 z_j}$ and $P_\train^{V + M_0 \tilde{z}_j}$ are not
	mutually singular\footnote{Two probability measures $\mu$, $\nu$ on a space $(\Omega, \mathbb{F})$ are mutually singular if there exists a measurable set $E \in \mathbb{F}$ such that $\mu(E^c) = \nu(E) = 0$.}, and $\vspan(\{M_0(z_j-\tilde{z}_j)\mid j\in\{1,...,q\}\})=\col(M_0)$.
\end{assumption}
Assumption~\ref{ass:linear-cf} ensures that we observe the same realization of the predictor $X$ under distinct environments, and these environments span a space that is rich enough.
The following assumption is a slight modification of~\citep[][Theorem~2.3]{newey1999} with the difference that we do not require
$\rank(M_0) = p$ since we allow for underidentified settings, too.
\begin{assumption}[%
		Differentiable $f_0$ and $\gamma_0$%
	]\label{ass:diffble-cf}
	Let $\mathcal{F}$ and $\mathcal{G}$ each be the class of differentiable functions.
	The boundary of the support of $(V, Z)$
	has zero probability under $P_\train$
	and the interior of the
	support of $(V, Z)$ is convex
	under $P_{\train}$ (this implies that $Z$ is not discrete).
	Furthermore, the control and structural functions $\gamma_0, f_0: \R^p \to \R$ are differentiable.
\end{assumption}
Under either Assumption~\ref{ass:linear-cf} or~\ref{ass:diffble-cf}, the observational distribution $P_{\train}$ contains enough information to identify the \cfname\ $f_\cf$.
\begin{proposition}[\cfname\ $f_\star$ is identifiable]\label{prop:identification}
	Assume Setting~\ref{set:setting-1}.
	Suppose that either Assumption~\ref{ass:linear-cf} or Assumption~\ref{ass:diffble-cf} holds.
	Then, the \cfname\ $f_\cf$ is identifiable from $P_{\train}$.
\end{proposition}
A proof can be found
in Appendix~\ref{proof:identification}.
Algorithm~\ref{alg:bcf}
provides a procedure for identifying 
the \cfname{}
$f_{\star}$
from the observed distribution if it is identifiable.
\begin{algorithm}[t]
	\caption{Boosted control function identification}
	\label{alg:bcf}

	\begin{algorithmic}[1]
		\Require Observational distribution $P_{\train}^{X, Y, Z}$ over $(X, Y, Z)$.

		\Ensure \cfname\ $f_\cf: \R^p \to \R$.

		\State Compute conditional expectation $\Exp_{P_{\train}}[X \mid Z] = M_0Z$ to identify  $M_0 \in \R^{p \times r}$. \label{alg:line-1}

		\State  Compute the control variable $V = X - M_0 Z$ and a basis $R \in \R^{p \times (p-q)}$ for the left null space $\ker(M_0^\top)$. \label{alg:line-2}

		\State  Compute the additive conditional expectation
		$\Exp_{P_{\train}}[Y \mid X, V] = f_0(X) + \gamma_0(V)$.
		\label{alg:line-3}

		\State
		Compute the conditional expectation
		$\Exp_{P_{\train}}[\gamma_0(V) \mid R^\top X]$.
		\label{alg:line-4}

		\State Return
		$f_\cf(x) \coloneqq f_0(x) +
			\Exp_{P_{\train}}[\gamma_0(V) \mid R^\top X = R^{\top}x]$.
	\end{algorithmic}
\end{algorithm}
We will now see that $f_\star$ indeed comes with generalization guarantees and is invariant most predictive.

\subsection{Generalization Guarantees}\label{subsec:distgen}

In this section, we study under which conditions
the \cfname\ $f_{\cf}$ is a generalizing function, that is
\begin{align*}%
	\sup_{P \in \mathcal{P}_0} \risk(P, f_{\cf}) = \inf_{f \in \mathcal{F}} \sup_{P \in \mathcal{P}_0} \risk(P, f),
\end{align*}
where $\mathcal{P}_0$ denotes the set of distributions induced by
the \modelname\
as defined in~\eqref{eq:mathcalP}.
Clearly, if $\mathcal{P}_0$ only contains $P_{\train}$, then the least squares solution is optimal in terms of distribution generalization.
In particular,
it
is possible to quantify the additional risk of the BCF $f_\cf$ on the training
distribution $P_{\train}$ compared to the least squares prediction
$f_\ols\colon x\mapsto \Exp_{P_{\train}}[Y \mid X=x]$.
\begin{proposition}\label{prop:cf-imp-are-imp}
	Assume Setting~\ref{set:setting-1}, let
	$\gamma_0(V) = \Exp_{P_\train}[U \mid V]$ and define $R$ as in
	\eqref{eq:definition_R}
	Then, the \cfname\
	$f_{\cf}$ is invariant according to
	Definition~\ref{def:inv_funs}. Moreover, the
	mean squared
	error
	risks of $f_\cf$ and the least squares predictor
	$f_\ols$ are related
	by
	\begin{align}
		\label{eq:risk_comparison_bcf_ols}
		\risk(P_{\train}, f_{\cf})=\risk(P_{\train}, f_{\ols})+\Exp_{P_{\train}}
		\left[
			\left(
			\Exp_{P_{\train}}[\gamma_0(V)\vert R^\top X]-\Exp_{P_{\train}}[\gamma_0(V)\vert X]
			\right)^2
			\right].
	\end{align}
\end{proposition}
A proof can be found in
Appendix~\ref{proof:cf-imp-are-imp}.
The additional term in \eqref{eq:risk_comparison_bcf_ols}, is the
price one needs to pay in terms of prediction performance to be
invariant. That is, if $P_{\test}=P_{\train}$ then the BCF is
outperformed by the least squares predictor by exactly this term. In
contrast, if there is sufficient heterogeneity in $\mathcal{P}_0$, then
the BCF $f_\star$ is a generalizing function.
Formally, we use the following assumption to
quantify what is meant by sufficient heterogeneity.
\begin{assumption}\label{ass:condexp}
	Assume Setting~\ref{set:setting-1}, let $\gamma_0(V)
		= \Exp_{P_\train}[U \mid V]$ and define $R$ as in \eqref{eq:definition_R}.
	Then, assume that
	\begin{align*}
		\inf_{P \in \mathcal{P}_0} \Exp_P \left[\left( \Exp_P[\gamma_0(V) \mid R^\top X] - \Exp_{P}[\gamma_0(V) \mid X] \right)^2 \right]  = 0.
	\end{align*}
\end{assumption}
Whether this assumption holds depends on the class of distributions
$\mathcal{P}_0$ and on the control function $\gamma_0$.
From the SCM perspective, the
assumption requires that
the interventions
on $Z$ are strong enough so that the only part of $X$ that is relevant to
explain the confounder $V$ is the invariant one, i.e.,
$R^\top X$.
Assumption~\ref{ass:condexp} holds, for example, when the noise terms are
jointly Gaussian or when
the control function $\gamma_0(V)$ is bounded, as long as the class of
distributions $\mathcal{Q}_0$ contains standard Gaussian
distributions
with arbitrarily large variance.
\begin{proposition}\label{prop:ass-5}
	Assume Setting~\ref{set:setting-1} and suppose that
	\begin{equation*}
		\{ N(0, k^2I_r)\mid k\in\N \}\subseteq \mathcal{Q}_0.
	\end{equation*}
	Suppose the joint distribution $\Lambda_0$ of $(U, V)$  satisfies one of the following conditions.
	\begin{enumerate}
		\item[(a)] $V$ has a density w.r.t.\
			Lebesgue, and the control function
			$\gamma_0(V) = \Exp_{P_{\train}}[U \mid V]$ is almost
			surely bounded.
		\item[(b)]
			$\Lambda_0$
			is a multivariate centered Gaussian distribution (with a non-degenerate covariance matrix).
	\end{enumerate}
	Then, Assumption~\ref{ass:condexp} holds.
\end{proposition}
A proof can be found in Appendix~\ref{proof:ass-5}.
We thank Alexander M. Christgau for pointing us to
their trigonometric argument in~\citet{christgau2023} that turned
out to be helpful to prove part (a)
of Proposition~\ref{prop:ass-5}.
Condition~(a) in Proposition~\ref{prop:ass-5}
allows for a wide range of distributions, including
those inducing bounded
random variables with densities, for example.
\begin{theorem}
	\label{thm:minimax-1}
	Assume Setting~\ref{set:setting-1}.
	Let $f_{\cf} \in \imp$ denote the \cfname.
	Suppose that Assumptions~\ref{ass:ident} and~\ref{ass:condexp} hold. Then,
	\begin{align*}%
		\sup_{P \in \mathcal{P}_0} \risk(P, f_{\cf}) = \inf_{f \in \mathcal{F}} \sup_{P \in \mathcal{P}_0} \risk(P, f).
	\end{align*}
\end{theorem}
A proof can be found in Appendix~\ref{proof:minimaxthm-1}.
As a corollary of Theorem~\ref{thm:minimax-1}, we have that the
\cfname\ $f_\cf$ is indeed the IMP.
\begin{cor}\label{cor:imp}
	Assume Setting~\ref{set:setting-1}
	and suppose that
	Assumptions~\ref{ass:ident} and~\ref{ass:condexp} hold.
	Let
	$f_{\cf} \in \imp$ denote the \cfname.
	Then,
	\begin{align*}%
		\risk(P_{\train}, f_\cf) = \inf_{f \in \mathcal{I}_0} \risk(P_{\train}, f).
	\end{align*}
\end{cor}
A proof can be found in Appendix~\ref{proof:imp}.
We conclude the section by revisiting Example~\ref{ex:inv-stronger-ind-2}, which showed that the set $\mathcal{H}_1 = \{f \in \mathcal{F} \mid Y - f(X) \ind Z\ \text{under}\ P_{\train}\}$ may contain non-invariant functions.
Here, we provide an even stronger result, namely that optimizing the training risk over $\mathcal{H}_1 \supseteq \mathcal{I}_0$ may not yield the IMP.
\begin{example}[Optimizing the training risk over  $\mathcal{H}_1$ may not return the IMP]\label{ex:opt-h1-not-imp}
  ~\\  Consider the setting of Example~\ref{ex:inv-stronger-ind-2} and recall that the structural function $f_0$ is invariant, i.e., $f_0 \in \mathcal{I}_0$.
  By definition, the BCF coincides with the structural function $f_0$.
  As we argued in Example~\ref{ex:inv-stronger-ind-2}, this
  setup satisfies Assumption~\ref{ass:ident}. 
  Moreover, it satisfies Assumption~\ref{ass:linear-cf}: $Z$ is a centered, non-degenerate discrete random variable, the function $v \mapsto \gamma_0(v) \coloneqq v$ is linear, 
  $P_\train^{V + 2\pi} = N(2\pi, \sigma^2)$ and $P_\train^{V - 2\pi} = N(-2\pi, \sigma^2)$ are not
	mutually singular, and $\vspan(\{2\pi - (-2\pi)\})=\R$.
  Hence, by Proposition~\ref{prop:identification}
  the BCF $f_0$ is identifiable from $P_\train$.

  Suppose that the set of perturbed distribution satisfies $\{N(0, k^2) \mid k \in \N\} \subseteq \mathcal{Q}_0$.
  Then, by Proposition~\ref{prop:ass-5}, Assumption~\ref{ass:condexp} is fulfilled and by Corollary~\ref{cor:imp} the BCF $f_0$ is the IMP function.

  Now, assume that the noise variance of $V$ is $\sigma^2 = 2$. We show that the non-invariant function $f(x) \coloneqq f_0(x) + \sin(x) \in \mathcal{H}_1\setminus \mathcal{I}_0$ achieves a lower risk than $f_0$; hence, optimizing over $\mathcal{H}_1$
  does
  not yield the IMP.
  By Lemma~\ref{lem:mse-sin},
  \begin{align*}
	\risk(P_\train, f) = 1 + \sigma^2 + \frac{(1 - e^{-2\sigma^2})}{2} - 2\sigma^2e^{-\sigma^2/2}.
  \end{align*}
  For $\sigma^2 = 2$, one can verify that 
  $\risk(P_\train, f) < 1 + \sigma^2 = \risk(P_\train, f_0)$, 
  and therefore
  \begin{align*}
	\inf_{f \in \mathcal{H}_1} \risk(P_\train, f) < \risk(P_\train, f_0).
  \end{align*}
\end{example}

\subsection{Estimating Boosted Control Functions}\label{sec:bcf-est}

Consider Setting~\ref{set:setting-1} and let
$(X_1, Y_1,Z_1),\ldots,(X_n,Y_n,Z_n)$ be the i.i.d.\ training sample
drawn from the training distribution $P_{\train}\in\mathcal{P}_0$. We denote by
$\bfX \in \R^{n \times p}$, $\bfY \in \R^{n}$,
$\bfZ \in \R^{n \times r}$ the respective design matrices resulting
from row-wise concatenations of
the observations.
The algorithm to identify the BCF from $P_{\train}$,
Algorithm~\ref{alg:bcf}, consists of two parts. First
(lines~\ref{alg:line-1}--\ref{alg:line-2}), it computes the matrix
$M_0 \in \R^{p \times r}$, the control variables
$V \coloneqq X - M_0 Z$, and a basis $R \in \R^{p \times (p - q)}$
for the null space of $M_0^\top$.  Second
(lines~\ref{alg:line-3}--\ref{alg:line-4}), it computes
the additive function
$(x, v) \mapsto \Exp_{P_{\train}}[Y \mid X = x, V = v]$ and the
resulting BCF.
We now
convert each part into an
estimation procedure.

For the first part of Algorithm~\ref{alg:bcf},
we consider the framework of reduced-rank regression
\citep{reinsel1998} to estimate the matrix $\hat{M}_0$ and its rank
$\hat{q} \leq \min\{p, r\}$,
since $M_0 \in \R^{p \times r}$ is not
necessarily
full rank.
More precisely,
we adopt the rank selection criterion proposed by
\citet{bunea2011}.
For a fixed penalty $\lambda > 0$, they propose to estimate $M_0$ by
\begin{align*}%
	\hat{M}_0(\lambda) \coloneqq \argmin_{M} \ \left\{ \norm{\bfX - \bfZ M^\top }_F^2 + \lambda\ \rank(M)\right\}.
\end{align*}
To do so, they observe
\begin{align*}%
	\min_{M} \ \left\{ \norm{\bfX - \bfZ M^\top }_F^2 + \lambda\ \rank(M)  \right\}
	=
	\min_k \left\{ \min_{\substack{M,\ \rank(M) = k}} \left[ \norm{\bfX - \bfZ M^\top}_F^2 + \lambda k \right] \right\},
\end{align*}
which suggests that
one can estimate a rank-$k$ matrix
$\hat{M}_{0, k}(\lambda)$ for each
$k \in \{1, \dots, \min\{p, r\}\}$ and choose the optimal
$\hat{q}(\lambda) \coloneqq \argmin_k\|\bfX - \bfZ
	\hat{M}_{0,k}(\lambda)^\top\|_F^2 + \lambda k$.  As mentioned by
\citet{bunea2011}, $\hat{M}_{0, k}(\lambda)$ and $\hat{q}(\lambda)$
can be computed in closed form and efficiently. Moreover, one can
select the optimal
$\lambda^*$
by
cross-validation.
This results in the final
estimators
$\hat{M}_0 \coloneqq \hat{M}_{0,
	\hat{q}(\lambda^*)}(\lambda^*)$ and
$\hat{q}\coloneqq\hat{q}(\lambda^*)$.
\citet{bunea2011} show that their rank selection criterion consistently recovers the true rank $q$ of $M_0$ if $V$ is multivariate standard Gaussian and if the $q$-th largest singular value of the signal $M_0 Z$ is well separated from the smaller ones.
Based on these estimators $\hat{M}_0$ and $\hat{q}$, we then
estimate the control variables
$\hat{\bfV} \coloneqq \bfX - \bfZ \hat{M}_0^\top$ and a basis of
the null space by
$\hat{R} \coloneqq (r_1, \dots, r_{\hat{q}}) \in \R^{p \times (p -
		\hat{q})}$, where $r_1, \dots, r_{p - \hat{q}}$ are the
$p - \hat{q}$ left singular vectors of $\hat{M}_0$ associated to
zero singular values.

For the second part,
given the estimated control variables
$\hat\bfV$ and the basis of the
estimated
null space $\hat{R}$,
we then estimate the BCF. As shown in
Algorithm~\ref{alg:bcf}, the BCF is obtained by performing two
separate regressions
such that
\begin{itemize}
  \item[(i)] $(x,v)\mapsto \hat{f}(x) + \hat{\gamma}(v)$
		estimates
		$(x,v)\mapsto\Exp_{P_\train}[Y \mid X = x, V = v]$,
  \item[(ii)] $x\mapsto \hat{\delta}(\hat{R}^{\top}x)$
		estimates
		$x\mapsto\Exp_{P_\train}[\gamma_0(V) \mid R^\top X = R^\top x]$.
\end{itemize}

For the additive function in step (i), we
devise a practical estimator inspired by the alternating conditional
expectation (ACE) algorithm by \cite{breiman1985estimating}.  Unlike
\cite{breiman1985estimating}, here we do not assume full additivity of
the functions, but allow $\tilde{f}$ and $\tilde{\gamma}$ to belong to general
function classes $\mathcal{F}$
and $\mathcal{G}$, respectively.
The procedure estimates $\tilde{f}$
and $\tilde\gamma$ by alternating between the two regressions. Formally,
we assume we have two arbitrary nonparametric regression methods
that result in estimates $\hat{f}\in\mathcal{F}$ and
$\hat{\gamma}\in\mathcal{G}$. We then start by estimating
$\hat{\gamma}$ based on a regression of
${\bfY}-\frac{1}{n}\sum_{i=1}^nY_i$ on $\hat{\bfV}$ after which we
estimate $\hat{f}$ based on a regression of
${\bfY}-\hat{\gamma}(\hat{\bfV})$ on ${\bfX}$, and then iterate this
after updating $\bfY$
at each step
(see Algorithm~\ref{alg:controltwicing}).
We call the algorithm
ControlTwicing; the first part of the name alludes to the fact that we
deal with a regression problem in a control function setup.
The second part of the name refers to the twicing idea \citep[see][Chapter~16]{tukey1977exploratory} consisting of fitting an additive model over repeated iterations.
Once we estimate
$\hat{f}$ and $\hat{\gamma}$,
in step (ii), we again use a nonparametric regression procedure (generally the same as the one used in step (i) to estimate $\hat{f}$) to obtain an estimate $\hat\delta$.
More
specifically, we regress the pseudo response
$\hat{\gamma}(\hat{\bfV})$ on the pseudo covariates $\bfX \hat R$.
The final estimated BCF is then defined
for all $x \in \R^p$ by $\hat{f}_\star(x) \coloneqq \hat{f}(x) + \hat{\delta}(\hat{R}^{\top}x).$

In Proposition~\ref{prop:bcf-rates} in Appendix~\ref{sec:app_rates}, we show that when using sampling-splitting the overall convergence rate of the BCF estimator is at least as fast as the slower of the two rates from steps~(i) and~(ii). For simplicity, the result assumes the first-step quantities ($M_{0}$, $R$, $V$) are known.
\begin{algorithm}[t]
	\caption{ControlTwicing}
	\label{alg:controltwicing}
	\begin{algorithmic}[1]
		\Require data $(\bfX, \bfY, \hat{\bfV}) \in \R^{n \times (p + 1 + p)}$; nonparametric regressors $\hat{f}\in\mathcal{F}$ and $\hat{\gamma}\in\mathcal{G}$; max passes~$J$.

		\Ensure  $\hat{f} \in \mathcal{F}$, $\hat{\gamma} \in
			\mathcal{G}$.

		\State $\widetilde{\bfY} \gets {\bfY} - \bar{\bfY}$

		\For{$j\in\{1, \dots, J\}$}

		\State Estimate $\hat{\gamma}$ based on $\widetilde{\bfY}\sim \hat{\bfV}$

		\State $\widetilde{\bfY} \gets {\bfY} - \hat{\gamma}(\hat{\bfV})$

		\State Estimate $\hat{f}$ based on $\widetilde{\bfY}\sim {\bfX}$

		\State $\widetilde{\bfY} \gets {\bfY} - \hat{f}({\bfX})$

		\EndFor

	\end{algorithmic}
\end{algorithm}

\section{Numerical Experiments} \label{sec:num-res}

We now study the properties of the estimated BCF on simulated
data. In our first experiment, we analyze how well the BCF and the oracle IMP generalize to testing distributions.
In the second experiment, we consider how
well the reduced-rank regression estimates the matrix $M_0$.
The last experiment considers the California housing dataset \citep{pace1997sparse}
to evaluate the robustness of the BCF estimator to distributional shifts compared to the least squares estimator.
The code to reproduce the results can be found at \url{https://github.com/nicolagnecco/bcf-numerical-experiments}.

\subsection{Experiment 1: Predicting unseen interventions}\label{sec:experiment-1}
In the first experiment, we assess the predictive performance of the BCF estimator
compared to the
(oracle) IMP and the least squares (LS) estimator in a fixed SIMDG.
We measure the predictive performance by the mean squared error (MSE) between the response and the predicted values.
The BCF is estimated using a random forest
with 100 fully-grown regression trees for $\hat{f}$ and  $\hat{\delta}$
and
ordinary linear least squares for $\hat{\gamma}$. The number of passes in the ControlTwicing algorithm (see Algorithm~\ref{alg:controltwicing}) is set to $J = 2$.
The LS method
uses
a
random forest (with the same parameters) to regress ${\bfY}$ on
${\bfX}$.
The theoretical IMP corresponds to the population version of the BCF (see Corollary~\ref{cor:imp}) and can be computed in closed form.
As additional baselines, we consider the MSE of
predicting $Y$ with its unconditional mean (constant mean
estimator)
and with the true structural function.
We generate data from a
SIMDG $(f_0, M_0, \Lambda_0, \mathcal{Q}_0)$ with the
following specifications.  The function $f_0: \R^p \to \R$ is a
decision tree depending on
a subset of the predictors
(further details on how the tree is sampled can be found in Appendix~\ref{app:experiment-1}).
The matrix $M_0 \coloneqq A B^\top \in \R^{p \times r}$ is a rank-$q$
matrix where $A \in \R^{p \times q}$ and $B \in \R^{r \times q}$ are
orthonormal matrices sampled from the Haar measure, i.e., the uniform distribution over orthonormal matrices.
The singular
values of $M_0$ are $\sigma_1 = \dots = \sigma_q = 1$ and $\sigma_{q + 1} = \dots = \sigma_{\min\{p, r\}} = 0$.
The distribution $\Lambda_0$ over $\R^{p + 1}$ is a mean-zero Gaussian
with $\E[VV^\top] = I_p$, $\Exp[U^2] = c^2 + 0.1^2$ and
$\Exp[VU] = c \eta$, where $c > 0$ denotes the confounding strength
and $\eta \in \R^p$ is a vector sampled uniformly
on the unit sphere.
Here, we set the number of predictors to $p = 10$, the number of
exogenous variables to $r = 5$,
and the confounding strength to $c =
	2$. Finally, we define the set of distributions over the exogenous
variables by
$\mathcal{Q}_0 \coloneqq \{N(0, k^2 I_r)\mid k > 0\}$ and
define
$Q_\train \coloneqq N(0, I_r)$ and
$Q_\test^k \coloneqq N(0, k^2 I_r)$ for all $k \geq
	1$.
Therefore, $k$ specifies the perturbation strength
relative to the training distribution.

Figure~\ref{fig:sim1}
displays the MSE of the BCF estimator and the competing methods,
averaged over $50$
repetitions for different perturbation
strengths $k \in \{1, \dots, 10\}$.  For each
repetition, we generate random instances of $f_0$ and $M_0$
and then generate $n=1000$ i.i.d.\
observations
$(X_1,Y_1,Z_1),\ldots,(X_n,Y_n,Z_n)$ from the model
$(f_0, M_0, \Lambda_0, Q_{\train})$ on which we train each
method. For all $k\in\{1,\ldots,10\}$, we then generate $n=1000$
i.i.d.\ observations $(X_1,Y_1), \ldots, (X_n,Y_n)$ from the model
$(f_0, M_0, \Lambda_0, Q_{\test}^k)$, which we use to evaluate
MSE.
Each panel consists of a different
value of $q = \rank(M_0)$.
For all values of $q$,
the results demonstrate that the BCF estimator
indeed performs similarly to the theoretical IMP and its
performance remains approximately invariant even for large
perturbations.
Moreover, the BCF estimator outperforms the
true structural function since it
further uses all the signal in $\gamma_0(V)$ that is invariant to
shifts in the exogenous variable.
As the value of $q$ increases, for large perturbation strength $k$,
the gain in predictive performance
of BCF over LS becomes more pronounced. This
is because $q$ determines the dimension of the subspace in the
predictor space where perturbations occur. In the case where $q = p$,
the perturbations can occur in any direction in the
predictor space, and therefore the performance of
non-invariant methods deteriorate
significantly.
Moreover, for increasing values of $q$, the BCF estimator and the theoretical IMP converge to the structural function $f_0$.
This is
because the dimension $p - q$ of the invariant space $\ker(M_0^\top)$
decreases as $q$ increases; in the limit case when $q = p$, we have
that $\ker(M_0^\top)= \{0\}$ and therefore the BCF corresponds
to the structural function.
Finally, as the value of $q$ increases, when $k = 1$ and thus the training and testing distribution are the same, the LS estimator performs better than the BCF estimator and the theoretical IMP.
This behavior is expected because the LS estimator can always use all the
information in $X$ to predict $\gamma_0(V)$, while the BCF and IMP
can only use the information in the invariant space $R^\top X$ to predict $\gamma_0(V)$
and the dimension of $R^\top X$
decreases  for increasing $q$.

\begin{figure}[t]
	\centering
	\includegraphics[scale=0.85]{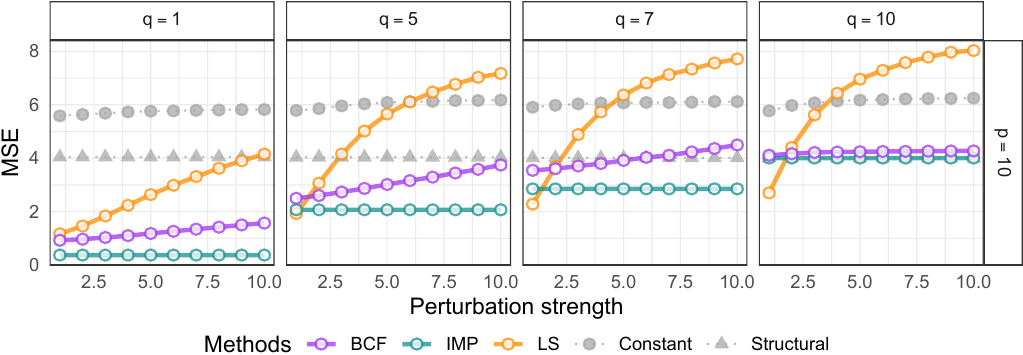}
	\caption{
		MSE between the response and the predicted values for increasing perturbation strength over different values of $q = \rank(M_0)$. 
		We fix the number of predictors to $p = 10$, the number of exogenous variables to $r = 5$,
		and the confounding strength to $c =
			2$.
		Each point is an average over $50$
		repetitions.
		The BCF estimator achieves a performance close to the theoretical
		optimally achievable error given by the IMP and is in particular
		better than the structural function. The LS baseline is clearly
		non-invariant and performs worse than the competing methods when the perturbation in the testing
		distribution is large.
	}
	\label{fig:sim1}
\end{figure}

\subsection{Experiment 2: Estimating $M_0$}
In the second experiment, we assess the performance of the rank
selection criterion to estimate a low-rank matrix $\hat{M}_0$ and its
left null space $\ker(M_0^{\top})$.  Given the
matrix of the observed predictors $\bfX$ and exogenous variables
$\bfZ$, we estimate a low-rank matrix
$\hat{M}_0 \in \R^{p \times r}$ that is  subsequently used to
compute a basis for $\ker(M_0^\top)$
and the control variables $\hat{\bfV}$.  The
hardness of estimating the
column and null spaces of a matrix depends (among other things) on its eigengap,
which
is defined as the size of the smallest non-zero singular value.
When the
eigengap of a matrix is close to zero, it is hard to disentangle its
column space from its left null space \citep[see][]{wainwright2019,
	cheng2021}
because eigenvectors
associated with small singular values are sensitive
to small perturbations of the
original matrix;
intuitively, small
estimation errors in $\hat{M}_0$ are amplified
in the estimation of the column and left null space.  In this
experiment, we sample a rank-$q$ matrix
$M_0 \coloneqq \tau A B^T \in \R^{p \times r}$, where
$A \in \R^{p \times q}$ and $B \in \R^{r \times q}$ are orthonormal
matrices sampled from the Haar measure.
The singular
values of $M_0$ are $\sigma_1, \dots, \sigma_q = \tau$ and
$\sigma_{q + 1}, \dots, \sigma_{\min\{p, r\}} = 0$, so that its
eigengap is $\tau$.
For this experiment, we only need
observations of $X$ and $Z$, which we generate by first
sampling
$Z \sim N(0, I_r)$ and
$V \sim N(0, I_p)$ independently and then
generating
the covariates according
to $X = M_0 Z + V$.

Figure~\ref{fig:sim2} displays the distance between the column space of the true matrix $M_0$ and the estimated matrix $\hat{M}_0$, defined as $||\Pi_{M_0} - \Pi_{\hat{M}_0}||_F^2$, averaged over $50$ repetitions for increasing values of the eigengap and sample size.
We estimate $\hat{M}_0$ using the rank selection criterion defined in Section~\ref{sec:bcf-est}.
We fix the true rank to be $q = 5$ and consider different combinations of number of
covariates $p$ and exogenous variables $r$ (which are shown in the
four panels).  As expected,
for larger values of the
eigengap $\tau$, the distance between the true and estimated linear subspaces converges to zero much faster, compared to smaller values of
$\tau$.
\begin{figure}[t]
	\centering
	\includegraphics{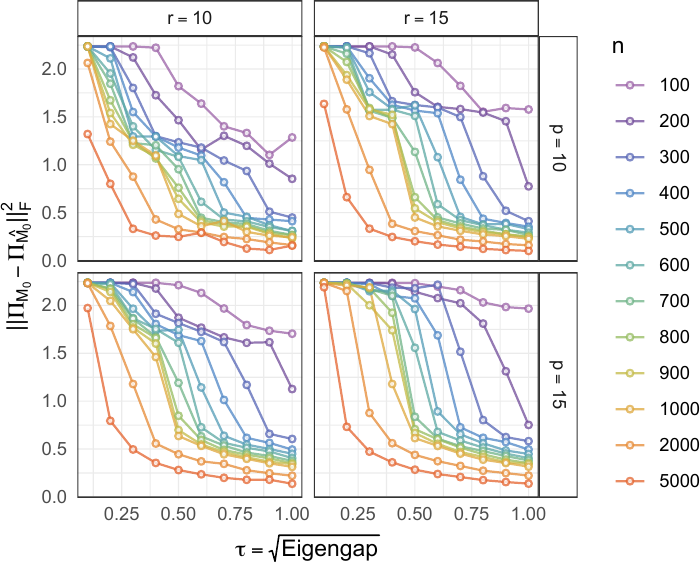}
	\caption{
		Average distance between the column space of the true matrix $M_0$ and estimated $\hat{M}_0$
		for increasing eigengap $\tau$ and sample size $n$ over different number of
		covariates $p$ and exogenous variables $r$.
		We fix
		$q =\rank(M_0) = 5$.
		Each point is an average over $50$
		repetitions.
		For large values of
		the eigengap $\tau$, the distance between the linear subspaces converges to zero much faster, compared to smaller values of $\tau$.
	}
	\label{fig:sim2}
\end{figure}

\subsection{California Housing Dataset}\label{sec:california-housing}
We consider the California housing dataset \citep{pace1997sparse} consisting of 20,640 observations derived from the 1990 U.S. census.
The unit of analysis is a block group, which is the smallest
geographical denomination for which the U.S. Census Bureau publishes
sample data. The primary aim
of our experiment
is to predict the median house value from
the following covariates: median income, median house age,
average number of rooms per household,
average number of bedrooms per household,
total population,
average number of household members, and
average annual temperature between 1991 and 2020.
We use latitude and longitude as exogenous variables.
The reason for excluding latitude and longitude from the set of predictors is that when extrapolating the model to new regions,
latitude and longitude do not help with the prediction unless assumptions are made about the function
class $\mathcal{F}$ (see Assumption~\ref{ass:ident}).
The temperature data are integrated from PRISM \citep{prism2020} and
are used as
a three-decade average.
While data from 1990 are unavailable,
we conjecture that this thirty-year range adequately reflects typical weather patterns correlated with housing prices.

To evaluate the ability of the boosted control function (BCF) estimator to handle distributional shifts, we create multiple training/testing geographical splits of the data. These splits, defined by the 35th parallel north and the 120th meridian west, partition the dataset into regions such as N/S, E/W, and SE/(N+SW).
For instance, SE/(N+SW) means that the models are trained on the SE region and evaluated on the union of the N and SW regions.
For each split: (i) we randomly select $80\%$ of the training split to train the methods; (ii) we evaluate the methods on the held-out $20\%$ of the training split and the entire testing split; (iii) we repeat~(i) and~(ii) 10 times with different random subsamples.

The BCF estimator is configured using extreme gradient boosting \citep{tianqi2016}. The learning rate for estimating $\hat{f}$, $\hat{\gamma}$, and $\hat{\delta}$ is set to $0.05$, and the number of passes in the ControlTwicing algorithm (see Algorithm~\ref{alg:controltwicing}) is fixed to $J = 10$.
As competing methods, we consider GroupDRO \citep{sagawa2019distributionally}, which estimates a predictive function that minimizes the worst-case risk across multiple training distributions, and Anchor Boosting \citep{londschien2025domain}, which extends the anchor regression methodology to the nonlinear setting using gradient-boosted trees.
For GroupDRO, we follow the original algorithm \citep{sagawa2019distributionally}, but use a lightweight neural network with a single hidden layer of 64 nodes. We form the training groups by splitting each training dataset into four equally sized parts based on latitude and longitude. The model is trained using the Adam optimizer \citep{kingma2014adam} 
with a learning rate $0.001$, mirror ascent step size 0.01, batch size 256, and 500 epochs.
For Anchor Boosting, we consider the original implementation and follow the recommended default settings for the tuning parameters \citep[][Section~4]{londschien2025domain}. We fit the method for several values of $\gamma \in (1, 5]$ and report the best-performing configuration ($\gamma=1.5$).

As baselines, we include the
classical control function (CF) estimator (corresponding to Lines~\ref{alg:line-1}--\ref{alg:line-3} of Algorithm~\ref{alg:bcf} and sharing the same settings as the BCF estimator), the
least squares (LS) estimator (using extreme gradient boosting with a learning rate of 0.05) and a constant mean estimator as a reference.

Figure~\ref{fig:housing-methods} presents the mean squared error (MSE) on the testing (left) and training (right) splits for the 
different methods.
Each point represents the average MSE across the 10 repetitions, and the shaded bands indicate the range of MSE values (smallest to largest) observed across these repetitions.
Compared to the LS and CF estimator,
the BCF estimator appears more robust to distributional shifts on certain test splits, such as SW/(N+SE), (N+SW)/SE, and N/S.
Compared to Anchor Boosting, the BCF generally achieves lower MSE across most test splits, except for the SW/(N+SE), where Anchor Boosting performs best.
Finally, we observe that GroupDRO and BCF achieve comparable performance, despite being fundamentally different in their training procedures, how they model distributional shifts, and their generalization guarantees.
On the training data, all methods perform similarly, except for the constant mean estimator.
\begin{figure}[t]
	\centering
	\includegraphics[scale=0.85]{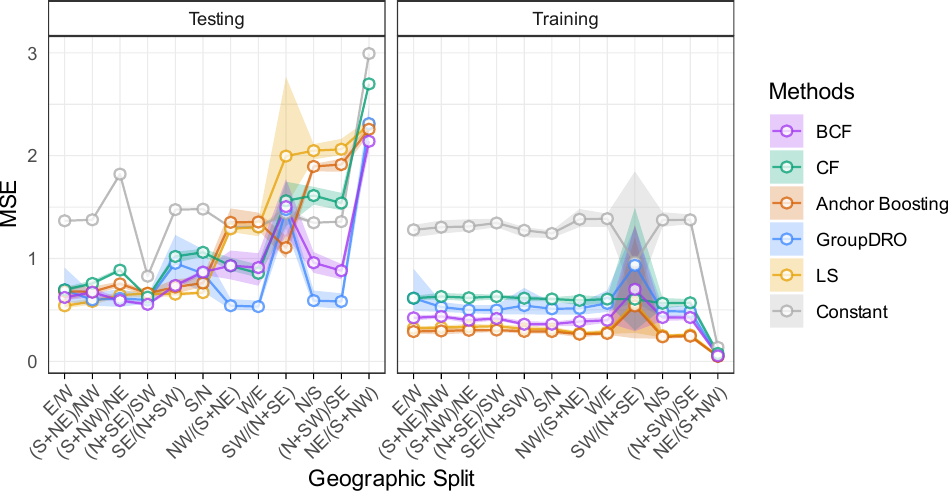}
	\caption{
		Mean squared error
		(MSE) for the BCF estimator and competing methods evaluated on the testing (left) and training (right) splits.
		The notation E/W means that region East was used for training, region West for testing.
		For each geographical split and each method, the points correspond to the average MSE over 10 repetitions, and the uncertainty bands indicate the range of MSE values over such repetitions.
	}
	\label{fig:housing-methods}
\end{figure}

\section{Conclusion and Future Work}\label{sec:conclusion}
This work studies the challenge of distribution generalization in the presence of unobserved confounding when distributional shifts are induced by exogenous variables.
We propose a strong notion of invariance that, unlike existing weaker notions, ensures distribution generalization when the structural function is nonlinear and possibly not identifiable.
We define the Boosted Control Function (BCF), a novel target of inference that satisfies our strong notion of invariance and is identifiable from the training data. The BCF is provably worst-case optimal against distributional shifts induced by the exogenous variable and aligns with the invariant most predictive (IMP) function.
Our theoretical analysis is built around the Simultaneous Equation Models for Distribution Generalization (SIMDGs),
a novel
framework that connects
machine learning with econometrics by describing data-generating processes under distributional shifts.
We demonstrate the practical effectiveness of our findings by introducing the ControlTwicing algorithm, which leverages nonparametric machine-learning techniques to learn the BCF.
Empirical evaluations on synthetic and real-world datasets indicate that our methodology outperforms traditional machine learning methods based on empirical risk minimization when generalizing to unseen distributional shifts.

Looking forward, several routes for future research emerge. One direction is to extend our strong invariance notion to settings where the outcome variable $Y$ influences some covariates $X$. Such scenarios, common in time-series and dynamic systems, have been explored in linear settings, but there is limited research on generalizing these findings to nonlinear models. Addressing this challenge could significantly broaden the applicability of invariance-based methods to more realistic and complex settings.
Another direction involves understanding how non-additive effects of exogenous variables $Z$ on $(X, Y)$ affect our notion of invariance and distribution generalization. Previous research has focused on unconfounded settings or empirically analyzed worst-case optimality guarantees. However, formulating a notion of invariance for confounded non-additive models that leads to provable generalization guarantees remains an open question.
From a theoretical perspective, it would be valuable to extend the analysis of the convergence rate of the BCF estimator to the case in which the first-step quantities are treated as random. This would require additional conditions on the nonparametric regression functions and could build on the literature of nonparametric regression with generated covariates \citep{mammen2012}.%

\section*{Acknowledgments}
We thank Jeffrey Glenn Adams, Alexander M. Christgau, Anton R. Lundborg, Leonard Henckel, Manuel Hentschel, Cesare Miglioli, Sorawit Saengkyongam, and Stanislav Volgushev for helpful discussions.
In particular, we thank Alexander M. Christgau for pointing out the trigonometric argument in the proof of Lemma~\ref{lem:l2}.
During parts of the project, JP was supported by a research grant (18968) from VILLUM FONDEN. SE was
supported by a research grant (186858) from the Swiss National Science
Foundation (SNSF). NP was supported by a research grant (0069071) from Novo Nordisk Fonden.
NG was supported by a research grant (210976) from the Swiss National Science
Foundation (SNSF).

\newpage
\appendix

\section{Additional Results on SIMDGs}
\subsection{Generating Distributions with SIMDGs}\label{sec:generative}
In this section, we shortly discuss formally how a SIMDG generates the
class of distributions $\mathcal{P}_0$ and provide the full technical
details on the notation used throughout this work.

We assume there is a fixed background probability space
$(\Omega, \mathcal{F}, \P)$. Define
$\mathcal{X}\coloneqq \R^{1+p+r+p+1}$ and let
$(\mathcal{X}, \mathcal{B}(\mathcal{X}))$ be a fixed measurable sample
space, where $\mathcal{B}(\mathcal{X})$ denotes the Borel sigma
algebra on $\mathcal{X}$. Let $(f_0, M_0, \Lambda_0, \mathcal{Q}_0)$
be a SIMDG as defined in Definition~\ref{def:ivg}. For all
$Q\in\mathcal{Q}_0$, the model $(f_0, M_0, \Lambda_0, Q)$ generates a
random variable $(U_Q, V_Q, X_Q, Y_Q, Z_Q)$ with distribution $P_Q$ on
the sample space $(\mathcal{X}, \mathcal{B}(\mathcal{X}))$ as follows:
\begin{itemize}
	\item[(1)] Let $((U_Q, V_Q), Z_Q): \Omega\to \R^{p+1} \times \R^r$ be a
		random variable with distribution $\Lambda_0 \otimes Q$.
	\item[(2)] Let $X_Q:\Omega\to \R^p$ be the random variable defined by
		$X_Q=M_0Z_Q+V_Q$.
	\item[(3)] Let $Y_Q:\Omega\to \R$ be the random variable defined by
		$Y=f_0(X_Q)+U_Q$.
	\item[(4)] Denote the distribution of $(U_Q, V_Q, X_Q, Y_Q, Z_Q)$ by $P_Q$.
\end{itemize}
For each $Q\in\mathcal{Q}_0$, we therefore get a random vector
$(U_Q, V_Q, X_Q, Y_Q, Z_Q)$ and a distribution $P_Q$ on
$(\mathcal{X}, \mathcal{B}(\mathcal{X}))$. To ease notation, we
suppress the $Q$ subscript in the notation. Moreover, to make the
dependence on the distribution clear we index probabilities and
expectations with the corresponding distribution $P$, i.e.,
$\mathbb{P}_P$ and $\E_{P}$.

The above procedure is similar to what has been suggested by
\citet{bongers2016structural}
with the difference that \citet{bongers2016structural}
consider solutions of structural equations which then are assumed to hold almost surely, not for every $\omega \in \Omega$ (this is particularly important when introducing cycles).

\subsection{{\modelname}s as Triangular Simultaneous Equation Models}\label{sec:sim_equations}

The SIMDG $(f_0, M_0, \Lambda_0, \mathcal{Q}_0)$ in
Definition~\ref{def:ivg} consists of a set of
SIMs
that are
closely related
to the
triangular simultaneous equation models considered by e.g.,
\citet{newey1999}. Formally, these are given by the simultaneous equations
\begin{equation}
	\label{eq:triangular_sem}
	\begin{split}
		Y &= g(X, Z) + U,\\
		X &= h(Z) + V,\\
	\end{split}
	\quad\text{with}\quad
	\Exp[U \mid V, Z] = \Exp[U \mid V], \quad
	\Exp[V\mid Z] = \Exp[V] = 0,
\end{equation}
where $(U, V, Z)$ are exogenous and $(X, Y)$ are endogenous.
SIMs are a tool to model identifiable causal effects from a
set of covariates $(X, Z)$ to a response $Y$. The structural
function $g$ in~\eqref{eq:triangular_sem} describes the average causal
effects that interventions on $(X, Z)$ have on $Y$. Generally, $g$
is assumed to be identifiable, that is, it is unique for any
distribution over $(X, Y, Z)$ satisfying
\eqref{eq:triangular_sem}. Unlike SCMs (discussed in
Appendix~\ref{sec:scm_equations}), SIMs are generally interpreted as
non-generative in the sense that they do not describe the full
causal structure. Instead, they only implicitly model parts of the
causal structure via the conditional mean assumptions on the
residuals $(U, V)$ and the separation of the covariates into
exogenous $Z$ and endogenous $X$ variables. As such, SIMs
naturally extend regression models to settings in which some
covariates (the endogenous ones) are confounded with the
response. Nevertheless, they can also be viewed -- as done here for
SIMDGs -- as generative models: first generate the exogenous
variables $(Z, U, V)$, then generate $X$ according to the reduced
form equation and finally generate $Y$ from the structural equation.

An attractive feature of these models is that they only describe the
relationship of interest as causal (i.e., the structural equation of
$Y$), while reducing the remaining causal relations into a reduced
form conditional relationship between exogenous and endogenous
variables that may be causally misspecified. Similarly, in the SIMs
used in SIMDGs \eqref{eq:sem}--\eqref{eq:sem-3} we do not assume that
the causal relations are correctly specified in the reduced form
equation. Furthermore, as our goal is not to model causal effects
but instead to model the invariant parts of the distributions in
$\mathcal{P}_0$, there are several further differences between the
SIMs used in SIMDGs and the commonly used ones in
\eqref{eq:triangular_sem}.

On the one hand,
the SIMs used in SIMDGs
relax \eqref{eq:triangular_sem} in two ways.
First, we do not assume
identifiability of the structural function $f_0$ (as it is often done for $g$ in~\eqref{eq:triangular_sem}) since we do not impose the necessary rank and order conditions of identifiability \citep{amemiya1985advanced},
i.e., $\rank(M_0) \leq p$ and $p \leq r$.
In particular, we do not assume that $\rank(M_0) = \min \{p, r\}$ and allow the number of covariates to exceed the number of exogenous variables, i.e., $p > r$.
Second, we allow the exogenous variable $Z$ to directly affect the response variable $Y$ (see Remark~\ref{rmk:projectability}) and therefore
$f_0$ does not need to coincide with the structural function from $X$ to $Y$.

On the other hand, SIMDGs are more restrictive than common SIMs,
as unlike
model~\eqref{eq:triangular_sem}, we assume $h$ is linear in $Z$, and we
replace the assumptions on the conditional expectations of the
exogenous variables $(U, V, Z)$ with the assumptions
in~\eqref{eq:sem-2}.
The independence assumption $Z\ind (U, V)$ directly implies the two
conditional expectation constraints in \eqref{eq:triangular_sem}.  We
make this stronger assumption to allow for the factorization
$(U,V,Z) \sim \Lambda_0 \otimes Q$. This ensures that the marginal
distribution $Q$ can be changed without affecting the distribution~$\Lambda_0$.

\subsection{{\modelname}s as Structural Causal Models}\label{sec:scm_equations}

The set of distributions $\mathcal{P}_0$ defined in~\eqref{eq:mathcalP} can also be written as a set of distributions induced by a class of SCMs.
To see this, we consider the
SCM over the variables
$(X, Y, Z, H)\in\R^{p + 1 + r + q}$ given by
\begin{align}
	\label{eq:anchor-mod}
	\begin{split}
		Z &\ := \epsi_Z\\
		H &\ := B_{HZ} Z + \epsi_H \\
		X &\ := B_{XZ} Z + B_{XH} H + \epsi_X\\
		Y &\ := g_0(X)+ B_{YZ} Z + B_{YH} H + \epsi_Y,
	\end{split}
\end{align}
where $\epsi_Z, \epsi_H, \epsi_X, \epsi_Y$ are jointly independent
noise terms. Here, $(X, Y, Z)$ are observed and $H$ is
unobserved. Moreover, $g_0: \R^p \to \R$ denotes the causal function
from $X$ to $Y$ and $B_{(\boldsymbol{\cdot\cdot})}$ are linear maps of
suitable sizes. The graph induced by~\eqref{eq:anchor-mod} is given in Figure~\ref{fig:graphs}
(left). The SCM in \eqref{eq:anchor-mod} corresponds to the nonlinear
anchor regression setup with a nonlinear causal function
\citep{rothenhausler2021anchor, buhlmann2020invariance}.

We now show how to rewrite the SCMs~\eqref{eq:anchor-mod} as a model generated by a SIMDG $(f_0, M_0, \Lambda_0, \mathcal{Q}_Z)$, where $\mathcal{Q}_Z$ denotes the set of distributions induced by interventions
on $Z$.
To this end,
define the matrix $M_0:= B_{XZ} + B_{XH}B_{HZ}$, the vector
$\beta_0^\top:=B_{YZ}$,
and assume that the projectability condition 	\citep{rothenhausler2021anchor} holds, i.e., $B_{YZ}^\top\in \col(M_0^\top)$.
Define
the two random variables
\begin{align*}
	{V}:= B_{XH} \epsi_H + \epsi_X
	\quad\text{and}\quad
	{U}:= B_{YH}\epsi_H + \epsi_Y, %
\end{align*}
and note that  $Z \ind ({U},{V})$ since
$\epsilon_Z$ is independent of
$(\epsilon_H, \epsilon_X, \epsilon_Y)$.
Moreover,
substituting the structural equations \eqref{eq:anchor-mod} into each other yields
the reduced SCM
\begin{align}
	\label{eq:reduced-anchor-mod}
	\begin{split}
		Z & := \epsilon_Z\\
		X & := M_0 Z + {V}  \\
	  Y & :=  \left( g_0(X) + \beta_0^\top M_0^{\dagger} X \right) + \left({U} - \beta_0^\top M_0^{\dagger} V \right),
	\end{split}
\end{align}
where $U$ and $V$ are not necessarily independent of each other.
Furthermore, denote by $\Lambda_0$ the distribution of the random vector $( U - \beta_0^\top M_0^{\dagger} V , V)$.
The graph induced by~\eqref{eq:reduced-anchor-mod} is given in Figure~\ref{fig:graphs} (right), where bi-directed edges indicate dependence.
By
construction the induced observational distribution over $(X, Y)$ is the
same as in \eqref{eq:anchor-mod}. Moreover, since $\epsilon_Z$ is
independent of $(U, V)$ also the interventional
distributions over $(X, Y)$ are the same as in \eqref{eq:anchor-mod}
for any intervention on $Z$.
Hence, we get that $(g_0 + \beta_0^\top M_0^\dagger, M_0, \Lambda_0, \mathcal{Q}_Z)$ is a SIMDG satisfying Definition~\ref{def:ivg}.
The reduced SCM
\eqref{eq:reduced-anchor-mod} and the original SCM
\eqref{eq:anchor-mod}
only induce the same intervention distribution on $(X, Y)$ for
interventions on $Z$ but generally not for interventions on other variables.
\begin{figure}
	\centering
	\begin{tikzpicture}
		\node[nodeObs]  (x1) {$X$};
		\node[nodeConf] (h) [above right of = x1] {$H$};
		\node[nodeObs]  (z) [above of = h] {$Z$};
		\node[nodeObs]   (y) [below right of = h] {$Y$};
		\draw[edgeObs, ->] (z) to (h);
		\draw[edgeObs, ->, bend right = 5] (z) to (x1);
		\draw[edgeObs, ->, bend left = 5] (z) to (y);
		\draw[edgeObs, ->] (x1) to node[auto, swap] {} (y);
		\draw[edgeObs, ->] (h) to (x1);
		\draw[edgeObs, ->] (h) to (y);
	\end{tikzpicture}
	\hspace{1cm}
	\begin{tikzpicture}
		\node[nodeObs]   (x1) {$X$};
		\node (h) [above right of = x1] {};
		\node (e) [above left of = x1] {};
		\node[nodeConf] (u) [above of = y] {$U$};
		\node[nodeConf] (v) [above of = x1] {$V$};
		\node[nodeObs]  (z) [left of = v] {$Z$};
		\node[nodeObs]   (y) [below right of = h] {$Y$};
		\draw[edgeObs, ->] (z) to (x1);
		\draw[edgeObs, ->](v) to (y);
		\draw[edgeObs, ->] (x1) to node[auto, swap] {} (y);
		\draw[edgeObs, ->] (v) to (x1);
		\draw[edgeObs, ->] (u) to (y);
		\draw[edgeObs, <->] (u) to (v);
	\end{tikzpicture}
	\caption{(left) Graph corresponding to the SCM defined in
		\eqref{eq:anchor-mod}. (right) Graph corresponding to the SCM
		defined in \eqref{eq:reduced-anchor-mod}. Both models induce the
		same observational distribution and the same interventional
		distribution for interventions on $Z$.}
	\label{fig:graphs}
\end{figure}

\section{Proofs}

\subsection{Proof of Proposition~\ref{prop:inv-implies-ind}}
\label{proof:inv-implies-ind}

\begin{proof}
	Recall that $P_{\train} \in \mathcal{P}_0$ denotes the distribution
	of $(U, V, X, Y, Z)$ induced by $\Lambda_0 \otimes Q_{\train}$, as
	detailed also in Appendix~\ref{sec:generative}. Due to this generative model we
	can express both $X$ and $Y$ as functions of $U$, $V$ and $Z$.
	Moreover, for all measurable $A \subseteq \R$ it holds that
	\begin{align*} %
		P_{\train}^{Y - f(X)}(A) \coloneqq \int \mathds{1}\left\{
		Y(u, v, z) - f\left(X(v, z)\right) \in A\right\} \d\Lambda_0(u, v)\d Q_{\train}(z).
	\end{align*}
	Moreover, for all fixed $z_0 \in \R^r$ and all measurable sets $A \subseteq \R$, define
	\begin{align}\label{eq:proof_def1}
		P_{z_0}^{Y - f(X)}(A) \coloneqq \int \mathds{1}\left\{
		Y(u, v, z_0) - f\left(X(v, z_0)\right) \in A\right\} \d\Lambda_0(u, v).
	\end{align}
	We first show that $f \in \mathcal{I}_0$ implies $Y - f(X) \ind Z$ under $P_{\train}$. Let $A \subseteq \R$ and $B \subseteq \R^r$ be measurable sets. Then,
	\begin{align*}
		P_{\train}^{Y - f(X), Z}\left(A \times B\right)
		= & 
    \int \1\left\{ Y(u, v, z) - f(X(v, z)) \in A, z \in B \right\}
      \d\Lambda_0(u, v)\ \d Q_{\train}\left(z\right)\\
		= &
    \int_B 
    \left[
        \int  \1\left\{ Y(u, v, z) - f(X(v, z)) \in A\right\} \d\Lambda_0(u, v) 
    \right] \d Q_{\train}\left(z\right)\\
		\stackrel{(i)}{=} &
    \int_B P_{z}^{Y - f(X)}(A)\ \d Q_{\train}\left(z\right)
    \stackrel{(ii)}{=}
    \int_B P^{Y - f(X)}_{\train} \left(A\right)\,
      \d Q_{\train}\left(z\right)\\
		= &\
    P^{Y - f(X)}_{\train} \left(A\right) \int_B \, \d Q_{\train}\left(z\right)
		= 
    P^{Y - f(X)}_{\train} \left(A\right) Q_{\train}\left(B\right),
	\end{align*}
  where $(i)$ holds by equation~\eqref{eq:proof_def1}, and $(ii)$ holds by invariance
  and $\delta_z\in\mathcal{Q}_0$.

	We now show that $Y - f(X) \ind Z$ under
	$P_{\train}$ implies $f \in \mathcal{I}_0$ if Assumption~\ref{ass:ident0} holds.
	Fix measurable sets $A \subseteq \R$ and $B \subseteq
		\R^r$.
	Then, with the same argument as above, we get
	\begin{align}\label{eq:cond-prob-proof}
		\begin{split}
			P_{\train}^{Y - f(X), Z}(A \times B)
			=  & \int_B P_{z}^{Y - f(X)}(A)\ \d Q_{\train}\left(z\right).
		\end{split}
	\end{align}
	Moreover, since $Y - f(X) \ind Z$ under
	$P_{\train}$, it holds that
	\begin{align}\label{eq:ind-proof}
		P_{\train}^{Y - f(X), Z}\left(A \times B\right) = P_{\train}^{Y - f(X)}(A)\ Q_{\train}(B).
	\end{align}
	From~\eqref{eq:cond-prob-proof} and~\eqref{eq:ind-proof}
	it holds that
	\begin{align*}
		\int_B P_z^{Y - f(X)}(A) - P_{\train}^{Y - f(X)}(A)\ \d Q_{\train}(z) = 0.
	\end{align*}
	Since $B$ is arbitrary, it holds
	$Q_{\train}$-a.s. that
	\begin{align}\label{eq:proof-as}
		P_{z}^{Y - f(X)}(A) = P_{\train}^{Y - f(X)}(A).
	\end{align}
	Now, fix a distribution $Q_* \in \mathcal{Q}_0$ and consider $P_* \in \mathcal{P}_0$ induced by $\Lambda_0 \otimes Q_*$.
	From Assumption~\ref{ass:ident0}, it holds that $P_* \ll P_{\train}$, which implies that $Q_* \ll Q_{\train}$:
	For any measurable set $C \subseteq \R^r$, it holds that
	$Q_{\train}(C) = 0 \implies P_{\train}(\mathbb{R}^{2p + 2} \times C) = 0
		\implies P_*(\mathbb{R}^{2p+2} \times C) = 0
		\implies Q_*(C) = 0$.
	Thus, equation~\eqref{eq:proof-as} holds $Q_*$-a.s., too.
	Therefore,
	\begin{align*}
		P_*^{Y - f(X)}\left(A \right)
		= & \int P_z^{Y - f(X)}(A)\ \d Q_*(z)
		= \int P_{\train}^{Y - f(X)}(A)\ \d Q_*(z)
		= P_{\train}^{Y - f(X)}(A).
	\end{align*}
	Since $A \subseteq \R$ was arbitrary, it follows that $P_{\train}^{Y - f(X)} = P_*^{Y - f(X)}$.
	Since $P_* \in \mathcal{P}_0$ was arbitrary, it follows that $f \in \mathcal{I}$.

	If Assumption~\ref{ass:ident0} is not satisfied, Example~\ref{ex:inv-stronger-ind-2} shows that $\mathcal{H}_1 \not\subseteq \mathcal{I}_0$ even under Assumption~\ref{ass:ident}.
\end{proof}

\subsection{Proof of Proposition~\ref{prop:ident-cf}}\label{proof:ident-cf}

\begin{proof}
	All equalities in this proof  are meant to hold $P_{\train}$-a.s.

	\textbf{(``\emph{only if}'')}
	Suppose that $f_\cf(X)$ is identifiable
	and that $h(X) + g(V) = 0$.
	Fix $f(X) \coloneqq f_0(X) + h(X)$ and $\gamma(V) \coloneqq \gamma_0(V) + g(V)$.
	Clearly $f(X) + \gamma(V) = f_0(X) + \gamma_0(V)$, and so, see \eqref{eq:contr-fun},
	\begin{align}\label{eq:s1}
		\Exp_{P_{\train}}[Y \mid X, V] = f(X) + \gamma(V).
	\end{align}
	We first consider the case $q=p$. Then, by
	Definition~\ref{def:control_estimator}, the \cfname\ is
	$f_\cf = f_0$ and since
	$f_\cf$ is identifiable (see Definition~\ref{def:ident-cf})
	\eqref{eq:s1} implies that
	$f_\cf(X)=f(X)+\Exp_{P_{\train}}[\gamma(V)]$. Hence,
	$h(X)=f(X)-f_0(X)=f_\cf(X)+\Exp_{P_{\train}}[\gamma(V)]-f_\cf(X)=\Exp_{P_{\train}}[\gamma(V)]$
	and so implication~\eqref{eq:ident-cf} holds.

	Next, consider the case $q<p$. Then, again by
	Definition~\ref{def:control_estimator}, the \cfname\ is defined for
	$P_{\train}$ all $x\in\R^p$ by
	$f_\cf(x) \coloneqq f_0(x) + \Exp_{P_{\train}}[\gamma_0(V) \mid
		R^\top X=R^{\top}x]$ and since $f_\cf$ is identifiable \eqref{eq:s1}
	implies that
	$f_\cf(X) = f(X) + \Exp_{P_{\train}}[\gamma(V) \mid R^\top
			X]$. Therefore
	\begin{align*}
		h(X) = f(X) - f_0(X) = \E_{P_{\train}}[\gamma_0(V) - \gamma(V) \mid R^\top
			X] \eqqcolon \delta(R^\top X),
	\end{align*}
	and so implication~\eqref{eq:ident-cf} holds.

	\textbf{(``\emph{if}'')} Suppose that implication~\eqref{eq:ident-cf} holds and that $\Exp_{P_{\train}}[Y \mid X, V] = f(X) + \gamma(V)$ for arbitrary fixed functions $f, \gamma \in \mathcal{F}$.
	Since $\Exp_{P_{\train}}[Y \mid X, V] = f_0(V) + \gamma_0(V)$ and since the conditional expectation is unique $P_{\train}$-a.s., it follows that
	\begin{align}\label{eq:s2}
		f(X) + \gamma(V) = f_0(X) + \gamma_0(V).
	\end{align}
	Define $h(X) \coloneqq f(X) - f_0(X)$ and
	$g(V) \coloneqq \gamma(V) - \gamma_0(V)$. From~\eqref{eq:s2} it follows that $h(X) + g(V) = 0$.

	We first consider the case $q=p$. In this case,
	implication~\eqref{eq:ident-cf} implies that
	$h(X)=c$. Hence,
	$g(V) = -c$ and
	\begin{equation*}
		f_\cf(X) = f_0(X) +
		\Exp_{P_{\train}}[\gamma_0(V)]
		= f(X) - c +
		\Exp_{P_{\train}}[\gamma(V)] + c,
	\end{equation*}
	which implies that $f_{\cf}$ is
	identifiable.

	Next, we consider the case $q<p$. In this case, the
	implication~\ref{eq:ident-cf} implies that
	there exists a function
	$\delta: \R^{p - q} \to \R$ such that
	$h(X) = \delta(R^\top X)$.  Since $h(X) + g(V) = 0$ and
	$h(X) = \delta(R^\top X)$, it follows that $g(V) =-\delta(R^\top X)$.
	Thus, by the way we defined $h$ and
	$g$ and the fact that $\delta(R^\top X)$ is $R^\top X$-measurable
	it follows
	\begin{align*}
		f(X) + \Exp_{P_{\train}}[\gamma(V) \mid R^\top X]
		= & \ f_0(X) + h(X) + \Exp_{P_{\train}}[\gamma_0(V) + g(V) \mid R^\top X]                         \\
		= & \ f_0(X) + \delta(R^\top X) + \Exp_{P_{\train}}[\gamma_0(V) - \delta(R^\top X) \mid R^\top X] \\
		= & \ f_0(X) + \Exp_{P_{\train}}[\gamma_0(V) \mid R^\top X] = f_\cf(X).
	\end{align*}
	Thus, $f_\cf$ identifiable.
\end{proof}

\subsection{Proof of Proposition~\ref{prop:identification}}
\label{proof:identification}

\begin{proof}
	Unless otherwise stated, all equalities involving random variables hold $P_{\train}$-a.s.

	First, Setting~\ref{ass:train-test} ensures that $M_0$ is identified by
	\begin{align*}
		M_0 = \Exp_{P_{\train}}[XZ^\top] \{\Exp_{P_{\train}}[ ZZ^\top ] \}^{-1},
	\end{align*}
	since $\Exp_{P_{\train}}[ ZZ^\top ] \succ 0$.  Since $X$ and
	$Z$ are observed and $M_0$ is identified, the control variables
	$V=X-M_0Z$ are also identified.

	Suppose now that Assumption~\ref{ass:linear-cf} holds and
	let $f:\R^p\rightarrow\R$ be measurable and
	$\gamma\in\R^p$ such
	that $\Exp_{P_{\train}}[Y \mid X, V] = f(X) + \gamma^\top V$.
	Since the conditional expectation is unique $P_{\train}$-a.s., it holds that $f(X) + \gamma^\top V = f_0(X) + \gamma_0^\top V$. Define
	$h: x \mapsto f(x) - f_0(x)$
	so that
	\begin{align*}
		h(X) = (\gamma_0 - \gamma)^\top V = (\gamma_0 - \gamma)^\top (X - M_0 Z).
	\end{align*}

	Define $\tilde{\mathcal{Z}} \coloneqq \{z_1, \tilde{z}_1,
		\ldots,z_q,\tilde{z}_q\} \subseteq \mathcal{Z}$,
	$A \coloneqq \{(v, z) \in \R^{p + r} \mid h(v + M_0 z) - (\gamma_0 - \gamma)^\top v = 0\}$ and $A_z \coloneqq \{x \in \R^p \mid h(x) - (\gamma_0 - \gamma)^\top(x - M_0z) = 0\}$ for all $z \in \mathcal{Z}$.
	First, note that $P^{V, Z}_{\train}(A) = 1$.
	Moreover, since
	\begin{align*}
		\sum_{z \in \mathcal{Z}} P^{V + M_0 z}_{\train}(A_z) P^Z_{\train}(\{z\})
		=
		\sum_{z \in \mathcal{Z}} P^{V}_{\train}(A_z - M_0z) P^Z_{\train}(\{z\})
		= P^{V, Z}_{\train}(A) = 1,
	\end{align*}
	and since $P^Z(\{z\}) > 0$ for all $z \in \mathcal{Z}$, it follows that $P^{V + M_0 z}(A_z) = 1$ for all $z \in \mathcal{Z}$.

	Using Assumption~\ref{ass:linear-cf}, fix an arbitrary $j \in \{1, \dots, q\}$ and let $z_j, \tilde{z}_j \in \tilde{\mathcal{Z}}$.
	Denote by $S_j$ and $\tilde{S}_j$ the support of $V + M_0 z_j$ and $V + M_0 \tilde{z}_j$, respectively
	and denote by $S_V$ the support of $V$.
	Since the random variables $V + M_0 z_j$ and $V + M_0 \tilde{z}_j$ are
	not mutually singular, it holds for all Borel sets $E \subseteq \R^p$  that
	\begin{align}\label{eq:non-sing}
		\left( P^{V + M_0 z_j}_{\train}(E) = 1 \implies P^{V + M_0 \tilde{z}_j}_{\train}(E) > 0 \right).
	\end{align}

	We now show that the set
	\begin{align}\label{eq:subset-special}
		\begin{split}
			A^*_j
			\coloneqq
			\left\{x \in \R^p \mid \exists\ v, \tilde{v} \in
			S_V
			\text{ s.t. } x = M_0 z_j + v = M_0 \tilde{z}_j + \tilde{v} \text{ and } x \in A_{z_j} \cap A_{\tilde{z}_j} \right\} \neq \varnothing.
		\end{split}
	\end{align}
	Since $P_{\train}^{V + M_0 z_j}(A_{{z}_j}) = 1$ and $P_{\train}^{V + M_0 z_j}(S_j) = 1$ it follows that $P_{\train}^{V + M_0 z_j}(A_{{z}_j} \cap S_j) = 1$. Therefore, from~\eqref{eq:non-sing}, $P_{\train}^{V + M_0 \tilde{z}_j}(A_{{z}_j}\cap S_j) > 0$.
	By symmetry, it holds that $P_{\train}^{V + M_0 \tilde{z}_j}(A_{\tilde{z}_j} \cap \tilde{S}_j) = 1$ and $P_{\train}^{V + M_0 {z}_j}(A_{\tilde{z}_j} \cap \tilde{S}_j) > 0$.
	Fix the set $B_j \coloneqq S_j \cap \tilde{S}_j \cap A_{z_j} \cap A_{\tilde{z}_j}$ and notice that $P_{\train}^{V + M_0 {z}_j}(B_j) > 0$ and $P_{\train}^{V + M_0 \tilde{z}_j}(B_j) > 0$, which implies that $B_j \neq \varnothing$.
	We now show that $B_j \subseteq A_j^*$.

	Fix $x \in B_j$ and note that $x \in S_j \cap \tilde{S}_j \cap A_{z_j} \cap A_{\tilde{z}_j}$.
	We want to show that there exist $v, \tilde{v} \in S_V$ such that $x = M_0 z_j + v$ and $x = M_0 \tilde{z}_j + \tilde{v}$.
	By definition of the support, $v \in S_V$ if and only if all open neighborhoods of $v$, $N_v \subseteq \R^p$, have positive probability, i.e., $P^V(N_v) > 0$.
	Fix $v = x - M_0 z_j$, and let $N_v$ be a neighborhood of $v$.
	$P^V(N_v) = P^{V + M_0z_j}(N_v + M_0 z_j)= P^{V + M_0
				z_j}(N_x)$, where $N_x\coloneqq N_v+M_0 z_j \subseteq \R^p$
	is an open neighborhood of $x$.
	Since $x \in S_j$, by definition of the support,
	it holds
	that
	$0<P^{V + M_0 z_j}(N_x) = P^V(N_v)$. Hence, since $N_v$ was arbitrary, we have
	shown
	that $v = x - M_0 z_j \in S_V$.
	By the same argument, there exists $\tilde{v} \in S_V$ such that $x = M_0 \tilde{z}_j + \tilde{v}$.
	Therefore,
	by the definition
	of $A_j^*$, it follows that $x \in A_j^*$, and thus $B_j \subseteq A_j^*$
	and therefore~\eqref{eq:subset-special} holds.

	For
	all $j\in\{1,\ldots,q\}$
	fix $x \in A_j^{*}$, then by the definition of $A_j^*$ it
	holds that
	\begin{align*}
		(\gamma_0 - \gamma)^\top(x - M_0 z_j) = h(x) = (\gamma_0 - \gamma)^\top(x - M_0 \tilde{z}_j),
	\end{align*}
	which implies $(\gamma_0 - \gamma)^\top M_0 (z_j - \tilde{z}_j) = 0$.
	Furthermore, from Assumption~\ref{ass:linear-cf}, we have
	$\vspan(\{M_0(z_j-\tilde{z}_j)\mid j\in\{1,...,q\}\})=\col(M_0)$,
	which
	implies that $\gamma_0 - \gamma \in \ker(M_0^\top)$.

	First consider the case $q=p$. Since
	$\gamma_0 - \gamma \in \ker(M_0^\top)$ it holds that
	$\gamma_0 - \gamma = 0$, thus
	$h(X)=0$
	and
	hence $f(X)=f_0(X)=f_\cf(X)$ which implies that $f_\cf$ is
	identifiable from $P_{\train}$. Next, consider the case $q<p$. Then,
	$\gamma_0 - \gamma \in \ker(M_0^\top)$ implies that
	\begin{align*}
		h(X) = (\gamma_0 - \gamma)^\top V  = (\gamma_0 - \gamma)^\top X.
	\end{align*}
	Since $(\gamma_0 - \gamma)^\top V = (\gamma_0 - \gamma)^\top X$ is $R^\top X$-measurable, it follows that
	\begin{align*}
		f(X) + \Exp_{P_{\train}}[\gamma^\top V \mid R^\top X]
		= & \ f_0(X) + h(X) +  \Exp_{P_{\train}}[\gamma^{\top}V \mid R^\top X]       \\
		= & \ f_0(X) + (\gamma_0 - \gamma)^\top X
		+\Exp_{P_{\train}}[\gamma^\top V \mid R^\top X]                              \\
		= & \ f_0(X)
		+\Exp_{P_{\train}}[\gamma^\top V + (\gamma_0 - \gamma)^\top V \mid R^\top X] \\
		= & \ f_\cf(X),
	\end{align*}
	and therefore $f_\cf$ is identifiable
	from $P_{\train}$.

	Suppose now that Assumption~\ref{ass:diffble-cf} holds and
	$f, \gamma:\R^p\rightarrow\R$ are differentiable functions
	such that $\Exp_{P_{\train}}[Y \mid X, V] = f(X) + \gamma(V)$.
	We adapt the proof of \citep[][Theorem~2.3]{newey1999}.  Denote
	the support of $(V, Z)$ and its interior by
	$\supp(P_\train^{V, Z})$ and $\supp(P_{\train}^{V, Z})^{\circ}$,
	respectively.  Since the conditional expectation is unique
	$P_{\train}$-a.s., it holds that
	$f(X) + \gamma(V) = f_0(X) + \gamma_0(V)$. Define
	$h: x \mapsto f(x) - f_0(x)$ and $g: v \mapsto \gamma(v) - \gamma_0(v)$
	so that
	\begin{align}\label{eq:ident2}
		h(X) + g(V) = h(M_0Z + V) + g(V) = 0
	\end{align}
	where we used the fact that $X = M_0Z + V$.
	We now argue that~\eqref{eq:ident2} holds for all $(v, z) \in \supp(P_\train^{V, Z})^\circ$.
	Define the set $A \coloneqq \{(v, z) \mid h(M_0 z + v) + g(v) = 0\}$, and note that $P_\train^{V, Z}(A) = 1$. By Assumption~\ref{ass:diffble-cf}, the boundary of $\supp(P_\train^{V, Z})$ has probability zero, and therefore $P_\train^{V, Z}(\supp(P_\train^{V, Z})^\circ ) = 1$. Therefore, $P_\train^{V, Z}(\supp(P_\train^{V, Z})^\circ \cap A) = 1$, i.e.,
	equality~\eqref{eq:ident2} holds for all $(v, z) \in \supp(P_\train^{V, Z})^\circ$.

	Since $h$ and $g$ are differentiable, differentiating~\eqref{eq:ident2} with respect to $z$ for all $(v, z) \in \supp(P_\train^{V, Z})^\circ$ yields
	\begin{align}\label{eq:dvtv}
		M_0^{\top}\nabla h(M_0z + v) = 0,\ \text{for all}\ (v, z) \in \supp(P_{\train}^{V, Z})^{\circ},
	\end{align}
	where $\nabla h: \R^p \to \R^p$ is the gradient of $h$.

	Define  $S_X \coloneqq  \supp(P_\train^X)$, $S_V\coloneqq \supp(P_\train^V)$ and $S_Z \coloneqq \supp(P_\train^Z)$ and note that $S_V^{\circ} $ and $S_Z^{\circ}$ are convex
	since $(V, Z)$ are independent under $P_\train$
	and since
	$P_{\train}^{V, Z}$ has convex support by assumption
	(to see this, use
	$\supp(P_\train^{V,
			Z})^\circ = S_V^\circ \times S_Z^\circ$
	and that the projection of a convex set is convex).
	Furthermore, it holds that
	\begin{equation}
		\label{eq:decomposition_of_X}
		S_X^\circ\subseteq M_0S_Z^\circ+S_V^\circ.
	\end{equation}
	To see this, assume that $S_X^\circ\not\subseteq M_0S_Z^\circ+S_V^\circ$.
	Therefore, there exists a set $A \subseteq (M_0S_Z^\circ + S_V^\circ)^c$ such that $A \neq \varnothing$ and $A$ is open (since $S_X^\circ$ is open).
	Fix $x \in A$. Since $A$ is open and since $A \subseteq S_X^\circ$, using the definitions of open set and support, there exists an open neighborhood $N_x \subseteq A$ such that $x \in N_x$ and $P_\train^X(N_x) > 0$.
	Moreover, notice that
	\begin{align*}
		P_\train^X(M_0 S_Z^\circ + S_V^\circ)
		= & \
		\int_{z \in S_Z^\circ} P_\train^{V + M_0 z}(S_V^\circ + M_0 z) \d P_\train^Z(z)
		=
		\int_{z \in S_Z^\circ} P_\train^{V}(S_V^\circ) \d P_\train^Z(z)
		= 1.
	\end{align*}
	Therefore, since $N_x \subseteq A \subseteq(M_0S_Z^\circ + S_V^\circ)^c$ it holds that $P_\train^X(N_x) \leq P_\train^X((M_0S_Z^\circ + S_V^\circ)^c) = 0$, which is a contradiction.

	We now show (by contradiction) that for all $v \in
		S_V^\circ$ the function
	$z\mapsto
		h(M_0z+v)\in\R$
	is constant on $S_Z^\circ$.
	Fix $v_0 \in S_V^\circ$ and suppose
	there exist $\bar{z}, \tilde{z} \in S_Z^\circ$ such that $h(M_0\bar{z} + v_0) \neq h(M_0\tilde{z} + v_0)$.
	Define the function $\ell:[0, 1] \to \R$ for all $t \in [0, 1]$ by
	\begin{align*}
		\ell(t) \coloneqq h\left( M_0\left\{ \bar{z} + t(\tilde{z} - \bar{z})\right\} + v_0  \right).
	\end{align*}

	Notice that $\ell(0) = h(M_0\bar{z} + v_0) \neq h(M_0 \tilde{z} + v_0) = \ell(1)$.
	Therefore, using convexity of the support
	$S_Z^\circ$,
	by the mean value theorem there exists a $c \in (0, 1)$ such that $\ell^{'}(c) = \ell(1) - \ell(0) \neq 0$.
	By the definition of the function $\ell$, the chain rule, and~\eqref{eq:dvtv}, it follows that
	\begin{align*}
		\ell^{'}(c) = (\bar{z} - \tilde{z})^{\top}M_0^{\top}\nabla h\left( M_0\left\{ \bar{z} + c(\tilde{z} - \bar{z})\right\} + v_0  \right) = 0,
	\end{align*}
	which is a contradiction. Therefore, for all $v \in
		S_V^\circ$ the function
	$z\mapsto
		h(M_0z+v)\in\R$
	is constant on $S_Z$, that is, for all $v \in S_V$ and for all $z_1, z_2 \in S_Z$ it  holds that
	\begin{align}\label{eq:z-const}
		h(M_0z_1 + v) = h(M_0z_2 + v).
	\end{align}

	Using this we will now show that for all
	$x,\tilde{x}\in S_X^\circ$ the following
	implication holds
	\begin{align}\label{eq:prop-h}
		\text{there exists }  u \in \R^r \text{ such that } x - \tilde{x} = M_0 u
		\implies h(x) = h(\tilde{x}).
	\end{align}
	To see this, fix arbitrary
	$x,\tilde{x}\in S_X^\circ$ such that there exists
	$u \in \R^r$ satisfying $x - \tilde{x} = M_0 u$. Using
	\eqref{eq:decomposition_of_X}, let $v,\tilde{v}\in S_V^\circ$ and
	$z,\tilde{z}\in S_Z^\circ$ be such that $x=M_0z+v$ and
	$\tilde{x}=M_0\tilde{z}+\tilde{v}$. Next, define the path
	$\gamma:[0,1]\rightarrow \operatorname{Im}(M_0)$ for all
	$t\in[0,1]$ by
	\begin{equation*}
		\gamma(t)\coloneqq t(v-\tilde{v}).
	\end{equation*}
	This is well-defined since
	$x-\tilde{x}=M_0(z-\tilde{z})+(v-\tilde{v})\in\operatorname{Im}(M_0)$. Next,
	fix $\eta>0$ such that
	\begin{equation}
		\label{eq:open_subset}
		\{\bar{z}\in \R^r\mid \|M_0\tilde{z}-M_0\bar{z}\|_2<\eta\}\subseteq
		S_Z^\circ.
	\end{equation}
	Then, by uniform continuity of $\gamma$ there exists
	$K\in\mathbb{N}$ such that for all $t\in[0,1-\frac{1}{K}]$ it
	holds that $\|\gamma(t+\frac{1}{K}) -\gamma(t)\|_2<\eta$. Now for
	all $k\in\{0,\ldots, K\}$ define
	$v_k\coloneqq v-\gamma(\frac{k}{K})$, which is in $S_V^\circ$ since
	$S_V^\circ$ is convex.  Then, for all $k\in\{0,\ldots,K-1\}$ it holds by
	construction that
	$\|M_0\tilde{z} -
		(M_0\tilde{z}+v_{k}-v_{k+1})\|_2=\|v_{k+1}-v_{k}\|_2<\eta$.
	Since $v_{k} - v_{k+1} = \gamma(\frac{1}{K}) \in \col(M_0)$ it holds that $ M_0\tilde{z}+v_{k}-v_{k+1}\in \col(M_0)$, i.e., there exists $z_k \in \R^r$ such that
	\begin{equation}
		\label{eq:zk_def}
		M_0\tilde{z}+v_{k}-v_{k+1}=M_0z_k.
	\end{equation}
	Hence, by~\eqref{eq:open_subset} $z_k \in S_Z^\circ$.
	Finally, we can use
	\eqref{eq:z-const} (indicated by $\star$) and \eqref{eq:zk_def}
	(indicated by $\dagger$) to get that
	\begin{align*}
		h(x)
		 & \stackrel{\star}{=}h(M_0\tilde{z}+v_0)\stackrel{\dagger}{=}h(M_0z_0+v_1)           \\
		 & \stackrel{\star}{=}h(M_0\tilde{z}+v_1)\stackrel{\dagger}{=}h(M_0z_1+v_2)           \\
		 & \cdots                                                                             \\
		 & \stackrel{\star}{=}h(M_0\tilde{z}+v_k)\stackrel{\dagger}{=}h(M_0z_{k}+v_{k+1})     \\
		 & \cdots                                                                             \\
		 & \stackrel{\star}{=}h(M_0\tilde{z}+v_{K-1})\stackrel{\dagger}{=}h(M_0z_{K-1}+v_{K}) \\
		 & \stackrel{\star}{=}h(M_0\tilde{z}+v_{K})=h(\tilde{x}),
	\end{align*}
	which proves \eqref{eq:prop-h}.

	We first consider the case $q=p$. In this case,
	\eqref{eq:prop-h} implies that $h$ is constant on all of
	$S_X^\circ$. Hence, by Proposition~\ref{prop:ident-cf},
	$f_{\cf}$ is identifiable from $P_{\train}$.
	Next,
	consider the case $q<p$. Define
	the function $\delta: \R^{p - q} \to \R$ such that for
	all $x \in S_X^\circ$ it holds that
	\begin{align*}
		\delta(R^\top x) \coloneqq h(x).
	\end{align*}
	The function $\delta: \R^{p - q} \to \R$ is well-defined: Take $x, \tilde{x} \in S_X^\circ$ such that $R^\top x = R^\top \tilde{x}$. Then, $R^\top(x - \tilde{x}) = 0$ which implies that $x - \tilde{x} = M_0 b$ for some $b \in \R^r$. Therefore, from~\eqref{eq:prop-h} it follows that
	$$
		\delta(R^\top x) = h(x) = h(\tilde{x})
		= \delta(R^\top\tilde{x}).$$
	Hence, by Proposition~\ref{prop:ident-cf},
	$f_{\cf}$ is identifiable from $P_{\train}$.
\end{proof}

\subsection{Proof of Proposition~\ref{prop:cf-imp-are-imp}}\label{proof:cf-imp-are-imp}

\begin{proof}
Recall that $P_{\train} \in \mathcal{P}_0$ denotes the distribution
over $(U, V, X, Y, Z)$ induced by
$\Lambda_0 \otimes Q_{\train}$.
From
\eqref{eq:sem} and the definition
of the BCF we
get for $q=p$ that
\begin{equation*}
  Y - f_{\cf}(X)
  = f_0(X) + U - f_0(X) = U,
\end{equation*}
and similarly for $q<p$ that
\begin{equation*}
  Y - f_{\cf}(X)
  = f_0(X) + U - f_0(X) - \Exp_{P_{\train}}[\gamma_0(V) \mid R^\top
    X] = U - \Exp_{P_{\train}}[\gamma_0(V) \mid R^\top V],
\end{equation*}
where in the last step we used that
$R^\top X = R^{\top}M_0Z + R^{\top}V= R^{\top}V$. In both cases, the
residuals $Y-f_{\cf}(X)$ only depend on $(U, V)$ which have the same
marginal distribution for all $P\in\mathcal{P}_0$ (see
Appendix~\ref{sec:generative} for details on the generative
model). Therefore, it holds that
for all $P \in \mathcal{P}_0$
\begin{align*}
  P_{\train}^{Y - f_{\cf}(X)} = P^{Y - f_{\cf}(X)}
\end{align*}
and hence $f_\cf$ is invariant, i.e., $f_{\cf} \in \mathcal{I}_0$.

We now show how the risk of $f_\cf$ relates to the risk of the least
squares predictor $f_\ols$.
First, for the case $q=p$, we get that
\begin{align*}
\begin{split}
  \risk(P_{\train}, f_\cf)
  = & \ \Exp_{P_{\train}}\left[ \left(Y - f_\cf(X)  \right)^2\right]                         \\
  = & \ \Exp_{P_{\train}}\left[ \left(Y - f_0(X)  \right)^2\right]                           \\
  = & \ \Exp_{P_{\train}}\left[ \left(Y - (f_0(X)+ \Exp_{P_{\train}}[\gamma_0(V) \mid X])
  + \Exp_{P_{\train}}[\gamma_0(V) \mid X] - \Exp_{P_{\train}}[\gamma_0(V)]  \right)^2\right] \\
  = & \ \Exp_{P_{\train}}\left[ \left(Y - f_\ols(X)\right)^2\right]
  + \Exp_{P_{\train}}\left[ \left(\Exp_{P_{\train}}[\gamma_0(V) \mid X]
  - \Exp_{P_{\train}}[\gamma_0(V)]  \right)^2\right]                                         \\
  = & \ \risk(P_{\train}, f_\ols)
  + \Exp_{P_{\train}}\left[ \left(\Exp_{P_{\train}}[\gamma_0(V) \mid X]
    - \Exp_{P_{\train}}[\gamma_0(V)]  \right)^2\right],
\end{split}
\end{align*}
where in the third equality we used that
$\Exp_{P_{\train}}[\gamma_0(V)]=0$ and in the fourth equality that the
cross terms are zero. Similarly, in the case $q<p$, we get
\begin{align*}
  \begin{split}
    \risk(P_{\train}, f_\cf)
    = &\ \Exp_{P_{\train}}\left[ \left(Y - f_\cf(X)  \right)^2\right]\\
    = &\ \Exp_{P_{\train}}\left[ \left(Y - f_0(X) - \Exp_{P_{\train}}[\gamma_0(V) \mid R^\top  X]  \right)^2\right]\\
    = &\ \Exp_{P_{\train}}\left[ \left(Y - \left\{ f_0(X) + \Exp_{P_{\train}}[\gamma_0(V) \mid X] \right\}
      + \Exp_{P_{\train}}[\gamma_0(V) \mid X]
      - \Exp_{P_{\train}}[\gamma_0(V) \mid R^\top  X]  \right)^2\right]\\
    \stackrel{\blacktriangle}{=} &\ \Exp_{P_{\train}}\left[ \left(Y - f_\ols(X)\right)^2\right]
    + \Exp_{P_{\train}}\left[ \left(\Exp_{P_{\train}}[\gamma_0(V) \mid X]
      - \Exp_{P_{\train}}[\gamma_0(V) \mid R^\top  X]
      \right)^2\right]\\
    = &\
    \risk(P_{\train}, f_\ols)
    + \Exp_{P_{\train}}\left[ \left(\Exp_{P_{\train}}[\gamma_0(V) \mid X]
      - \Exp_{P_{\train}}[\gamma_0(V) \mid R^\top  X]  \right)^2\right].
  \end{split}
\end{align*}
The cross term in $\blacktriangle$ vanishes by the
properties
of the conditional expectation $f_\ols(X)
  = \Exp_{P_\train}[Y \mid X]$
(specifically, the property that
for all measurable functions $g$ it holds that $\Exp_{P_\train}[g(X)(Y - \Exp_{P_\train}[Y \mid X])] = 0$).
\end{proof}

\subsection{Proof of Proposition~\ref{prop:ass-5}}\label{proof:ass-5}
\begin{proof}
	We want to show that
	\begin{align}\label{eq:condexp-3}
		\inf_{P \in \mathcal{P}_0} \Exp_P \left[\left( \Exp_{P}[\gamma_0(V) \mid R^\top X] - \Exp_{P}[\gamma_0(V) \mid X] \right)^2 \right]  = 0.
	\end{align}
	By definition of the infimum, it is equivalent to show that for any $\varepsilon > 0$ there exists a $P^* \in \mathcal{P}_0$ such that
	\begin{align*}
		\Exp_{P^*} \left[\left( \Exp_{P^*}[\gamma_0(V) \mid R^\top X] - \Exp_{P^*}[\gamma_0(V) \mid X] \right)^2 \right]  < \varepsilon.
	\end{align*}
	Let $Q_k\coloneqq N(0, k^2 I_r)\in\mathcal{Q}_0$ and denote
	by $(U_k,V_k,X_k,Y_k,Z_k)$ the random vector and by
	$P_k\in\mathcal{P}_0$ the distribution induced by
	$(f_0,M_0,\Lambda_0, Q_k)$. Moreover, denote by $(U, V, X, Y, Z)$
	the random vector generated by $P_{\train}$ (see
	Setting~\ref{set:setting-1}). Throughout, this proof we will always
	add subscripts to the random variables to increase clarity (see
	Appendix~\ref{sec:generative} for details on this). We prove
	the statement for each of the Conditions (a) and (b) separately.

	\begin{enumerate}
		\item[(a)]
			Suppose $\Lambda_0$
			satisfies that for $(U, V)\sim\Lambda_0$ it holds that $V$ has
			a density w.r.t.\ Lebesgue and
			$\gamma_0(V)=\E_{P_{\train}}[U\mid V]$ is almost surely
			bounded.

			Since $\operatorname{im}(R)=\ker(M_0^\top)$, it holds for all
			$k\in\mathbb{N}$ that
			$R^\top X_k = R^\top V_k$ and hence  it
			holds $P_k$-a.s.\ that
			\begin{equation*}
				\Exp_{P_k}[\gamma_0(V_k) \mid R^\top X_k]
				= \Exp_{P_k} [\gamma_0(V_k) \mid R^\top V_k].
			\end{equation*}
			Therefore, we get that
			\begin{align*}
				\begin{split}
					\Exp_{P_k}&\left[ \left(\Exp_{P_k}[\gamma_0(V_k) \mid R^\top X_k] -  \Exp_{P_k}[\gamma_0(V_k) \mid X_k] \right)^2 \right]\\
					=& \ \Exp_{P_k}\left[\left(\Exp_{P_k}[\gamma_0(V_k) \mid R^\top V_k]- \Exp_{P_k}[\gamma_0(V_k) \mid X_k] \right)^2 \right].
				\end{split}
			\end{align*}

			By assumption, the control function $\gamma_0(V_k)$ is
			bounded, so there exists $c > 0$ such that $|\gamma_0(V_k)| < c$ $P_k$-a.s.,
			and so it holds $P_k$-a.s.\ for all $k \in \N$ that
			$\abs{\Exp_{P_k}[\gamma_0(V_k) \mid X_k]} \leq c$.
			Therefore, by Lemma~\ref{lem:l2}, it holds that
			\begin{align}\label{eq:lem-limit}
				\lim_{k \to \infty} \Exp_{P_k}\left[\left(\Exp_{P_k}[\gamma_0(V_k) \mid R^\top V_k]- \Exp_{P_k}[\gamma_0(V_k) \mid X_k] \right)^2 \right] = 0.
			\end{align}
			Fix $\varepsilon >
				0$. By~\eqref{eq:lem-limit}, there
			exists a $k^* > 0$ such that
			\begin{align*}
				\Exp_{P_{k^*}}\left[\left(\Exp_{P_{k^*}}[\gamma_0(V_{k^*}) \mid R^\top V_{k^*}]- \Exp_{P_{k^*}}[\gamma_0(V_{k^*}) \mid X_{k^*}] \right)^2 \right] < \varepsilon,
			\end{align*}
			which proves the claim.

		\item[(b)]
			Suppose $\Lambda_0$ is a multivariate centered
			(non-degenerate) Gaussian distribution and denote by $\Sigma$
			the covariance matrix corresponding to the marginal distribution
			of $V$. Since $(U, V)\sim\Lambda_0$ are jointly Gaussian, the
			control function $\gamma_0(V)=\Exp_{P_{\train}}[U\mid V]$
			is
			linear in $V$ so
			that $\gamma_0(V) = \gamma_0^\top V$, for some fixed vector
			$\gamma_0 \in \R^p$.
			For all $k\in\mathbb{N}$ define the
			matrix $B_k \coloneqq k^2M_0M_0^\top + \Sigma$. By
			construction of $P_k$ and $(U_k, V_k, X_k, Y_k, Z_k)$ it holds
			that $Z_k\sim N(0, k^2 I_r)$,
			$X_k=M_0Z_k + V_k\sim N(0, B_k)$ and $\Exp_{P_k}[X_k V_k^\top]=\Sigma$.
			Using the joint Gaussianity
			of $(X_k, V_k)$, we now show how to rewrite the conditional
			expectations in~\eqref{eq:condexp-3} for $P_k$. Since
			$R^\top X_k = R^\top V_k$, it holds
			$P_k$-a.s.
			\begin{align}\label{eq:cond-exp-gauss-1}
      \begin{split}
				\Exp_{P_k}[V_k \mid R^\top X_k]
				= & \ \Exp_{P_k} [V_k  \mid R^\top V_k] = \Exp_{P_k}[V_kV_k^\top R]  \Exp_{P_k}[R^\top V_kV_k^\top R]^{-1} R^\top V_k \\
				= & \ \Sigma R (R^\top \Sigma R)^{-1}R^\top V_k
				= \Sigma \Pi V_k,
      \end{split}
			\end{align}
			where we defined $\Pi\coloneqq R (R^\top \Sigma R)^{-1}R^\top$.
			Furthermore, it holds $P_k$-a.s.
			\begin{align} \label{eq:cond-exp-gauss-2}
				\Exp_{P_k}[V_k  \mid X_k]
				= \Exp_{P_k}[V_kX_k^\top]\Exp_{P_k}[X_kX_k^\top]^{-1}X_k
				= \Sigma B_k^{-1} X_k.
			\end{align}
			Therefore, using~\eqref{eq:cond-exp-gauss-1} and~\eqref{eq:cond-exp-gauss-2}, it holds 
			\begin{align}\label{eq:exp-equality}
				\begin{split}
					\lim_{k\to\infty}\Exp_{P_k}&\left[ \left(\Exp_{P_k}[\gamma_0^\top V_k \mid R^\top X_k] -  \Exp_{P_k}[\gamma_0^\top V_k \mid X_k] \right)^2 \right]\\
					= &\ \lim_{k\to\infty}\gamma_0^\top \Exp_{P_k} \left[\left(  \Sigma \Pi V_k -   \Sigma B_k^{-1} X_k \right)\left(  V_k^\top
						\Pi \Sigma   -  X_k^\top B_k^{-1} \Sigma  \right) \right] \gamma_0\\
					\stackrel{\diamondsuit}{=} &\ \lim_{k\to\infty}\gamma_0^\top
					\left[
					\Sigma \Pi \Sigma
					+ \Sigma B_k^{-1} \Sigma
					- \Sigma \Pi \Sigma B_k^{-1} \Sigma
					- \Sigma B_k^{-1} \Sigma \Pi \Sigma
					\right]
					\gamma_0\\
					\stackrel{\clubsuit}{=}&\ \gamma_0^\top \left[ \Sigma \left(
						\Pi +  \Pi - \Pi\Sigma \Pi - \Pi\Sigma \Pi
						\right)\Sigma \right]\gamma_0
            \stackrel{\diamondsuit}{=} 0,
				\end{split}
			\end{align}
			where in $\diamondsuit$ we used that $\Pi\Sigma\Pi = \Pi$ and in $\clubsuit$ we used Lemma~\ref{lem:b-goesto-ker} which states that $\lim_{k\to\infty}B_k^{-1} = \Pi$.
      Fix $\varepsilon >
				0$. By~\eqref{eq:exp-equality}, there
			exists a $k^* > 0$ such that
			\begin{align*}
				\Exp_{P_{k^*}}\left[\left(\Exp_{P_{k^*}}[\gamma_0^\top V_{k^*} \mid R^\top X_{k^*}]- \Exp_{P_{k^*}}[\gamma_0^\top V_{k^*} \mid X_{k^*}] \right)^2 \right] < \varepsilon,
			\end{align*}
			which proves the claim.
\end{enumerate}
\end{proof}

\begin{lemma}\label{lem:l2}
	Assume Setting~\ref{set:setting-1} and suppose that
	\begin{equation*}
		\{ N(0, k^2I_r)\mid k\in\N \}\subseteq \mathcal{Q}_0.
	\end{equation*}
	Define $\gamma_0(V) \coloneqq \Exp_{P_\train}[U \mid V]$
	and assume it is almost surely bounded, and define
	$R$ as
	in \eqref{eq:definition_R}. For all $k\in\mathbb{N}$ let
	$Q_k\coloneqq N(0, k^2I_r)\in\mathcal{Q}_0$ and
	$P_k\in\mathcal{P}_0$ the distribution induced by
	$(f_0,M_0,\Lambda_0, Q_k)$. Furthermore, make the following
	additional assumptions:
	\begin{enumerate}
		\item[(i)] for $(U, V)\sim\Lambda_0$, the random variable $V$ has a
			density w.r.t.\ the Lebesgue measure,
		\item[(ii)] there exists a constant $c>0$ such that $P_1$-a.s. for all
			$k\in\mathbb{N}$ it holds that
			\begin{align*}
				\abs{\Exp_{P_1}[\gamma_0(V) \mid k M_0 Z + V]} \leq c.
			\end{align*}
	\end{enumerate}
	Then, it holds that
	\begin{align*}
		\lim_{k \to \infty} \Exp_{P_{k}}\left[\left(\Exp_{P_k}[\gamma_0(V) \mid R^\top V]- \Exp_{P_{k}}[\gamma_0(V) \mid X_k] \right)^2 \right] = 0.
	\end{align*}
\end{lemma}

\begin{proof}
	Let $(U, V, X, Y, Z)$ be the random vector generated by
	$P_1$ (details on how the random variables are generated are given
	in Appendix~\ref{sec:generative}).
	Moreover, fix $k \in \mathbb{N}$ and note that by
	construction it holds that $(U, V, X_k, Y_k, Z)\sim P_k$.

	Let $S=(s_1,\ldots,s_q)\in\R^{p\times q}$
	with $q=\rank(M_0)$
	denote the left singular vectors of $M_0$ associated to non-zero singular values --
	this implies that
	$S$ is an orthonormal basis for $\col(M_0)$.
	Moreover, define the transformed random
	variables $\tilde{X}_k\coloneqq (S, R)^{\top}X_k$,
	$\tilde{Z}\coloneqq S^{\top}M_0Z$ and
	$\tilde{V}=(\tilde{V}_S,\tilde{V}_R)\coloneqq (S^{\top}V,
		R^{\top}V)$. It then holds that,
	\begin{align*}
		\tilde{X}_k
		= k \begin{bmatrix}
			    S^\top M_0 Z \\ R^\top M_0 Z
		    \end{bmatrix}
		+ \begin{bmatrix}
			  S^\top V \\ R^\top V
		  \end{bmatrix}
		= k \begin{bmatrix}
			    \tilde{Z} \\ 0
		    \end{bmatrix}
		+ \begin{bmatrix}
			  \tilde{V}_S \\ \tilde{V}_R
		  \end{bmatrix}.
	\end{align*}
	Let $\tilde{\gamma}_0 := \gamma_0 \circ (S, R)$ and rewrite the
	conditional expectations, using
	$\gamma_0(V) = \tilde\gamma_0(\tilde{V})$
	and that $(S, R)$ are an orthonormal basis,
	as
	\begin{align}
		\label{eq:condexp-2a}
		 & \Exp_{P_k}[\gamma_0(V) \mid X_k]
		= \Exp_{P_k}[\tilde{\gamma_0}(\tilde{V}) \mid \tilde X_k]
		= \Exp_{P_1}[\tilde{\gamma_0}(\tilde{V}) \mid k \tilde{Z} + \tilde{V}_S, \tilde{V}_R], \\
		\label{eq:condexp-2b}
		 & \Exp_{P_k}[\gamma_0(V) \mid R^\top X_k]
		= \Exp_{P_1}[\tilde{\gamma_0}(\tilde{V}) \mid R^\top V]
		= \Exp_{P_1}[\tilde{\gamma_0}(\tilde{V}) \mid \tilde{V}_R].
	\end{align}
	Using the trigonometric identities $\sin(\arctan(k))=\frac{k}{\sqrt{1+k^2}}$ and $\cos(\arctan(k))=\frac{1}{\sqrt{1+k^2}}$, we have the following equalities of sigma-algebras,
	\begin{align*}
		 & \sigma(k \tilde{Z} + \tilde{V}_S)
		= \sigma\left( \frac{k \tilde{Z} + \tilde{V}_S}{\sqrt{1 + k^2}} \right) = \sigma\left( \tilde{Z} \sin\theta_k + \tilde{V}_S\cos\theta_k \right),
	\end{align*}
	where $\theta_k = \arctan(k) \in (0, \pi / 2)$.  So we can
	rewrite the r.h.s.\ of~\eqref{eq:condexp-2a} as
	\begin{align}\label{eq:integral-2}
		\Exp_{P_1}[\tilde{\gamma_0}(\tilde{V}) \mid k \tilde{Z} + \tilde{V}_S, \tilde{V}_R]
		= \Exp_{P_1}[\tilde{\gamma_0}(\tilde{V}) \mid \tilde{Z} \sin\theta_k + \tilde{V}_S \cos\theta_k, \tilde{V}_R].
	\end{align}
	Since $k\in\mathbb{N}$ was fixed arbitrarily,
	\eqref{eq:integral-2} holds for all $k\in\mathbb{N}$.
	Next, fix an arbitrary
	$\theta\in(0, \pi/2)$ and
	define the random variables $W^1 = (W^1_S, W^1_R) \in \R^p$ and
	$W^2 = (W^2_S, W^2_R) \in \R^p$ by
	\begin{align*}
		\begin{pmatrix}
			W^1_S & W^2_S \\
			W^1_R & W^2_R
		\end{pmatrix}
		\coloneqq
		\begin{pmatrix}
			\tilde{Z} & \tilde{V}_S \\
			0         & \tilde{V}_R
		\end{pmatrix}
		\begin{pmatrix}
			\sin\theta & -\cos\theta \\
			\cos\theta & \sin\theta
		\end{pmatrix}.
	\end{align*}
	This implies that
	\begin{align*}
		\begin{pmatrix}
			\tilde{Z} & \tilde{V}_S \\
			0         & \tilde{V}_R
		\end{pmatrix}
		=
		\begin{pmatrix}
			W^1_S & W^2_S \\
			W^1_R & W^2_R
		\end{pmatrix}
		\begin{pmatrix}
			\sin\theta  & \cos\theta \\
			-\cos\theta & \sin\theta
		\end{pmatrix}.
	\end{align*}
	By construction and the fact that
	$\supp(P_1^{M_0Z})=\operatorname{Im}(M_0)$ (since $P_1^Z$ has full support)
	the support of
	$W^1_S=\tilde{Z}\cos(\theta)+\tilde{V}_S\sin(\theta)$ satisfies
	$\supp(P_1^{W^1_S})=\operatorname{Im}(S^{\top})$ and hence does not depend on
	$\theta$.
	By assumption~(i), we know that $V$ has a density
	w.r.t.\
	Lebesgue which implies that also the transformed
	variable $\tilde{V}$ has a density w.r.t.\ Lebesgue, which we
	denote by $f_{\tilde{V}}: \R^p \to \R$. Moreover,
	$\tilde{Z}\sim N(0, M_0^{\top}M_0)$, hence it has a density w.r.t.\
	Lebesgue $f_{\tilde{Z}}$ which satisfies
	$\sup_{z\in\R^r}f_{\tilde{Z}}(z) < \infty$.
	By density transformation,
	and by the independence of $\tilde{Z}$ and $\tilde{V}$,
	$(W^1, W^2)$ has density
	\begin{align*}
		f_{W^1, W^2}(w^1_S, w^1_R, w^2_S, w^2_R)
		= & \ f_{\tilde{Z}, \tilde{V}}(w^1_S \sin\theta - w^2_S \cos\theta, w^1 \cos\theta + w^2 \sin\theta)      \\
		= & \ f_{\tilde{Z}}(w^1_S \sin\theta - w^2_S \cos\theta)\ f_{\tilde{V}}(w^1 \cos\theta + w^2 \sin\theta),
	\end{align*}
	where $w^1 = (w^1_S, w^1_R)$ and $w^2 = (w^2_S, w^2_R)$.
	For all $w^1_S \in \supp(P_1^{W^1_S})$ and all $v_R\in \supp(P_1^{V_R})$ it holds that
	\begin{align}\label{eq:integral}
		\begin{split}
			&\Exp_{P_1}\left[ \tilde{\gamma_0}(\tilde{V}) \mid \tilde{Z} \sin\theta + \tilde{V}_S \cos\theta = w^1_S, \tilde{V}_R = v_R \right]\\
			= &\ \Exp_{P_1}\left[ \tilde{\gamma_0}(W^1\cos\theta + W^2
				\sin\theta) \mid W^1_S = w^1_S, W^1_R \cos\theta + W^2_R \sin\theta=v_R \right] \\
			= &\ \frac{\int \tilde{\gamma}_0(w^1_S \cos\theta + w^2_S \sin\theta, v_R)\ f_{\tilde{Z}}(w^1_S \sin\theta - w^2_S \cos\theta)\ f_{\tilde{V}}(w^1_S \cos\theta + w^2_S \sin\theta, v_R)\ \d w^2_S}{\int f_{\tilde{Z}}(w^1_S \sin\theta - w^2_S \cos\theta)\ f_{\tilde{V}}(w^1_S \cos\theta + w^2_S \sin\theta, v_R)\ \d w^2_S}\\
			= &\ \frac{\int \tilde{\gamma}_0(t, v_R)\ f_{\tilde{Z}}(w^1_S \sin\theta - h(t, \theta) \cos\theta)\ f_{\tilde{V}}(t, v_R)\ \d t}{\int f_{\tilde{Z}}(w^1_S \sin\theta -  h(t, \theta) \cos\theta)\ f_{\tilde{V}}(t, v_R)\ \d t},
		\end{split}
	\end{align}
	where in the last equality we substituted
	$t \coloneqq w^1_S \cos\theta + w^2_S \sin\theta$ and
	$h(t, \theta) \coloneqq (t - w^1_S\cos\theta)/ \sin\theta$.  Notice
	that $\lim_{\theta \uparrow \pi/2} h(t, \theta) = t$
	and
	$ \lim_{\theta \uparrow \pi/2} f_{\tilde{Z}}(w^1_S \sin\theta - h(t, \theta)
		\cos\theta) = f_{\tilde{Z}}(w^1_S)$
	by continuity of
	$f_{\tilde{Z}}$.
	Since $\theta\in(0,\pi/2)$ was arbitrary, \eqref{eq:integral}
	implies for all $w^1_S \in \supp(P_1^{W^1_S})$ and all $v_R\in
		\supp(P_1^{V_R})$ (both $\supp(P_1^{W^1_S})$ and
	$\supp(P_1^{V_R})$ do not depend on $\theta$) that
	\begin{align}\label{eq:dom-conv-lim}
		\begin{split}
			& \lim_{\theta \uparrow \pi/2} \Exp_{P_1}\left[ \tilde{\gamma_0}(\tilde{V}) \mid \tilde{Z} \sin\theta + \tilde{V}_S \cos\theta = w^1_S, \tilde{V}_R = v_R \right]\\
			= &\ \lim_{\theta \uparrow \pi/2} \frac{\int \tilde{\gamma}_0(t, v_R)\ f_{\tilde{Z}}(w^1_S \sin\theta - h(t, \theta) \cos\theta)\ f_{\tilde{V}}(t, v_R)\ \d t}{\int f_{\tilde{Z}}(w^1_S \sin\theta -  h(t, \theta) \cos\theta)\ f_{\tilde{V}}(t, v_R)\ \d t}\\
			= &\ \frac{ f_{\tilde{Z}}(w^1_S) \int \tilde{\gamma}_0(t, v_R)\ f_{\tilde{V}}(t, v_R)\ \d t}{f_{\tilde{Z}}(w^1_S) \int f_{\tilde{V}}(t, v_R)\ \d t}\\
			= &\ \Exp_{P_1}[\tilde{\gamma}_0(\tilde{V}) \mid \tilde{V}_R = v_R],
		\end{split}
	\end{align}
	where in the second equality we invoked the dominated convergence
	theorem since $f_{\tilde{Z}}$ is bounded
	and $\tilde{\gamma}_0$ is bounded.

	Now, fix $\varepsilon > 0$. By~\eqref{eq:dom-conv-lim}, there exists
	$k^*\in\mathbb{N}$ such that
	$\theta^*\coloneqq\arctan(k^*)\in (0, \pi /2)$
	such that for all
	$\theta \in [\theta^*, \pi / 2)$ it holds $P_1$-a.s.\ that
	\begin{align*}
		\Big|\Exp_{P_1}[\tilde{\gamma}_0(\tilde{V}) \mid \tilde{V}_R] - \Exp_{P_1}\left[ \tilde{\gamma_0}(\tilde{V}) \mid \tilde{Z} \sin\theta + \tilde{V}_S \cos\theta , \tilde{V}_R \right] \Big| < \varepsilon.
	\end{align*}
	Moreover, by \eqref{eq:integral-2}, for all $k> k^*$ this implies $P_1$-a.s. that
	\begin{align*}
		\Big|\Exp_{P_1}[\gamma_0(V) \mid R^\top V] - \Exp_{P_1}\left[ \gamma_0(V) \mid k M_0 Z + V \right]\Big| < \varepsilon.
	\end{align*}
	Since $\varepsilon > 0$ was arbitrary, it follows $P_1$-a.s.
	\begin{align}\label{eq:chung-1}
		\lim_{k \to \infty} \Exp_{P_1}\left[ \gamma_0(V) \mid k M_0 Z + V \right] = \Exp_{P_1}[\gamma_0(V) \mid R^\top V].
	\end{align}
	Finally, by assumption~(ii), there exists a constant $c>0$ such that
	$P_1$-a.s. for all $k\in\mathbb{N}$ it holds that
	\begin{align}
		\label{eq:chung-2}
		\abs{\Exp_{P_1}[\gamma_0(V) \mid k M_0 Z + V]} \leq c.
	\end{align}
	Therefore, by~\eqref{eq:chung-1} and~\eqref{eq:chung-2}, using~\citep[][Theorem~4.1.4]{chung2001course} it holds that
	\begin{align*}
		\lim_{k \to \infty} \Exp_{P_1}\left[\left(\Exp_{P_1}[\gamma_0(V) \mid R^\top V]- \Exp_{P_1}[\gamma_0(V) \mid k M_0 Z + V] \right)^2 \right] = 0.
	\end{align*}
	Using that $(U, V, kZ, X_k, Y_k)\sim P_k$ and $X_k=kM_0Z+V$, this in particular
	implies that
	\begin{align*}
		\lim_{k \to \infty} \Exp_{P_k}\left[\left(\Exp_{P_k}[\gamma_0(V) \mid R^\top V]- \Exp_{P_k}[\gamma_0(V) \mid X_k] \right)^2 \right] = 0.
	\end{align*}
\end{proof}

\begin{lemma}\label{lem:b-goesto-ker}
  For each $k \in \N$, define $B_k \coloneqq k^2M_0M_0^\top + \Sigma$, where $M_0 \in \R^{p \times r}$ with $\rank(M_0) = q$ and $\Sigma \in \R^{p \times p}$ is positive definite. Define $R \in \R^{p \times (p - q)}$ as in~\eqref{eq:definition_R}.
  Then, as $k \to \infty$, it holds that $B_k^{-1} \to R (R^\top \Sigma R)^{-1}R^\top$.
\end{lemma}
\begin{proof}
  Let $S\in\R^{p\times q}$
	denote the left singular vectors of $M_0$ associated to non-zero singular values --
	this implies that
	$S$ is an orthonormal basis for $\col(M_0)$ and $S^\top M_0 M_0^\top S = L \in \R^{q \times q}$ is a positive definite diagonal matrix.
  Define $Q = [S, R] \in \R^{p \times p}$ so that $QQ^\top = I_p$. Then, $B_k^{-1} 
  = ( Q Q^{\top} B_k Q Q^\top )^{-1} 
  = Q ( Q^\top B_k Q)^{-1} Q^\top$, and so 
  \begin{align*}
    Q^\top B_k Q 
    = 
    \begin{bmatrix}
      k^2 L + S^\top \Sigma S 
      & 
      S^\top \Sigma R\\
      (S^\top \Sigma R)^\top
      &
      R^\top \Sigma R
    \end{bmatrix}
    \eqqcolon
    \begin{bmatrix}
      A & B\\
      B^\top & D
    \end{bmatrix}
    \eqqcolon K.
  \end{align*}
  Using the block matrix inversion formula and the Schur complement $(K/D) \coloneqq A - BD^{-1}B^\top$, we get
  \begin{align}
    (Q^\top B_k Q)^{-1}
    = &\
    \begin{bmatrix}
      (K/D)^{-1}
      &
      -(K/D)^{-1}BD^{-1}\\
      -D^{-1}B^\top (K/D)^{-1}
      &
      D^{-1} + D^{-1}B^\top(K/D)^{-1}BD^{-1}
    \end{bmatrix},\text{ and}\label{eq:block-inverse}\\
    (K/D)^{-1} 
    = &\
    \left(k^2 L + S^\top \Sigma S - S^\top \Sigma R (R^\top \Sigma R)^{-1} R^\top \Sigma S\right)^{-1}\notag\\
    = &\ \frac{1}{k^2}\left(L + E / k^2 \right)^{-1},\notag
  \end{align}
  where $E \coloneqq S^\top \Sigma S - S^\top \Sigma R (R^\top \Sigma R)^{-1} R^\top \Sigma S$. 
  Since $E$ is symmetric and it is constant with respect to $k$, it has real valued eigenvalues. Therefore, there exists $k_0 \in \N$ such that for all $k > k_0$ the matrix $L + E/k^2$ is positive definite. By using the continuity of matrix inversion for non-singular matrices we get
  \begin{align*}
    \lim_{k \to \infty}\left(L + E / k^2 \right)^{-1} 
    =
    \left(L + \lim_{k \to \infty} E / k^2 \right)^{-1} = 
    L^{-1},
  \end{align*}
  and thus $\lim_{k \to \infty} (K / D)^{-1} = 0$. Since $B$ and $D$ are constant with respect to $k$, all factors in~\eqref{eq:block-inverse} containing $(K/D)^{-1}$ vanish as $k\to\infty$.
  Therefore, we obtain
  \begin{align*}
    \lim_{k\to\infty} (Q^\top B_k Q)^{-1}
    =
    \begin{bmatrix}
      0 & 0\\
      0 & (R^\top \Sigma R)^{-1}
    \end{bmatrix},
  \end{align*}
  which implies that
  \begin{align*}
    \lim_{k\to\infty} B_k^{-1}
    = 
    Q \lim_{k\to\infty}( Q^\top B_k Q)^{-1} Q^\top
    =
    R(R^\top \Sigma R)^{-1}R^\top.
  \end{align*}
\end{proof}

\subsection{Proof of Theorem~\ref{thm:minimax-1}}\label{proof:minimaxthm-1}
\begin{proof} 
  We first show
	that
	\begin{align*}
		\inf_{P \in \mathcal{P}_0} \Exp_P \left[ \left( f_{\cf}(X) - \Exp_P[Y \mid X] \right)^2 \right] = 0.
	\end{align*}
	Recall that if $q<p$ we have $P_{\train}$-a.s.
	$f_{\cf}(X) = f_0(X) + \Exp_{P_{\train}}[\gamma_0(V) \mid R^\top
			X]$. Therefore, by Assumption~\ref{ass:ident} it holds for $P\in\mathcal{P}_0$ that
	$P$-a.s. $f_{\cf}(X) = f_0(X) + \Exp_{P_{\train}}[\gamma_0(V) \mid
			R^\top X]$. Moreover, by \eqref{eq:sem} and
	Assumption~\ref{ass:ident} it also holds for all
	$P\in\mathcal{P}_0$ that $P$-a.s. that $\E_P[Y \mid X] = f_0(X) + \Exp_P[\gamma_0(V) \mid X]$.
	Hence, by Assumption~\ref{ass:condexp} it holds that
	\begin{align}\label{eq:thm1-1}
		\inf_{P \in \mathcal{P}_0} \Exp_P \left[ \left( f_{\cf}(X) - \Exp_P[Y \mid X] \right)^2 \right]
		=
		\inf_{P \in \mathcal{P}_0} \Exp_P \left[ \left(\Exp_{P_{\train}}[\gamma_0(V) \mid R^\top X] - \Exp_P[\gamma_0(V) \mid X] \right)^2 \right] = 0.
	\end{align}
	And similarly, for $q=p$ we get that
	\begin{align}\label{eq:thm1-1b}
		\inf_{P \in \mathcal{P}_0} \Exp_P \left[ \left( f_{\cf}(X) - \Exp_P[Y \mid X] \right)^2 \right]
		=
		\inf_{P \in \mathcal{P}_0} \Exp_P \left[ \left(\Exp_P[\gamma_0(V) \mid X] \right)^2 \right] = 0.
	\end{align}
	Now, recall that the \cfname\ has constant risk across
	$\mathcal{P}_0$ since it is invariant by Proposition~\ref{prop:cf-imp-are-imp}, that is
	\begin{align*}
		\risk(P_{\train}, f_\cf)= \E_{P_{\train}}\left[ \left( Y - f_\cf(X) \right)^2 \right] = \E_{P}\left[ \left( Y - f_\cf(X) \right)^2 \right],\quad\text{for all}\ P \in \mathcal{P}_0.
	\end{align*}
	We now want to show that for all $f \in \mathcal{F}$ there exists a $P \in \mathcal{P}_0$ such that
	\begin{align*}
		\Exp_{P}\left[ (Y - f_\cf(X))^2 \right] \leq
		\Exp_{P}\left[ (Y - f(X))^2 \right].
	\end{align*}
	Let $\varepsilon > 0$ and $f  \in \mathcal{F}$, and fix
	$\delta\coloneqq\min(\frac{\varepsilon}{2}, \frac{\varepsilon^2}{16\risk(P_{\train}, f_\cf)})>0$.
	Thus, by~\eqref{eq:thm1-1}~and~\eqref{eq:thm1-1b}, there exists a $P^* \in \mathcal{P}_0$ such that
	\begin{align*}
		\Exp_{P^*}\left[ \left( f_{\cf}(X) - \Exp_{P^*}
			[Y \mid X] \right)^2 \right] < \delta.
	\end{align*}
	Then, using the Cauchy-Schwarz inequality,
	\begin{align*}
		\Exp_{P^*}\left[ \left( Y - f_\cf(X) \right)^2 \right]
		=    & \ \Exp_{P^*} \left[ \left( Y - \Exp_{P^*}[Y \mid X] + \Exp_{P^*}[Y \mid X] - f_\cf(X)\right) ^2\right]                                   \\
		\leq & \ \Exp_{P^*} \left[ \left( Y - \Exp_{P^*}[Y \mid X] \right) ^2\right]
		+ \Exp_{P^*} \left[ \left( \Exp_{P^*}[Y \mid X] - f_\cf(X)\right) ^2\right]                                                                     \\
		     & + 2 \left(
		\Exp_{P^*} \left[ \left( \Exp_{P^*}[Y \mid X] - f_\cf(X)\right) ^2\right]
		\Exp_{P^*} \left[ \left( Y - \Exp_{P^*}[Y \mid X] \right) ^2\right] \right)^{1/2}                                                               \\
		<    & \ \Exp_{P^*} \left[ \left( Y - \Exp_{P^*}[Y \mid X] \right) ^2\right] + \delta + 2\left( \delta\ \risk(P_{\train}, f_\cf)  \right)^{1/2} \\
		\leq & \ \Exp_{P^*} \left[ \left( Y - \Exp_{P^*}[Y \mid X] \right) ^2\right]  + \varepsilon                                                     \\
		\leq & \ \Exp_{P^*} \left[ \left( Y - f(X) \right) ^2\right]  + \varepsilon.
	\end{align*}
	This completes the proof.
\end{proof}

\subsection{Proof of Corollary~\ref{cor:imp}}
\label{proof:imp}

\begin{proof}

	Using Theorem~\ref{thm:minimax-1} and
	$\mathcal{I}_0\subseteq \mathcal{F}$ (which holds by definition),
	we obtain
	\begin{align*} %
		\begin{split}
			\risk(P_{\train}, f_\cf)
			&= \sup_{P \in \mathcal{P}_0}\risk(P, f_\cf)
			=  \inf_{f\in\mathcal{F}}\sup_{P \in \mathcal{P}_0} \risk(P, f)\\
			&\leq  \inf_{f\in\mathcal{I}_0}\sup_{P \in \mathcal{P}_0} \risk(P, f)
			= \inf_{f\in\mathcal{I}_0}\risk(P_{\train}, f).
		\end{split}
	\end{align*}
	Also, since $f_\cf \in \mathcal{I}_0$
	it holds that
	\begin{align*}
		\risk(P_{\train}, f_\cf) \geq \inf_{f\in\mathcal{I}_0}\risk(P_{\train}, f).
	\end{align*}
\end{proof}

\begin{lemma}\label{lem:mse-sin}
Let $V \sim N(0, \sigma^2)$. Then
\begin{align*}
  \Exp[(V - \sin(V))^2] = \sigma^2 + \frac{(1 - e^{-2\sigma^2})}{2} - 2\sigma^2e^{-\sigma^2/2}.
\end{align*}
\end{lemma}
\begin{proof}
  First note that $\Exp[\sin(V)] = 0$ since the sine is an odd function.
  Also, recall that $\sin(x)^2 = (1 - \cos(2x))/2$, for all $x \in \R$.
  Using Euler's formula and the characteristic function for the Gaussian distribution, it holds for all $t \in \R$ that
  \begin{align*}
    \Exp[\cos(tV)]=\frac{1}{2}\Exp[e^{itV} + e^{-itV}] =  e^{-\sigma^2t^2/2}.
  \end{align*}
  Using Stein's lemma, we have that
  \begin{align*}
    \Exp[V \sin(V)] = \sigma^2 \Exp[\cos(V)] = \sigma^2 e^{-\sigma^2/2}.
  \end{align*}
  Putting everything together, we have that
  \begin{align*}
    \Exp[(V - \sin(V))^2]
    = &\ \Exp[V^2] + \Exp[\sin(V)^2] - 2 \Exp[V \sin (V)]\\
    = &\ \sigma^2 + \frac{(1 - e^{-2\sigma^2})}{2} - 2\sigma^2e^{-\sigma^2/2}.
  \end{align*}
\end{proof}

\subsection{Convergence rate of the BCF estimator with oracle first step}\label{sec:app_rates}
  Recall the definition of the BCF $f_{\star}(x) = f_{0}(x) + \Exp[\gamma_{0}(V) \mid R^{\top}X = R^{\top}x]$ , for $x \in \mathbb{R}^p$.
  We have a dataset  $\dataset$ consisting of iid copies $(X_i, Y_i, Z_i, U_{i}, V_i) \sim P$, with $i \in \{1, \dots, 2n\}$. Split the dataset into equally sized disjoint parts $\dataset_1$ and~$\dataset_2$. Estimate the BCF using sample splitting as follows.

  \begin{enumerate}
	\item (\emph{First split $\dataset_1$}).   On $\dataset_{1}$, fit the additive model $\hat{f}_n(x) + \hat{\gamma}_n(v)$ by regressing $Y$ on $(X, V)$. The function $\hat{f}_{n}+\hat{\gamma}_n$ estimates the conditional expectation $\Exp[Y \mid X = x, V = v] = f_0(x) + \gamma_0(v)$.

	\item (\emph{Second split $\dataset_{2}$}). On $\dataset_{2}$ evaluate the pseudo-outcomes  $\hat{\gamma}_n(V)$ and  fit the regression function $\hat{m}_{n}: \R^{p-q}\to\R$ by regressing $\hat{\gamma}_n(V)$ on $R^{\top}X$. The function $\hat{m}_{n}$ estimates the conditional expectation $m_n(w) \coloneqq \Exp[\hat{\gamma}_n(V)\mid R^\top X = w].$
  \end{enumerate}
  The resulting BCF estimator satisfies $\hat{f}_{\star}(x) = \hat{f}_n(x) + \hat{m}_n(R^{\top}x)$.
  Before stating the proposition we report the definition of the angle between two subspaces \citep[see][Definition~9.4]{deutsch2001best}.
    \begin{definition}[Angle between two subspaces]\label{def:angle}
	Let $\mathcal{H}$ be a Hilbert space and $\cM, \cN \subseteq \mathcal{H}$ two closed subspaces such that $\cM \cap \cN = \{0\}$.
	Following \citep[][Definition 9.4]{deutsch2001best}, define the angle $\theta \in [0, \pi/2]$ between $\cM$ and $\cN$ to be the angle whose cosine $\cos(\theta)$ is defined by
	\begin{align*}
	\cos(\theta) \coloneqq \sup \left\{ |\langle u, v\rangle| \colon u \in \cM, v \in \cN, \|u\| \leq 1, \|v\| \leq 1 \right\}.
	\end{align*}
  \end{definition}
  Definition~\ref{def:angle} is used in condition~(S3) of Proposition~\ref{prop:bcf-rates}, which is equivalent to require that $\cM + \cN$ is closed \citep[see][Theorem~9.35]{deutsch2001best}. This assumption is often used in the additive models' literature \citep[e.g.,]{buja1989linear}.
  In the following, all norms are $L_2(P)$ unless specified.
\begin{proposition}\label{prop:bcf-rates}
  Let $(\Omega, \mathcal{F}, \mathbb{P})$ be the probability space where the random vector $(X, Y, Z,$ $U, V) \sim P$ and the independent datasets $\dataset_{1}$ and $\dataset_{2}$ are defined.
  Define
  $\mathcal{H} \coloneqq L_{2}(\mathcal{F})$ with inner product $\langle A, B \rangle = \Exp[AB]$ and define the closed subspaces
  $\mathcal{H}_{X} \coloneqq L_2(\sigma(X))$,
  $\mathcal{H}_V \coloneqq L_2(\sigma(V))$, and $\mathcal{H}_W \coloneqq L_2(\sigma(W))$, where 
  $W \coloneqq R^{\top}X$.
  Consider the BCF and its estimator
  \begin{align*}
	f_{\star}(x) = f_0(x) +\Exp[\gamma_0(V) \mid W = R^{\top}x],\quad\quad
	\hat{f}_{\star}(x) = \hat{f}_n(x) + \hat{m}_n(R^{\top}x),
  \end{align*}
  and assume that $f_{\star}$ is identifiable according to Definition~\ref{def:ident-cf}. 
  Furthermore, assume the following conditions.
  \begin{enumerate}
	\item[(S1)] Let $r_{n}$ be a sequence of strictly positive real numbers. The additive function $\hat{f}_n(x) + \hat{\gamma}_n(v)$ fitted on $\dataset_{1}$ by regressing $Y$ on $(X, V)$ satisfies
	  \begin{align*}
  \|\hat{f}_n(X)+\hat\gamma_n(V) - f_0(X) - \gamma_0(V)\| = O_{p}(r_{n}),
	  \end{align*}
		  that is, for every $\varepsilon > 0$ there exist $M_1 > 0$ and $N_1 > 0$ such that for all $n > N_{1}$ it holds
		  \begin{align*}
		  \P \left( \|\hat{f}_n(X)+\hat\gamma_n(V) - f_0(X) - \gamma_0(V)\| > M_{1}r_{n}\right) < \varepsilon.
		  \end{align*}

	\item[(S2)] Let $s_{n}$ be a sequence of strictly positive real numbers. Conditional on
          $\cF_{1}:=\sigma(\dataset_{1})$,
          the regression function $\hat{m}_n\colon \R^{p-q}\to\R$ fitted on $\dataset_{2}$ satisfies
		  \begin{align*}
			\|\hat m_n - m_n\|
			= O_p(s_n),
		  \end{align*}
		  that is, for every $\varepsilon > 0$ there exist $M_2 > 0$ and $N_2 > 0$ not depending on $\cF_1$ such that for all $n > N_{2}$ it holds
		  \begin{align*}
			\mathbb{P}\left( \|\hat m_n - m_n\| > M_2 s_{n} \mid \cF_1 \right) < \varepsilon.
		  \end{align*}

	\item[(S3)] (\emph{Angle condition}). The angle $\theta \in [0, \pi/2]$ between $\mathcal{H}_{X}$ and $\mathcal{H}_{V} \cap \mathcal{H}_{W}^{\perp}$ as in Definition~\ref{def:angle} is such that $\theta > 0$.

\end{enumerate}
Then
\[
\|\hat f_\star - f_\star\| = O_p\big(r_n + s_n\big).
\]

\end{proposition}

\begin{proof}
  Since 
  the BCF $f_\star$ is identifiable,
  by Lemma~\ref{lem:ident-concurv} this implies $\mathcal{H}_{X} \cap (\mathcal{H}_V\cap\mathcal{H}_W^{\perp}) = \{0\}$.
  The trivial intersection between two subspaces is called ``no-concurvity'' in the additive models' literature \citep{buja1989linear} and plays the same role as the absence of multicollinearity in linear regression.
  Together with the angle condition~(S3), it will help us to upper bound the rate of the BCF estimator with the bounds of the additive model and the regression estimator $\hat{m}_{n}$.

  We have the following quantities.
  \begin{enumerate}
	\item[(Q1)] $\hat f_n, \hat \gamma_{n}$ are fitted on $\dataset_{1}$,
	\item[(Q2)] $\hat m_n$ is fitted on $\dataset_{2}$,
	\item[(Q3)] $m_{n}(W) = \Exp[\hat{\gamma}_n(V) \mid W]$,
	\item[(Q4)] $\tilde{\gamma}_n(V) \coloneqq \hat{\gamma}_n(V) - \hat{m}_n(W)$,
	\item[(Q5)] $\gamma^{\perp}(V) \coloneqq \tilde{\gamma}_n(V) - \Exp[\tilde{\gamma}_{n}(V) \mid W]$.
  \end{enumerate}
  Condition on $\cF_{12} \coloneqq \sigma(\dataset_1, \dataset_2)$.  We have that $\Exp[\tilde{\gamma}_n(V) \mid W] = m_n(W) - \hat{m}_n(W)$,  and therefore,
  \begin{align}
	\label{eq:bound-1}
	\gamma^{\perp}(V) - \tilde{\gamma}_n(V) = -\Exp[\tilde{\gamma}_n(V) \mid W] = \hat{m}_n(W) -m_n(W).
  \end{align}
  Define $\gamma_{\star}(V) \coloneqq \gamma_{0}(V) - \Exp[\gamma_{0}(V) \mid W]$ so that $f_{0}(X) + \gamma_0(V) = f_{\star}(X) + \gamma_{\star}(V)$.
  We want to bound
  \begin{align}
	\label{eq:bound-2}
  \hat{f}_{\star}(X) - f_{\star}(X) = \hat{f}_n(X) + \hat{m}_n(W) - f_{\star}(X).
  \end{align}
  Note that
  \begin{align}
	\begin{split}
	\label{eq:sn-def}
	S_n
	\coloneqq&\ \hat{f}_n(X) + \hat{\gamma}_n(V) - f_0(X) - \gamma_0(V)\\
	=&\ \hat{f}_n(X) + \hat{\gamma}_n(V) - f_{\star}(X) - \gamma_{\star}(V)\\
	\stackrel{(Q4)}{=}&\ \hat{f}_{n}(X) + \hat{m}_n(W) + \tilde{\gamma}_n(V) - f_{\star}(X) - \gamma_{\star}(V)\\
	=&\ \hat{f}_{n}(X) + m_n(W)  - f_{\star}(X) + \gamma^{\perp}(V)- \gamma_{\star}(V)\\
			 &+ \hat{m}_n(W) - m_n(W) + \tilde{\gamma}_n(V)  - \gamma^{\perp}(V)\\
	\stackrel{\clubsuit}{=}&\ \hat{f}_{n}(X) + m_n(W)  - f_{\star}(X) + \gamma^{\perp}(V)- \gamma_{\star}(V),
  \end{split}
  \end{align}
  where in $\clubsuit$ we used~\eqref{eq:bound-1}.
  Now, define $Q \coloneqq \hat{f}_{n}(X) + m_n(W)  - f_{\star}(X)$ and $T\coloneqq \gamma^{\perp}(V)- \gamma_{\star}(V)$, so that $S_n = Q + T$.
  Conditional on $\cF_{12}$, $Q$ is a measurable function of $X$ and so $Q \in \mathcal{H}_{X}$. Also, $T$ is a measurable function of $V$ with $\Exp[T \mid W] = 0$ so $T \in \mathcal{H}_{V}\cap\mathcal{H}_{W}^{\perp}$.
  Since  $\mathcal{H}_{X} \cap (\mathcal{H}_V\cap\mathcal{H}_W^{\perp}) = \{0\}$,  using Lemma~\ref{lem:hilbert-angle} and the fact that $\theta > 0$, by (S3) we get that
  \begin{align}
	\label{eq:hilb-angle}
  \|Q\| \leq \frac{\|S_n\|}{\sin(\theta)}.
  \end{align}
  Also,
  from~\eqref{eq:bound-2}, it holds
\begin{align}
  \label{eq:bound-4}
  \|\hat f_{\star}(X) - f_{\star}(X)\| \leq \|Q\| + \|\hat m_n(W) - m_n(W)\|.
\end{align}
Combining~\eqref{eq:hilb-angle} and~\eqref{eq:bound-4}, conditional on $\cF_{12}$, it holds that
\begin{align}\label{eq:bound-5}
  \|\hat f_{\star}(X) - f_{\star}(X)\| \leq \frac{\|S_{n}\|
  }{\sin(\theta)} + \|\hat m_n(W) - m_n(W)\|.
\end{align}
Fix $\varepsilon > 0$.  For $\tilde\varepsilon \coloneqq \varepsilon / 2$, there exist $M_1, N_1 > 0$ and $M_2, N_2 > 0$ such that~(S1) and~(S2) hold with $\tilde\varepsilon$, respectively.
Define $c_{\theta} \coloneqq 1/ \sin(\theta) > 0$ and fix $M \geq \max \{c_{\theta}M_{1}, M_{2}\}$.
Using~\eqref{eq:bound-5}, it holds that
\begin{align}\label{eq:bound-6}
  \begin{split}
  &\left\{ \|\hat{f}_{\star}(X) - f_{\star}(X)\| > M(r_n+s_n) \right\}\\
  &\quad\subseteq
  \left\{ \|S_n\|c_{\theta}+ \|\hat m_n(W) - m_n(W) \| > M(r_n+s_n) \right\}\\
  &\quad\subseteq
    \left\{  \|S_n\|c_{\theta} > M r_n \right\} \cup
    \left\{ \|\hat m_n(W)-m_n(W)\| > Ms_{n} \right\}\\
  &\quad\subseteq
    \left\{  \|S_n\| > M_1 r_n \right\} \cup
    \left\{ \|\hat m_n(W)-m_n(W)\| > M_2s_{n} \right\}.
  \end{split}
\end{align}
Conditional on $\cF_{12}$,  $\|S_n\|$ and $\|\hat m_n-m_n\|$ are deterministic quantities, and therefore using~\eqref{eq:bound-6} we get
\begin{align*}
  \P &\left( \|\hat f_{\star}(X) - f_{\star}(X)\| > M(r_n+s_n) \mid \cF_{12} \right)\\
  &\leq \mathbf{1} \left\{ \|S_{n}\| > M_1r_{n} \right\} + \mathbf{1} \left\{ \|\hat m_n - m_n\|  > M_2s_{n}\right\}.
\end{align*}
Using the law of iterated expectations twice we get, for all $n > \max\{N_{1}, N_{2}\}$,
\begin{align*}
  \P &\left( \|\hat f_{\star}(X) - f_{\star}(X)\| > M(r_n+s_n) \right)\\
	 &\leq \P \left( \|S_{n}\|> M_{1}r_{n} \right)
	   + \Exp\left[\P \left( \|\hat m_n - m_n\| > M_{2}s_{n} \mid \cF_{1}\right)\right]\leq \varepsilon.
\end{align*}
Since $\varepsilon > 0$ was arbitrary, this yields
\begin{align*}
\|\hat f_{\star}(X) - f_{\star}(X)\| = O_{p}(r_n+s_n).
\end{align*}
 \end{proof}

 \begin{lemma}\label{lem:ident-concurv}
   Let $\mathcal{H} = L_{2}(\mathcal{F})$ with inner product $\langle A, B \rangle = \Exp[AB]$, and consider the closed subspaces
  $\mathcal{H}_{X} = L_2(\sigma(X))$,
  $\mathcal{H}_V = L_2(\sigma(V))$, and $\mathcal{H}_W = L_2(\sigma(W))$.
  Suppose $f_{\star}$ is identifiable according to Definition~\ref{def:ident-cf}. Then, it holds that $\mathcal{H}_{X} \cap (\mathcal{H}_V\cap\mathcal{H}_W^{\perp}) = \{0\}$.
 \end{lemma}

  \begin{proof}
	Fix $\ell \in \mathcal{H}_{X} \cap (\mathcal{H}_V\cap\mathcal{H}_W^{\perp})$. Then, there exist $f(X) \in \mathcal{H}_{X}$ and $\gamma(V) \in \mathcal{H}_{V}\cap\mathcal{H}_W^{\perp}$ such that $\ell = f(X) = \gamma(V)$, $P_{\train}$-a.s., and therefore $f(X) - \gamma(V) = 0$,  $P_{\train}$-a.s..
	From Proposition~\ref{prop:ident-cf}, there exists a $\delta$ such that $f(X) = \delta(R^{\top}X) = \delta(W)$, $P_{\train}$-a.s..
	Conditioning on $W$ and using the fact that $f(X) - \gamma(V) = 0$, $P_{\train}$-a.s., we get
	\begin{align*}
	  0 = \Exp[f(X) - \gamma(V) \mid W]= \delta(W) - \Exp[\gamma(V) \mid W]
	  =  \delta(W) - 0 = \delta(W), \quad P_{\train}\text{-a.s.,}
	\end{align*}
	because $\Exp[\gamma(V) \mid W] = 0$ since $\gamma(V) \in \mathcal{H}_W^{\perp}$.
	Therefore, $\ell = f(X) = \gamma(V) = 0$ , $P_{\train}$-a.s. Since $\ell \in \mathcal{H}_{X} \cap (\mathcal{H}_V\cap\mathcal{H}_W^{\perp})$ was arbitrary, we conclude that $\mathcal{H}_{X} \cap (\mathcal{H}_V\cap\mathcal{H}_W^{\perp}) = \{0\}$.

  \end{proof}

  \begin{lemma}\label{lem:hilbert-angle}
	Let $\mathcal{H}$ be a Hilbert space and $\cM, \cN \subseteq \mathcal{H}$ two closed subspaces such that $\cM \cap \cN = \{0\}$.
	Let $\theta$ denote the angle between $\cM$ and $\cN$ as in Definition~\ref{def:angle}.
	Then, for all $u \in \cM$ and $v \in \cN$ it holds that
	\begin{align*}
	\|u\| \sin(\theta) \leq \|u + v\|.
	\end{align*}
  \end{lemma}
  \begin{proof}
	Fix $u \in \cM$ and $v \in \cN$ and note that by Definition~\ref{def:angle} it holds
	\begin{align}
	  \label{eq:uv-cos}
	  |\langle u, v \rangle | \leq \cos(\theta) \|u\| \|v\|.
	\end{align}
	Also note that
	\begin{align}
	  \begin{split}
		\label{eq:u+v}
		\|u+v\|^2
		= &\ \|u\|^2+\|v\|^2 + 2 \langle u, v \rangle\\
		\geq&\ \|u\|^2+\|v\|^2 - 2 |\langle u, v \rangle|\\
		\stackrel{\clubsuit}{\geq}&\ \|u\|^2+\|v\|^2 - 2 \cos(\theta)\|u\| \|v\|,
	  \end{split}
	\end{align}
	where in $\clubsuit$ we used~\eqref{eq:uv-cos}. The right-hand side of~\eqref{eq:u+v} is a quadratic function in $\|v\|$, which is minimized at $t^{*} = \cos(\theta) \|u\|$, therefore,
   \begin{align}
   \label{eq:u+v-ii}
	 \|u+v\|^2
	 \geq&\ \|u\|^2+\cos^2(\theta)\|u\|^2- 2\cos^2(\theta)\|u\|^2\\
	 = &\ \|u\|^2 \sin^2(\theta),
   \end{align}
   from which it follows that $\|u\|\sin(\theta) \leq \|u+v\|$.
  \end{proof}

\section{Comparison with Existing OOD Generalization Frameworks}
This section provides a discussion of existing families of methods for out-of-distribution (OOD) generalization, comparing their assumptions to those of BCF, and why they may fail under the types of distributional shifts considered in this work.

A first family of approaches addresses the problem of OOD generalization by optimizing the worst-case risk over a neighborhood of the training distribution.
In distributionally robust optimization (DRO) the robustness set is typically defined as a ball around the empirical training distribution with respect to a probability distance, such as the Wasserstein metric \citep{sinha2017certifying} or an $f$-divergence \citep{bagnell2005robust, hu2018does}.
Although such methods provide formal guarantees, they can be overly conservative. To address this, \citet{sagawa2019distributionally} proposed GroupDRO, which defines the robustness set as the convex hull of finitely many training distributions and minimizes the worst-case empirical risk across them.
As discussed in \citet{shen2023drig}, such models cannot guarantee robustness on perturbations outside the training support and do not provide a clear geometric characterization of the robustness set. For this reason, their connection to our setting remains unclear.

A second family of methods seeks predictors that satisfy a notion of invariance across different environments $Z$. An established approach is to find a representation  $x \mapsto \Phi(x)$ such that $\mathbb{P}[Y \mid \Phi(X) = \Phi(x)]$ remains invariant across perturbations in $Z$ \citep{magliacane2018domain, rojas2018invariant, krueger21a, arjovsky2020invariant}.
In our framework, these approaches may fail in the presence of hidden confounding between $X$ and $Y$, as we show in the next example.
For simplicity, we restrict our attention to linear representations  $x\mapsto \Phi(x) \coloneqq a + b x$, for $a, b \in \R$.
Consider the linear Gaussian structural causal model $X \coloneqq Z + V$, $Y \coloneqq X + U$,  where $(U, V) \sim N(0, \Sigma)$, and $Z \ind (U, V)$, and assume that the shifts are generated by arbitrary perturbations of the marginal distribution of $Z$.
We now show that for all $a, b \in \mathbb{R}$, there exists no invariant $\Phi$.
When $b = 0$, the representation $\Phi(x) \coloneqq a$, for $a \in \mathbb{R}$, is not invariant since interventions on $Z$ change the marginal distribution of $Y$.
When $b \neq 0$,
\begin{align*}
  \Exp[Y \mid \Phi(X) = x]
  = &\ \Exp\left[Y \mid X = \frac{x-a}{b} \right]
  = \frac{x-a}{b} + \Exp\left[U \mid X =\frac{x-a}{b} \right]\\
  = &\ \left(1 + \frac{\mathrm{Cov}(U, X)}{\mathrm{Var}(X)}\right) \frac{x-a}{b}
\end{align*}
is not invariant either, since it depends on $\mathrm{Var}(X) = \mathrm{Var}(Z) + \mathrm{Var}(V)$.

A distinct notion of invariance is the counterfactual invariance of~\citet{veitch2021inv}. In their definition, a predictor $f$ is counterfactually invariant to $Z$ if $f(X(z)) = f(X(z'))$ for all $z, z'$. This definition expressed in the potential outcome notation,  applies to both the causal and anticausal setting, including cases where $Z$ and $Y$ are confounded.
Within our SIMDG defined in~\eqref{eq:sem}--\eqref{eq:sem-3}, a counterfactually invariant predictor corresponds to a function of the form $f(X) = \delta(R^\top X, X^\top e)$, where $e_j=1$ if $X_j$ is not affected by $Z$ and $e_j =0$ otherwise.
If the structural function $f_0$ depends on components of $X$ affected by $Z$, then (i) our BCF predictor is not counterfactually invariant in the sense of~\citet{veitch2021inv}, and (ii) any counterfactually invariant predictor does not satisfy our invariance notion (see Definition~\ref{def:inv_funs}).

Several above approaches, including \citep{rojas2018invariant, arjovsky2020invariant, krueger21a,puli2021ood, veitch2021inv} can handle anti-causal settings, where the target $Y$ causes some of the observed covariates $X$.
In general, our BCF method does not handle anti-causal setting of the form
  \begin{align}\label{eq:anticausal}
  \begin{split}
    X =&\ M_0Z + V\\
    Y =&\ f_0(X) + U\\
    \tilde{X} =&\ g_0(Y) + \tilde{M}_0Z+ \tilde{V},
  \end{split}
  \end{align}
  because $Z$ influences $\tilde{X}$ via the composition $g_0(f_0(M_0 \cdot + v) + u)$, which may be nonlinear. Indeed, when $g_0$ and $f_0$ are nonlinear, \cite{christiansen2020causal} shows that generalization is impossible under arbitrary interventions on $Z$. However, if the composition  $g_0(f_0(M_0 \cdot + v) + u)$ happens to be linear, then our framework still applies since we can rewrite~\eqref{eq:anticausal} as a SIMDG  as defined in Definition~\ref{def:ivg}.

\section{Simulated Trees for Experiment 1 (Section~\ref{sec:experiment-1})}\label{app:experiment-1}
The function $f_0: \R^p \to \R$ is a
decision tree depending on the first $p_\text{eff} < p$
predictors.
For any $x \in \R^p$ it is defined by
\begin{align*}
	f_0(x) := \sum_{h = 1}^{2^d}\theta_h \mathbf{1}\{x \in t_h\},
\end{align*}
where $\theta_h$ denote the constant values over the rectangular
regions $t_h \subseteq \R^p$. We sample the constant values
independently according to
$\theta \sim N(0, 1.5^2)$. We build the rectangular regions $t_h$
recursively; for each $t_h$, we uniformly sample a predictor
$j \in \{1, \dots, p_\text{eff}\}$, where $p_\text{eff} < p$ denotes the number of effective predictors, and randomly choose the split point
$s_j \sim \mathrm{U}([-2, 2])$ to obtain two children
regions $t_{h, 0}$ and $t_{h, 1}$.  We set the number of effective sample predictors to $p_\text{eff} = 3$ and the tree-depth to $d = 3$.

\section{Fitting BCF with Neural Networks}\label{app:neural-nets}
In Section~\ref{sec:bcf-est}, we described how to estimate the BCF with nonparametric estimators. In particular, we introduced the ControlTwicing algorithm to estimate the conditional expectation $$\Exp_{P_\train}[Y \mid X = x, V = v] = f_0(x) + \gamma_0(v)$$
via nonparametric regression methods for both $f_0$ and $\gamma_0$. 
Here, we show how to estimate the BCF using neural networks.
The first step--which estimates the matrix $M_0 \in \R^{p \times q}$, its left null space matrix $R \in \R^{p \times (p-q)}$, and the control variables $V \in \R^p$--remains unchanged. The second and third steps are modified as follows. Given a sample of $n$ observations $(X_1, Y_1, V_1), \dots, (X_n, Y_n, V_n)$:
\begin{enumerate}
  \item  Estimate the conditional expectation $\Exp_{P_{\train}}[Y \mid X, V] = f_0(X) + \gamma_0(V)$ 
  by solving
  \begin{align}\label{eq:opt-nn}
  (\theta_1^*, \theta_2^*) \in \argmin_{\theta_1, \theta_2} \frac{1}{n} \sum_{i=1}^{n} (Y_i - f_{\theta_1}(X_i) - \gamma_{\theta_2}(V_i))^2,
  \end{align}  
  where $f_{\theta_1}$ and $\gamma_{\theta_2}$ are neural networks parametrized by $\theta_1$ and $\theta_2$, respectively.
  \item Estimate the conditional expectation $\Exp_{P_{\train}}[Y - f_{\theta_1^*}(X)\mid R^\top X]$ by solving
  \begin{align*}
  \theta_3^* \in \argmin_{\theta_3} \frac{1}{n} \sum_{i=1}^{n} (Y_i - f_{\theta_1}(X_i) - \delta_{\theta_3}(V_i))^2,
  \end{align*}  
  where $\delta_{\theta_3}$ is a neural network, parametrized by $\theta_3$.
\end{enumerate}
The two optimization problems can be solved with classical stochastic gradient descent algorithms \citep{cauchy1847methode, robbins1951stochastic} or adaptive optimizers, e.g., \citet{kingma2014adam, loshchilov2018decoupled}.
A ready-to-use implementation of BCF with neural networks is provided  in our repository \url{https://github.com/nicolagnecco/bcf-numerical-experiments} via the class \texttt{BCFMLP}.

\section{Additional Numerical Experiments}\label{app:num-exps}

\subsection{Data-Generating Process}\label{app:dgp}
  We consider a data-generating process similar to that in~\citet{saengkyongam2022exploiting}, defined as
  \begin{align}\label{eq:scm-nonlinear}
  \begin{split}
    X_1 =&\ Z_k + V_1 + \varepsilon_1,\\
    X_2 =&\ V_2 +\varepsilon_2,\\
    \quad Y =&\ f_{0}(X_1, X_2) + \gamma_0(V_1, V_2) + \varepsilon_Y,
  \end{split}
  \end{align}
  where $\varepsilon_1, \varepsilon_2, \varepsilon_Y \stackrel{iid}{\sim}N(0, 0.1)$, $V_1, V_2 \stackrel{iid}{\sim}N(0, 1)$, and $(Z_k, V_1, V_2, \varepsilon_1, \varepsilon_2, \varepsilon_Y)$ are mutually independent for $k \in (0, 4)$.
  We define two versions of the  exogenous variable $Z_k$, namely
  \begin{align}\label{eq:zk}
    Z_k \coloneqq
    \begin{cases}
      B_k U_{k, 4} + (1 - B_k) U_{0, k}, & \text{if continuous,}\\
      k B_2, & \text{if discrete,}
    \end{cases}
  \end{align}
   where $B_k \sim \mathop{Bern}(k/4)$ and $U_{a, b} \sim \mathop{Unif}(a, b)$ for $a, b \in \R$.
   Define the mixture of radial basis functions 
  \begin{align}\label{eq:sample-f0}
  \varphi(x)\coloneqq \sum_{i=1}^{10} w_{j} \exp\left(\frac{- \left\lVert x - c_{j}\right\rVert}{3}\right)^{2}, \ x \in \mathbb{R}^2,\end{align}
  where $c_{j} \sim Unif[-5, 5]$ and  $w_{j} \sim N(0, 4)$. 
  The structural function $f_0$ is defined by $f_0(x) \coloneqq \varphi(x)$, and the control function is defined by
  \begin{align}\label{eq:gamma0-exp}
    \gamma_0(v_1, v_2) \coloneqq
    \begin{cases}
      v_1 + v_2,& \text{if linear,}\\
      \alpha (\varphi(v_1, v_2) - \beta), & \text{if nonlinear,}
    \end{cases}
  \end{align}
  where $\alpha > 0$ and $\beta \in \R$ are such that $\Exp_{P_\train}[\gamma_0(V_1, V_2)] = 0$ and $\Exp_{P_\train}[\gamma_0(V_1, V_2)^2] = 2$.

  The parameter $k$ controls the strength of the perturbation on $Z_k$. The direction of the perturbations is spanned by $M_0 = (1, 0)^\top$ and the invariant space by $R = (0, 1)^\top$. Each $k \in (0, 4)$ induces a distribution $Q_k$ over $Z_k$ which in turns induces a distribution $P_k^{(X, Y)}$ over $(X, Y) \in \R^{(2 + 1)}$, via~\eqref{eq:scm-nonlinear}.
  
\subsection{Method Configurations}\label{app:method-config}
 In the experiments of Sections~\ref{app:robustness}, \ref{app:regularization}, and~\ref{app:identifiability} we consider two implementations of the BCF estimator, BCF-MLP and BCF-XGB. Both share the same logic as described in Algorithm~\ref{alg:bcf} and differ only in the choice of base regressors to estimate $f_0$, $\gamma_0$, and $\delta_0 \coloneqq \Exp_{P_{\train}}(\gamma_0(V) \mid R^\top X)$.
 
 \paragraph{BCF-MLP.} Each component function is modelled by a one-hidden layer neural network with 64 hidden units and a sigmoid activation function.
 The parameters of $f_{\theta_1}$ and $\gamma_{\theta_2}$ are optimized for 1000 epochs with learning rate $10^{-3}$ and weight decay $2.5 \times 10^{-3}$. The parameters of $\delta_{\theta_3}$ are optimized for 1500 epochs with a smaller learning rate $10^{-4}$ and no weight decay. All optimizations use the AdamW optimizer \citep{loshchilov2018decoupled}, which decouples the learning rate from the weight decay.
 The larger number of epochs to train $\delta_{\theta_3}$ helps to capture as much of the invariant signal from Step~\ref{alg:line-4} of the BCF algorithm (Algorithm~\ref{alg:bcf}).
 A reference implementation is provided in our repository via the class \texttt{BCFMLP}.

 \paragraph{BCF-XGB.} Each component function is modelled by an XGBoost regressor \citep{tianqi2016} with learning rate of 0.05 for estimating $\hat{f}$ and $\hat{\gamma}$ and of  0.25 for estimating $\hat{\delta}$.
 The larger learning rate for $\hat{\delta}$ helps to capture as much of the invariant signal from Step~\ref{alg:line-4} of the BCF algorithm (Algorithm~\ref{alg:bcf}).
 The number of passes in the ControlTwicing algorithm (see Algorithm~\ref{alg:controltwicing}) is set to $J = 75$ with an early stopping rule satisfied if
 \begin{align*}
  \frac{\norm{\hat{f}_{j+1}(\bfX) - \hat{f}_{j}(\bfX)}_2^2}{\norm{\bfY - \bar{\bfY}}_2^2} < 5 \times 10^{-3},
 \end{align*}
 where the $\hat{f}_{j}$ denotes the gradient boosted tree at iteration $j$.
 The large number of passes, the early stopping rule, and the relatively low learning rate for $\hat{f}$ and $\hat{\gamma}$ help the ControlTwicing algorithm to learn the highly nonlinear $f_0$ defined in~\eqref{eq:sample-f0} while reducing the risk of overfitting.
 A reference implementation is provided in our repository via the class \texttt{BCF}.

 In addition to the BCF implementations, we consider the following baselines.

    \paragraph{CF-MLP.} This baseline implements the standard control function approach, where the component functions $f_0$ and $\gamma_0$ are modelled by a one-hidden layer neural network with 64~hidden units and sigmoid activation function.
    The method returns the learned function $f_{\theta_1}$.
    The configuration is identical to BCF-MLP, except that the step to estimate $\delta_{\theta_3}$ is omitted.
    A reference implementation is provided in our repository via the class \texttt{BCFMLP} with parameter \texttt{predict\_imp=False}.

 \paragraph{CF-XGB.} This baseline implements the standard control function approach, where the component functions $f_0$ and $\gamma_0$ are modelled by XGBoost regressors~\citep{tianqi2016}. 
 The method returns the learned function $\hat{f}$.
 The configuration is identical to BCF-XGB, except that the step to estimate $\hat{\delta}$ is omitted.
 A reference implementation is provided in our repository via the class \texttt{BCF} with parameter \texttt{predict\_imp=False}.

  \paragraph{LS-MLP.} Implementation of the least squares regression estimator using a one-hidden layer neural network with 64 hidden units and sigmoid activation function.
  The parameters of the network are optimized for 1000 epochs with a learning rate $10^{-3}$ and a weight decay $2.5 \times 10^{-3}$ using the AdamW optimizer \citep{loshchilov2018decoupled}.
  A reference implementation is provided in our repository via the class \texttt{OLS-MLP}.

 \paragraph{LS-XGB.} Implementation of the least squares regression estimator using an XGBoost regressor \citep{tianqi2016} with learning rate of 0.05.
 A reference implementation is provided in our repository via the class \texttt{OLS}.

\subsection{Experiment on the Efficiency and Robustness of BCF}\label{app:robustness}
This experiment studies how the performance of the BCF estimator evolves with the training sample size $n$ and with perturbations $k$ in the exogenous variable~$Z_k$.
Varying~$n$ tests the estimator’s sample efficiency, namely its ability to achieve good predictive accuracy with limited data. Increasing~$k$ measures the estimator's robustness, that is, its stability across shifts in the distribution of~$Z$.

We use the data-generating process described in Section~\ref{app:dgp} with continuous~$Z_k$ and linear control function $\gamma_0(v_1,v_2)=v_1+v_2$.
For each training sample size $n \in \{500, 1000, 2500\}$, we repeat the following procedure ten times. 
\begin{enumerate}
  \item Draw a random realization of the structural function $f_0 = \varphi(x)$ as in~\eqref{eq:sample-f0}.
  \item Generate $n$ training observations $\{(X_{i}, Y_{i}, Z_{ki})\}_{i=1}^{n}$ according to~\eqref{eq:scm-nonlinear} with $k = 0.5$.
  \item Estimate the BCF using the BCF-XGB and the BCF-MLP configurations described in Section~\ref{app:method-config}.
  \item As oracle baselines, consider (i) the structural function $f_0$, and (ii) the IMP function $f_\cf(x) \coloneqq f_0(x) + \Exp_{P_\train}[\gamma_0(V1, V2) \mid R^\top X = x]$, where $R = (0, 1)^\top$. 
  \item For each perturbation strength $k \in \{0.5, 1, 1.5, \dots, 3.5, 3.99\}$, generate $n_\test = 1000$ test samples $\{(X_{i}, Y_{i})\}_{i=1}^{n_\test}$ from~\eqref{eq:scm-nonlinear} with the same $f_0$ and $\gamma_0$ and denote by $P_{k}^{(X, Y)}$ the resulting distribution of $(X, Y)$.
  \item Evaluate each method by its mean squared error (MSE) under $P_{k}^{(X, Y)}$.
\end{enumerate} 
Figure~\ref{fig:robustness-exp} displays the test MSE of each method as a function of the perturbation strength~$k$, for training sample sizes $n \in \{500, 1000, 2500\}$.
The solid lines show the average MSE across ten repetitions, while the shaded ribbons indicate the pointwise minimum and maximum MSE attained across runs.
The dotted and dashed lines correspond to the oracle risks of the structural~$f_0$ and the IMP~$f_\cf$, respectively.

We assess sample efficiency by looking at how the mean and variability of the MSE evolve with~$n$ for fixed $k$. As $n$ increases, the shaded ribbons shrink and the mean curves move closer to the population baseline, indicating reduced variance and estimation error.
In this setting, the BCF-MLP converges slightly faster toward the oracle than BCF-XGB.
We assess robustness by looking at how the average MSE evolves with $k$ for fixed $n$. For smaller sample sizes ($n=500$ and $n =1000$), BCF-MLP shows greater robustness than BCF-XGB. For $n =2500$, both estimators show nearly constant MSE across perturbation strengths, approaching the invariant risk of $f_{\cf}$.
\begin{figure}[h]
  \centering
  \includegraphics[width=\textwidth]{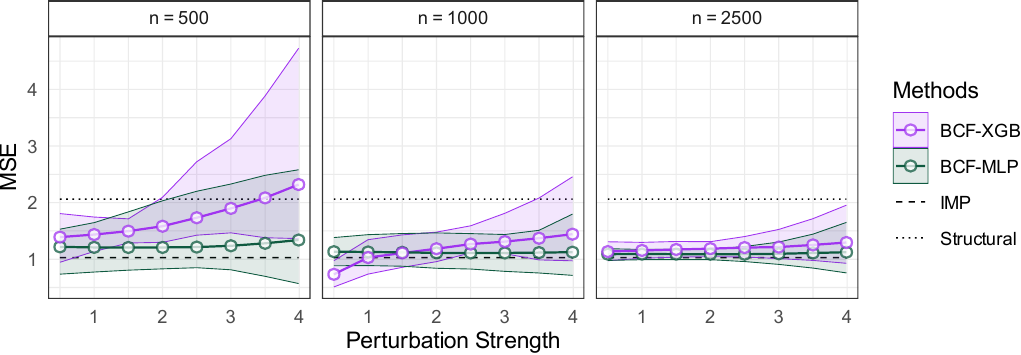}
  \caption{Test mean squared error (MSE) of different estimators as a function of the perturbation strength~$k$.
  Solid lines show the average MSE over ten repetitions; shaded ribbons indicate the pointwise minimum and maximum MSE across runs.
  Dotted and dashed lines correspond to the oracle predictors $f_0$ and $f_\cf$.
  As $n$ increases, both methods show a reduced variability in the MSE and their average MSEs move closer to the population baseline. As $k$ increases, BCF-MLP shows better robustness than BCF-XGB for small to medium sample sizes. For large $n$, both estimators show nearly constant average MSE across perturbation strengths.}
  \label{fig:robustness-exp}
\end{figure}

\subsection{Experiment on the Effect of Regularization on Control Function Method}\label{app:regularization}
Compared to nonparametric regression methods, such as random forests or boosted trees, neural networks  offer a direct way to regularize the weights $\theta$ of a predictive function $f_\theta$ through a weight decay factor~$\lambda > 0$. In this experiment, we regularize the neural networks' weights with the adaptive AdamW optimizer \citep{loshchilov2018decoupled}, which decouples the weight decay from the learning  rate.

Recall Step~\ref{alg:line-3} of the BCF algorithm, where we estimate the conditional expectation 
$$\Exp_{P_\train}[Y \mid X, V] = f_0(X) + \gamma_0(V).$$
As discussed in Example~\ref{ex:ident}, $f_0$ and $\gamma_0$ may not be identifiable from $P_\train$; there can exist infinitely many pairs $(f, \gamma)$ such that $\Exp_{P_\train}[Y \mid X, V] = f_0(X) + \gamma_0(V)$, $P_\train^{X, V}$-almost surely.
When estimating this conditional expectation by $f_{\theta_1} + \gamma_{\theta_2}$ via neural networks, the particular solution $(\theta_1^*, \theta_2^*)$ obtained during training depends on the capacity of the two estimators.
In this experiment, we study how regularizing $(\theta_1, \theta_2)$ influences the predictive performance and invariance of
$f_{\theta_1^*}$.

We consider the data-generating process described in Section~\ref{app:dgp} with continuous~$Z_k$ and linear control function $\gamma_0(v_1,v_2)=v_1+v_2$.
We repeat the following procedure ten times. 
\begin{enumerate}
  \item Draw a random realization of the structural function $f_0 = \varphi(x)$ as in~\eqref{eq:sample-f0}.
  \item Generate $n=1000$ training observations $\{(X_{i}, Y_{i}, Z_{ki})\}_{i=1}^{n}$ according to~\eqref{eq:scm-nonlinear} with $k = 0.5$.
  \item Estimate the BCF using the BCF-MLP configurations described in Section~\ref{app:method-config}.
  \item As baselines consider (i) the control function estimator CF-MLP described in Section~\ref{app:method-config}, with  weight decay parameters  $\lambda \in \{0.0025, 0.025, 0.25, 2.5\}$ for both~$f_{\theta_1}$ and~$\gamma_{\theta_2}$; (ii) the least squares estimator LS-MLP configured as in Section~\ref{app:method-config}; (iii) the structural function $f_0$; (iv) the IMP function $f_\cf(x) \coloneqq f_0(x) + \Exp_{P_\train}[\gamma_0(V1, V2) \mid R^\top X = x]$, where $R = (0, 1)^\top$.
  \item For each perturbation strength $k \in \{0.5, 1, 1.5, \dots, 3.5, 3.99\}$, generate $n_\test = 1000$ test samples $\{(X_{i}, Y_{i})\}_{i=1}^{n_\test}$ from~\eqref{eq:scm-nonlinear} with the same $f_0$ and $\gamma_0$ and denote by $P_{k}^{(X, Y)}$ the resulting distribution of $(X, Y)$.
  \item Evaluate each method by its mean squared error (MSE) under $P_{k}^{(X, Y)}$.
\end{enumerate} 
Figure~\ref{fig:regularization} shows the test MSE of each method as a function of the perturbation strength~$k$. 
The solid lines show the average MSE across ten repetitions, while the dotted and dashed lines correspond to the oracle risks of the structural~$f_0$ and the IMP~$f_\cf$, respectively.

In this specific setting, moderate regularization, $\lambda \in  \{0.0025, 0.025, 0.25\}$, stabilizes the control function estimator, yielding an approximately invariant MSE across different perturbation strengths. However, for strong regularization, $\lambda = 2.5$, the  invariance property degrades due to underfitting of $f_{\theta_1^*}$.
In the same setting, BCF-MLP trained with $\lambda = 0.0025$ achieves both low and invariant MSE across~$k$, approaching the oracle risk of the IMP~$f_\cf$.

\begin{figure}[t]
  \centering
  \includegraphics[scale=1]{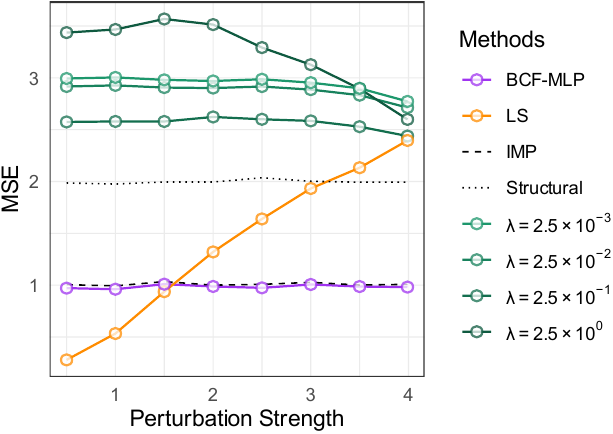}
  \caption{
	Test mean squared error (MSE) as a function of the perturbation strength~$k$ for different methods.
	Solid lines show the average MSE across ten repetitions.
	Dotted and dashed lines indicate the oracle predictors $f_0$ and~$f_\cf$.
	The green lines correspond to the standard control function estimator (CF-MLP) with weight decay parameter $\lambda$ applied to both $f_{\theta_1}$ and $\gamma_{\theta_2}$.
	Moderate regularization stabilizes CF-MLP, while strong regularization ($\lambda=2.5\times10^{0}$) leads to non-invariant MSE.
	BCF-MLP achieves a low and nearly invariant MSE across~$k$, close to the oracle IMP.}
  \label{fig:regularization}
\end{figure}

\subsection{Experiment on Identifiability Assumptions}\label{app:identifiability}
In this experiment, we study the performance of the BCF estimator in finite samples under the two identifiability assumptions discussed in Section~\ref{sec:ident} (Assumptions~\ref{ass:linear-cf} and~\ref{ass:diffble-cf}).
Assumption~\ref{ass:linear-cf} requires the control function $\gamma_0$ to be linear and the exogenous variable $Z_k$ to be discrete at training time, while Assumption~\ref{ass:diffble-cf} requires~$f_0$ and~$\gamma_0$ to be differentiable and~$Z_k$ to be continuous. For each training sample size $n \in \{1000, 5000\}$ and each of the two assumptions, we repeat the following experiment ten times.

\begin{enumerate}
  \item Draw a random realization of the structural function $f_0=\phi(x)$ as in~\eqref{eq:sample-f0}. Under Assumption~\ref{ass:linear-cf}, fix $\gamma_0(v_1, v_2) = v_1 + v_2$; otherwise draw a random realization $\gamma_0 = \alpha(\phi(v_1, v_2)- \beta)$ as in~\eqref{eq:gamma0-exp}.
  \item Generate $n=1000$ training observations $\{(X_{i}, Y_{i}, Z_{ki})\}_{i=1}^{n}$ according to~\eqref{eq:scm-nonlinear} with $k = 0.5$. Under Assumption~\ref{ass:linear-cf}, $Z_k$ is discrete; otherwise it is continuous as described in~\eqref{eq:zk}.
  \item Estimate the BCF using the BCF-XGB configurations described in Section~\ref{app:method-config}.
  \item As baselines, consider CF-XGB and LS-XGB as configured in Section~\ref{app:method-config}.
  As oracle  baselines, we consider (i) the structural function $f_0$, and (ii) the oracle-BCF function $\hat{f}_{\cf}(x)\coloneqq f_0(x) + \hat{\Exp}_{P_\train}[\gamma_0(V1, V2) \mid R^\top X = x]$,  where the second term is estimated with an XGBoost regressor, and $R = (0, 1)^\top$.
  \item For each perturbation strength $k \in \{0.5, 1, 1.5, \dots, 3.5, 3.99\}$, generate $n_\test = 1000$ test samples $\{(X_{i}, Y_{i})\}_{i=1}^{n_\test}$ from~\eqref{eq:scm-nonlinear} with the same $f_0$ and $\gamma_0$ and $Z_k$ continuous as in~\eqref{eq:zk}. Denote by $P_{k}^{(X, Y)}$ the resulting distribution of $(X, Y)$.
  \item Evaluate each method by its mean squared error (MSE) under $P_{k}^{(X, Y)}$.
\end{enumerate}
Figure~\ref{fig:identifiability} shows the test MSE of each estimator as a function of the perturbation strength~$k$ for both assumptions  and sample sizes $n \in \{1000, 5000\}$.
The solid lines show the average MSE across ten repetitions.
In this numerical experiment, across all settings, BCF-XGB outperforms the standard control function estimator CF-XGB and achieves lower MSE than LS-XGB for moderate to large perturbation strengths $k > 2$. Under both identifiability assumptions, the MSE curve of BCF-XGB approaches its oracle counterpart as the training sample size increases.

\begin{figure}[t]
  \centering
  \includegraphics[width=.7\textwidth]{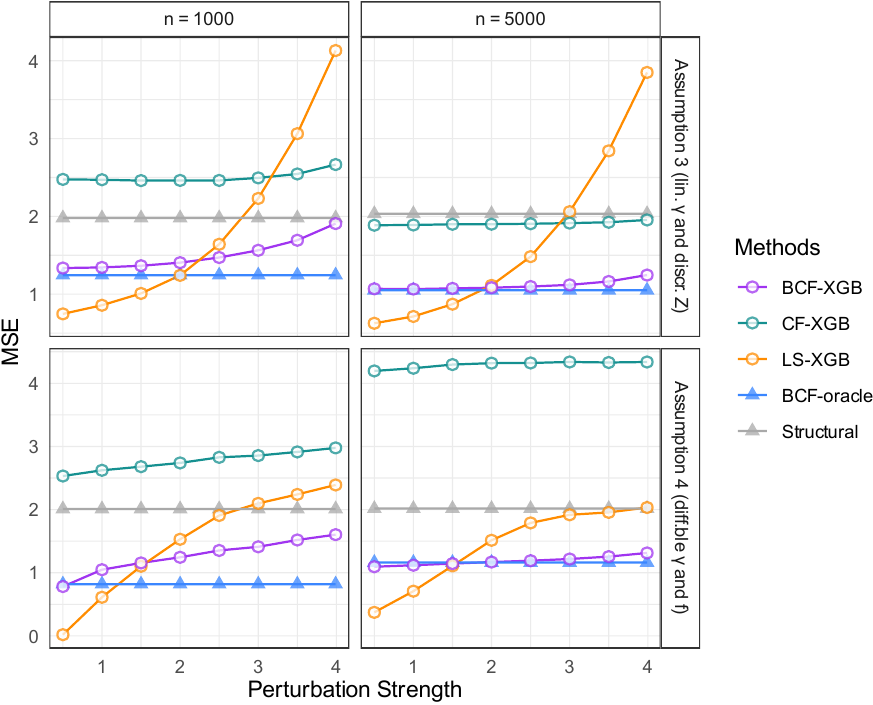}
  \caption{
	Test mean squared error (MSE) as a function of the perturbation strength~$k$ for different methods. Each panel shows different sample sizes and assumptions.
	Solid lines show the average MSE across ten repetitions.
	BCF-XGB outperforms CF-XGB under all settings and LS-XGB for moderate to large perturbation sizes. Under both assumptions, as $n$ increases, BCF-XGB approaches the performance of its oracle counterpart BCF-oracle.
	}
  \label{fig:identifiability}
\end{figure}

\section{Guidelines on When BCF Can Be Helpful in Practice}
In practice, using BCF can be helpful when (1)  unobserved variables $H$ influence both $X$ and $Y$, and (2) the practitioner has access to variables $Z$ that are (i) exogenous, and (ii) unavailable at test time.
The choice of an exogenous $Z$ typically relies on domain knowledge. The unavailability of $Z$ at test time can arise for practical reasons. For example, this can occur when $Z$ is categorical and new categories appear at test time, making $Z$ unusable as a covariate. In other cases, such as the housing datasets described in Section~\ref{sec:california-housing},  $Z$ is continuous and available at test time, but the conditional distribution of $Y \mid X, Z$ may not extrapolate to unseen values of $Z$.

Once such a variable $Z$ has been identified, a simple pipeline to motivate the use of BCF is as follows. First, split the training data according to the values of $Z$ (for categorical $Z$, use distinct categories; for continuous $Z$, create bins or buckets).
As a baseline, fit a least squares model, and evaluate whether its MSE remains stable across held-out splits. If it does not, this suggests that $Z$ might induce distribution shifts in $(X, Y)$, and applying BCF might be beneficial.
Finally, to heuristically assess BCF's robustness, verify that its MSE remains stable across different test splits.

\newpage
\bibliography{ref}

\end{document}